\documentclass{article}

\usepackage[utf8]{inputenc} 
\usepackage[T1]{fontenc}    
\usepackage{nicefrac}       
\usepackage{microtype}      

\usepackage{array}
\usepackage{subfigure}
\usepackage{color}
\usepackage{graphics}
\usepackage{epsfig}
\usepackage{booktabs}
\usepackage{url}
\usepackage{psfrag}
\usepackage{relsize}
\usepackage{amsmath,amsfonts,mathrsfs,mathtools,amssymb}
\usepackage{tikz,pgfplots}
\usepackage{tkz-tab}
\usepackage[format=hang,justification=justified,singlelinecheck=false,font=footnotesize]{caption}

\usepackage{multirow}
\usepackage{threeparttable}
\usepackage{wrapfig}

\newcommand{\eins}{\boldsymbol{1}}

\DeclareSymbolFont{wideparensymbol}{OMX}{yhex}{m}{n}
\DeclareMathAccent{\wideparen}{\mathord}{wideparensymbol}{"F3} 

\newcommand{\argmin}{\operatornamewithlimits{arg \, min}}
\newcommand{\argmax}{\operatornamewithlimits{arg \, max}}

\usepackage{tikz,pgfplots}
\pgfplotsset{compat=newest}
\pgfplotsset{plot coordinates/math parser=false,trim axis left}
\newlength\figureheight
\newlength\figurewidth

\newtheorem{theorem}{Theorem}
\newtheorem{assumption}{Assumption}
\newtheorem{definition}{Definition}
\newtheorem{lemma}{Lemma}
\newtheorem{proposition}{Proposition}
\newtheorem{proof}{Proof}

\usepackage{hyperref}

\usepackage[accepted]{icml2021}

\icmltitlerunning{GBHT: Gradient Boosting Histogram Transform 
	for Density Estimation}

\begin{document}
	
	\twocolumn[
	\icmltitle{GBHT: Gradient Boosting Histogram Transform 
		for Density Estimation}
	
	
	
	\icmlsetsymbol{equal}{*}
	
	\begin{icmlauthorlist}
	\icmlauthor{Jingyi Cui}{pku,equal}
	\icmlauthor{Hanyuan Hang}{ut,equal}
	\icmlauthor{Yisen Wang}{pku}
	\icmlauthor{Zhouchen Lin}{pku,pazhou}
\end{icmlauthorlist}

\icmlaffiliation{ut}{Department of Applied Mathematics, University of Twente, The Netherlands}
\icmlaffiliation{pku}{Key Lab. of Machine Perception (MoE), School of EECS, Peking University, China}
\icmlaffiliation{pazhou}{Pazhou Lab, Guangzhou, China}

\icmlcorrespondingauthor{Yisen Wang}{yisen.wang@pku.edu.cn}
	
	\icmlkeywords{Machine Learning, ICML}
	
	\vskip 0.3in
	]
	
	
	
	\printAffiliationsAndNotice{\icmlEqualContribution} 

\begin{abstract}
In this paper, we propose a density estimation algorithm called \textit{Gradient Boosting Histogram Transform} (GBHT), where we adopt the \textit{Negative Log Likelihood} as the loss function to make the boosting procedure available for the unsupervised tasks.
From a learning theory viewpoint, we first prove fast convergence rates for GBHT with the smoothness assumption that the underlying density function lies in the space $C^{0,\alpha}$. 
Then when the target density function lies in spaces $C^{1,\alpha}$, we present an upper bound for GBHT which is smaller than the lower bound of its corresponding base learner, in the sense of convergence rates.
To the best of our knowledge, we make the first attempt to theoretically explain why boosting can enhance the performance of its base learners for density estimation problems. 
In experiments, we not only conduct performance comparisons with the widely used KDE, but also apply GBHT to anomaly detection to showcase a further application of GBHT.
\end{abstract}


\section{Introduction} \label{sec::Introduction}
Regarded as one of the most important tasks in unsupervised learning, density estimation aims at inferring the true distribution of targeted unknown variables through limited samples. 
While basic statistical analysis can be directly carried out on density functions \citep{scott2015multivariate}, density estimation is further regarded as an imperative cornerstone to more sophisticated tasks, such as anomaly detection \citep{nachman2020anomaly,zhang2018adaptive,amarbayasgalan2018unsupervised} and clustering \citep{chen2020fast,parmar2019redpc,ghaffari2019improved,jang2019dbscanpp}.

On the other hand, as one of the most successful algorithms over two decades \cite{buhlmann2003boosting}, boosting attracts more and more attention in researches on machine learning \citep{mathiasen2019optimal, cortes2019regularized, parnell2020snapboost, duan2020ngboost, cai2020boosted, suggala2020generalized}. 
However, when boosting method shows its power and strength in the field of supervised learning, few studies focus on unsupervised learning, especially on the density estimation problem.
Furthermore, previous attempts \citep{ridgeway2002looking, rosset2003boosting} focus more on methodology study instead of statistical theories.
To the best of our knowledge, there remains little understood of the theoretical advantage of boosting over its base learners from the statistical learning point of view.

Under such background, by combing the boosting framework \cite{rosset2003boosting} with the random histogram transforms \cite{lopez2013histogram,blaser2016random}, this paper aims to establish a new boosting algorithm called \emph{Gradient Boosting Histogram Transform} (\textit{GBHT}) for density estimation, which not only has satisfactory performance but also has solid theoretical foundations. To be specific, we adopt the \textit{Negative Log Likelihood} loss, which makes the boosting method, typically used in supervised learning tasks, available for density estimation which is an unsupervised problem. Moreover, through complete learning theory analysis, we for the first time provide theoretical supports to the benefit of the boosting procedure in the density estimation problem. GBHT starts with generating a random histogram transform consisting of random rotations, stretchings, and translations. {(The histogram transforms are i.i.d. generated at each iteration)}. Then the input space is partitioned into non-overlapping cells corresponding to the unit bin in the transformed space. On those cells, we obtain base learners where piecewise constant functions are applied. Then the iterative process is started with adding a sequence of random histogram transforms for minimizing empirical negative log-likelihood loss by a natural adaption of gradient descent boosting algorithm. Finally, after the iterative process, we inversely transform the partitioned space to the original and obtain the GBHT density estimator.

The contributions of this paper come from the model, theoretical, and experimental perspectives: 
\begin{itemize}

\item 
While majority studies of boosting focus on supervised learning, we exploit boosting to improve the accuracy in density estimation by taking an unsupervised loss function. 

\item 
From a learning theory point of view, we prove the fast convergence rates of GBHT with assumptions that the underlying density functions lie in the H\"{o}lder space $C^{0,\alpha}$. 

\item
To our best knowledge, we are the first to explain the strength of boosting density estimation from the theoretical point of view. To be specific, in the space $C^{1,\alpha}$, we show that HT density estimator obtains lower bound as $O(n^{-{2}/{2+d}})$, which turns out to be greater than the upper bound for GBHT $O(n^{-{2(1+\alpha)}/{4(1+\alpha)+d}})$. 

\item 
In experiments, we validate the performances of our algorithms through parameter analysis and real data comparisons. 
Moreover, we apply GBHT as part of a density-based anomaly detection algorithm, where the results show the promising compatibility of our GBHT.
\end{itemize}

\section{Related Works}

\textbf{Density Estimation.}
The best-known and most traditional density estimation methods are histogram density estimation (HDE) and kernel density estimation (KDE), while the former one is often criticized for its lack of smoothness and the latter one is found weak against outlier and local adaptivity. In order to solve these problems, partition-based methods \cite{klemela2009multivariate,liu2014multivariate,lopez2013histogram}, e.g.~decision tree-based algorithms \cite{ram2011density, criminisi2011decision, criminisi2013decision} have been taken into consideration. 
However, partition-based algorithms inherently suffer from boundary discontinuity, i.e. the density estimation of adjacent partition cells may not correspond on their shared boundary.
In this paper, inspired by histogram density estimation, we aim at solving the boundary discontinuity by aggregating random histogram transform density estimators with the help of boosting.

\textbf{Boosting.}
Boosting is a widely used learning technique in machine learning.
It boosts the performance of a base learner by combining multiple weak learners.
In boosting, the weak learner in each iteration learns from the distance between truth and the estimated one, e.g.~residuals in regression and wrong labels in classification. Based on these ideas, various boosting algorithms such as AdaBoost \cite{schapire1995decision,freund1997decision}, GBDT and GBRT \cite{friedman2001greedy}, and XgBoost \cite{chen2016xgboost} become popular. 

Despite its great success in supervised learning, very few studies focus on exploiting the effectiveness of boosting in unsupervised learning, especially in density estimation problems.
For instance, \citet{rosset2003boosting} considers boosting as a gradient descent search method, and transforms the density estimation problem into a supervised learning problem by rationally adjusting the loss function. 
\citet{ridgeway2002looking} brings EM algorithm in to conduct boosting density estimation. 
These authors suggest that more researches can be done with boosting for density estimation problems, since present researches about boosting density estimation focus mainly on methodology following the derivation process of gradient descent and none of the above-mentioned boosting works present a satisfactory explanation from the statistical optimization view. 

This paper aims at filling the blank in studies of boosting in unsupervised learning, and at providing sound theoretical analysis to explain why boosting can enhance the performance of its base learners for density estimation problems.

\section{Methodology} \label{sec::Method}

\subsection{Notations} \label{sub::notations}

Throughout this paper, we assume that $\mathcal{X} \subset \mathbb{R}^d$ is compact and non-empty. For any fixed $r > 0$, we denote $B_r$ as the centered hyper-cube of $\mathbb{R}^d$ with size $2r$, that is, $B_r := [-r, r]^d := \{ x = (x_1, \ldots, x_d) \in \mathbb{R}^d : x_i \in [-r, r], i = 1, \ldots, d \}$, and for any $r' \in (0, r)$, we write $B^+_{r,r'} := [-r + r', r - r']^d$. Recall that for $1 \leq p < \infty$, the $L_p$-norm of $x = (x_1, \ldots, x_d)$ is defined by $\|x\|_p := (|x_1|^p + \cdots + |x_d|^p)^{1/p}$, and the $L_{\infty}$-norm is defined by $\|x\|_{\infty} := \max_{i=1,\ldots,d} |x_i|$.

Throughout this paper, we use the notation $a_n \lesssim b_n$ and $a_n \gtrsim b_n$ to denote that there exist positive constant $c$ and $c'$ such that $a_n \leq c b_n$ and $a_n \geq c' b_n$, for all $n \in \mathbb{N}$. Moreover, for any $x \in \mathbb{R}$, let $\lfloor x \rfloor$ denote the largest integer less than or equal to $x$. In the sequel, the following multi-index notations are used frequently. For any vector $x=(x_i)^d_{i=1}\in \mathbb{R}^d$, we write $\lfloor x \rfloor:=(\lfloor x_i \rfloor )_{i=1}^d$, $x^{-1}:=(x_i^{-1})_{i=1}^d$, $\log (x):=(\log x_i)^d_{i=1}$, $\overline{x}=\max_{i=1,\ldots,d} x_i$, and  $\underline{x}=\min_{i=1,\ldots,d} x_i$.

\subsection{Negative Log Likelihood Loss} \label{sub::LSRegression}

Let $f$ be the underlying density function of an unknown probability measure $\mathrm{P}$ on $\mathcal{X}$. Based on a dataset $D := \{ x_1 \ldots, x_n \}$ consisting of i.i.d.~observations drawn from $\mathrm{P}$, our goal in the density estimation is to construct a measurable function $\hat{f} : \mathcal{X} \to [0, \infty)$ satisfying $\int_{\mathcal{X}}\hat{f}(x) \, dx = 1$ to approximate $f$ properly. To evaluate the quality of $\hat{f}$, we use the \textit{Negative Log Likelihood} loss $L : \mathcal{X} \times [0, \infty) \to \mathbb{R}$ defined by 
\begin{align} \label{eq::L}
L(x,\hat{f}) := - \log \hat{f}(x).
\end{align}
Then the risk is defined by $\mathcal{R}_{L,\mathrm{P}}(\hat{f}) := \int_{\mathcal{X}} L(x, \hat{f}) \, d\mathrm{P}(x)$ and the empirical risk is defined by $\mathcal{R}_{L,\mathrm{D}}(\hat{f}) := \frac{1}{n} \sum_{i=1}^n L(x_i, \hat{f}(x_i))$. The Bayes risk, which is the smallest possible risk with respect to $\mathrm{P}$ and $L$, is given by $\mathcal{R}_{L, \mathrm{P}}^{*}:=\inf \{\mathcal{R}_{L, \mathrm{P}}(\hat{f}) | \hat{f} : \mathcal{X} \to [0, \infty) \text{ measurable and } \int_{\mathcal{X}}\hat{f}(x) \, dx = 1 \}$. It is easy to verify that the $\hat{f}(x)$ that minimizes $\mathcal{R}_{L,\mathrm{P}}$ is indeed the true density. Therefore, it is reasonable to consider the framework that using gradient-based functional optimization algorithms to generate density estimators.

\subsection{Histogram Transform (HT)  for Density Estimation} \label{sub::histogram}

To give a clear description of one possible construction procedure of histogram transforms, we introduce a random vector $(R,S,b)$ where each element represents the rotation matrix, stretching matrix, and translation vector, respectively. 

To be specific, $R$ denotes the rotation matrix which is a real-valued $d \times d$ orthogonal square matrix with unit determinant, that is $R^{\top} = R^{-1}$ and $\det(R) = 1$; $S$ stands for the stretching matrix which is a positive real-valued $d \times d$ diagonal scaling matrix with diagonal elements $(s_i)_{i=1}^d$ that are certain random variables. Obviously, we have
$\det(S) = \prod_{i=1}^d s_i$. Moreover, we denote $s = (s_i)_{i=1}^d$, and the bin width vector defined on the input space is given by $h =  s^{-1}$.
The translation parameter $b \in [0,1]^d$ is a $d$ dimensional vector named translation vector. 
{
Different from rotation and stretching that make no changes to the centroid of data, translation alters the relative position of the transformed data and histogram partition grids. Since we take $h_i'=1$, where $h_i'$ denotes the bin with of the histogram partition in the transformed space, then if we select $b_i \geq 1$, $i \in [d]$, the same effect can be achieved by using $b_i-1$.
Thus we only need to consider $b_i \in [0,1]$, i.e. $b \in [0,1]^d$.
}

We define the histogram transform $H:\mathcal{X}\to \mathcal{X}$ by \begin{align}\label{equ::HT}
H(x) := R \cdot S \cdot x + b. 
\end{align}
Figure \ref{fig::RHT} illustrates two-dimensional examples of histogram transforms. The left subfigure is the original data and the other two subfigures are possible histogram transforms of the original sample space, with different rotating orientations and scales of stretching.

\begin{figure}[!h]
	\centering
	\includegraphics[width = 0.45\textwidth]{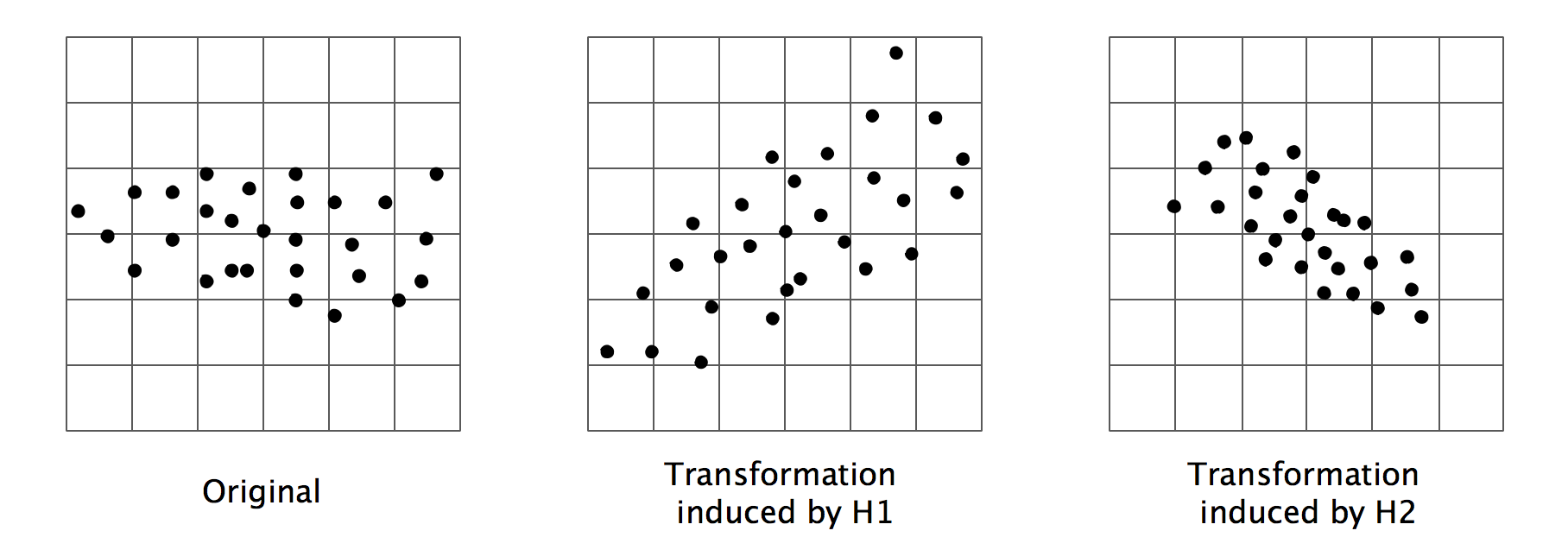}
	\centering
	\captionsetup{justification=centering}
	\caption{Two possible histogram transforms in $2$-D.}
	\label{fig::RHT}
\end{figure}

It is important to note that there is no point to consider the bin width {$h_i' \neq 1$} in the transformed space, since the same effect can be achieved by scaling the transformation matrix $H'$. Therefore, let $\lfloor H(x) \rfloor$ be the transformed bin indices, then the transformed bin is $A'_H(x) := \{ H(x') \ | \ \lfloor H(x') \rfloor = \lfloor H(x) \rfloor \}$ and the corresponding histogram bin containing $x \in \mathcal{X}$ in the input space is $A_H(x) := \{ x' \ | \ H(x') \in A'_H(x) \}$. 
{Note that here we use $A_H(x)$ and $A'_H(x)$ to denote the cells containing $x$ in the input space and the transformed space, respectively.}
We further denote all the bins induced by $H$ as $A_j' = \{ A_H(x) : x \in \mathcal{X} \}$ with the repetitive bin counted only once, and $\mathcal{I}_H$ as the index set for $H$ such that for $j\in \mathcal{I}_H$, we have $A_j' \cap B_r \neq \emptyset$. As a result, the set $\pi_H := \{ A_j \}_{j \in \mathcal{I}_H} := \{ A_j' \cap B_r \}_{j \in \mathcal{I}_H}$ forms a partition of $B_r$. For simplicity and uniformity of notations, in the sequel, we denote $\overline{h}_0 = \underline{s}_0^{-1}$ and $\underline{h}_0 = \overline{s}_0^{-1}$. 
{Then we show a uniform range of $h_i$, denoted as} $h_i \in [\underline{h}_0, \overline{h}_0] = [\overline{s}_0^{-1}, \underline{s}_0^{-1}]$, for $i=1,\ldots,d$.

Given a histogram transform $H$, the set $\pi_{H} = \{ A_j \}_{j\in \mathcal{I}_H}$ forms a partition of $B_r$. We consider the following function set $\mathcal{F}_H$ {containing piecewise constant density functions}
\begin{align}\label{equ::functionFn}
\mathcal{F}_H := \biggl\{ \sum_{j \in \mathcal{I}_H} c_j \eins_{A_j} \ \bigg| \ c_j \geq 0, \sum_{j \in \mathcal{I}_H} c_j \mu(A_j) = 1 \biggr\},
\end{align}
{
where $\eins_{A_j}(\cdot)$ denotes the indicator function, i.e. $\eins_{A_j}(x)=1$ when $x \in A_j$ and $0$ otherwise, and $\mu(\cdot)$ is the Lebesgue measure.
}
In order to constrain the complexity of $\mathcal{F}_H$, we penalize on the bin width $h := (h_i)_{i=1}^d$ of the partition $\pi_H$. Then the histogram transform (HT) density estimator can be produced by the regularized empirical risk minimization (RERM) over $\mathcal{F}_H$, i.e.
\begin{align}\label{equ::fD}
(f_{\mathrm{D},H},h^*)
= \argmin_{f \in \mathcal{F}_H, \, h \in \mathbb{R}^d} \Omega(h) + \mathcal{R}_{L,\mathrm{D}}(f),
\end{align}
where $\Omega(h) := \lambda \underline{h}_0^{-2d}$. It is worth pointing out that we adopt the isotropic penalty for each dimension rather than each elements $h_1, \ldots, h_d$ for simplicity of computation.

\subsection{Gradient Boosting Histogram Transform (GBHT) for Density Estimation} \label{sub::boosting}

In this work, we mainly focus on the boosting algorithm equipped with histogram transform density estimators as base learners since they are weak predictors and enjoy computational efficiency. Before we proceed, we need to introduce the function space that we are most interested in to establish our learning theory. Assume that $\{ H_t \}_{t=1}^T$ is an i.i.d.~sequence of histogram transforms drawn from some probability measure $\mathrm{P}_H$
and $\mathcal{F}_t := \mathcal{F}_{H_t}$, $t = 1, \ldots, T$, are defined as in \eqref{equ::functionFn}. 
Then we define the function space $E$ by
\small
\begin{align}\label{equ::En}
E := \biggl\{ f : B_r \to \mathbb{R} \, \bigg| \, f = \sum_{t=1}^T w_t f_t, \, f_t \in \mathcal{F}_t\ \text{s.t.}\ \sum^T_{t=1} w_i=1\biggr\}.
\end{align}
\normalsize

As is mentioned above, boosting methods may be viewed as iterative methods for optimizing a convex empirical cost function. To simplify the theoretical analysis, following the approach of \citet{blanchard2003rate}, we ignore the dynamics of the optimization procedure and simply consider minimizers of an empirical cost function to establish the oracle inequalities, which leads to the following definition.

\begin{definition}\label{def::RBHT}
Let $E$ be the function space \eqref{equ::En} and $L$ be the negative log-negative loss. Given $\lambda> 0$, we call a learning method that assigns to every $D \in (\mathcal{X} \times \mathcal{Y})^n$ a function $f_{\mathrm{D},\lambda} : \mathcal{X} \to \mathbb{R}$ such that
\begin{align}\label{equ::fdlambda}
(f_{\mathrm{D},\lambda}, h^*) 
= \argmin_{f \in E, \, h \in \mathbb{R}^d} \Omega(h) + \mathcal{R}_{L,\mathrm{D}}(f)
\end{align}
a gradient boosting histogram transform (GBHT) algorithm for density estimation with respect to $E$, where $\Omega(h):=\lambda \underline{h}_0^{-2d}$.
\end{definition}

The regularization term is added to control the bin width of the histogram transform, which has been discussed in Section \ref{sub::histogram}. In fact, it is equivalent to adding the $L_p$-norm of the base learners $f_t$, since they are piecewise constant functions on the cells with volume no more than $\overline{h}_0^d$.

With all these preparations, we now present the gradient boosting algorithm GBHT to solve the optimization problem \eqref{equ::fdlambda} in Algorithm \ref{alg::BHT}.

\begin{algorithm}[!h]
\caption{Gradient Boosting Histogram Transform (GBHT) for Density Estimation}
\label{alg::BHT}
\begin{algorithmic}
\STATE {\bfseries Input:} Training data $D := \{ x_1, \ldots, x_n \}$;
\\
\quad\quad\quad \ Bandwidth parameters $\underline{h}_0$, $\overline{h}_0$;
\\
\quad\quad\quad \ Number of iterations $T$.	
\STATE{\bfseries Initialization:} $F_0$ is set to be uniformly distributed on cells $A_j\in \pi_{H}$ satisfying $A_j\cap D\neq \emptyset$.
\\
\FOR{$t = 1$ {\bfseries to} $T$}
\STATE Set the sample weight $\omega_{t,i} = 1/F_{t-1}(x_i)$;
\STATE For random histogram transformation $H_{t}$ \eqref{equ::HT}:
\STATE Find $f_t = \argmax_{f\in \mathcal{F}_{t}} \sum_{i=1}^n \omega_{t,i} f(x_i)$;
\STATE Find $\alpha_t := \argmin_{\alpha} \sum_{i=1}^n -\log\big((1-\alpha)F_{t-1}(x_i) + \alpha f_t(x_i)\big)$;
\STATE Update $F_t = (1-\alpha_t)F_{t-1}+ \alpha_t f_t$;
\ENDFOR
\STATE {\bfseries Output:} $F_T$.
\end{algorithmic} 
\end{algorithm} 

The algorithm proceeds iteratively, that is, for $t=1,\ldots, T$, $F_t(x_i) = (1-\alpha_t)F_{t-1}+\alpha_tf_t$, where $F_t$ denotes the density estimator after $t$ iterations, $f_t \in \mathcal{F}_t$ denotes the $t$-th base learner, and $\alpha_t \in (0,1)$. Obviously we have $F_t = w_{t,0} F_0 + \sum_{j=1}^t w_{t,j}f_j$, where $w_{t,j} = (1-\alpha_t)\cdots(1-\alpha_{j+1})\alpha_j$ for $j=1,\ldots,t$, and $w_{t,0} = \prod_{j=1}^t(1-\alpha_j)$. By initiating $F_0 \in \mathcal{F}_0 := \mathcal{F}_H$, we have $F_t \in E$. Then we aim to search the base learner $f_t$ under partition $H_t$ and step size $\alpha_t$ to result in $F_t$ with lower empirical risk $\mathcal{R}_{L,\mathrm{D}}(F_t)$ in each iteration. In the $t$-th iteration, for every $\alpha_t \in (0,1)$, the minimization of $\mathcal{R}_{L,\mathrm{D}}(F_t)$ equals to the minimization of $\sum_{i=1}^n-\log(F_{t-1}(x_i)+\varepsilon_tf_t(x_i))$, where $\varepsilon_t=\alpha_t/(1-\alpha_t)$. Using Taylor expansion, we get
\begin{align*}
\sum_i &-\log(F_{t-1}(x_i) + \varepsilon_t f_{t}(x_i)) 
\nonumber \\
=& \sum_i -\log(F_{t-1}(x_i)) - \varepsilon_t\cdot  \omega_{t,i} f_t(x_i) + O(\varepsilon_t^2),
\end{align*}
where $\omega_{t,i}:= 1/F_{t-1}(x_i)$. For sufficiently small $\varepsilon_t$ (or $\alpha_t$), we can ignore the higher order term and find the maximum gradient $\max_{f_t\in \mathcal{F}_t} \sum_{i=1}^n \omega_{t,i}f_t(x_i)$. Then we determine the step size $\alpha_t$ by line search, which ensures that the updated $F_t$ remains to be a probability distribution.

It is worth mentioning that GBHT enjoys two advantages. First, the algorithm can be locally adaptive by applying random rotations, stretchings, and translations to the original input data. Regular density estimators such as KDE adopt uniform bandwidth, regardless of the fact that the local structures of real-world data usually vary from area to area. 
On the contrary, it is well known that boosting algorithms take local data structures into consideration by updating its vulnerable part in each iteration, and the adopted histogram transform catches exactly various local features of the input data. Thus, good combinations of random weak learners and the boosting procedure can lead to great local adaptivity.
Second, the boosting procedure brings smoothness to histogram-based density estimators, thanks to the randomness of base learners. Through iteration, GBHT adds more information obtained by the base learners into the boosting estimator, and it turns out to be the weighted average of all random base learners with different partition boundaries. As a result, it can be more smooth than regular histogram density estimators, which will also be theoretically verified in Section \ref{sec::mainresults} and experimentally validated by numerical simulations in Section \ref{sec::syn_exp}.

\section{Theoretical Results}\label{sec::mainresults}

Our theoretical analysis is built on the fundamental assumption on the smoothness of the underlying density function. Recall that a function $f : \mathcal{X} \to \mathbb{R}$ is $(k,\alpha)$-H\"{o}lder continuous, $\alpha \in (0, 1]$, $k \in \mathbb{N}_0$, if there exists a constant $c_L \in (0, \infty)$ such that 
\begin{align}
	\| \nabla^{\ell} f \| \leq c_L
\text{ for all } \ell \in \{ 1, \ldots, k \} \text{ and }
\end{align}
\begin{align}
	\| \nabla^k f(x) - \nabla^k f(x') \| \leq c_L \| x - x' \|^{\alpha}
\end{align} 
for all $x, x' \in B_r$. 
The set of such functions is denoted by $C^{k, \alpha} (B_r)$. Note that the functions contained in the space $C^{k,\alpha}$ with larger $k$ enjoy a higher level of smoothness. Throughout this paper, we make the following assumptions on the bin width $h$.

\begin{assumption}\label{assumption::h}
Let the bin width $h \in [\underline{h}_0, \overline{h}_0]$ and assume that there exists some constant $c_0 \in (0,1)$ such that $c_0 \overline{h}_0 \leq \underline{h}_0 \leq c_0^{-1}\overline{h}_0$. Moreover, if the bin width $h$ depends on the sample size $n$, that is, $h_n \in [\underline{h}_{0,n}, \overline{h}_{0,n}]$, we still have $c_{0} \overline{ {h}}_{0,n} \leq  \underline{ {h}}_{0,n} \leq c_{0}^{-1} \overline{ {h}}_{0,n}$.
\end{assumption}

Assumption \ref{assumption::h} indicates that the upper and lower bounds of the bin width h are of the same order.
In other words, we assume that under a certain partition, the extent of stretching in each dimension cannot vary too much.

Furthermore, to remove the boundary effect on the convergence rate, we denote $L_{\overline{h}_0}(x,t)$ as the negative log loss function restricted to $B^{+}_{R,\sqrt{d}\cdot \overline{h}_0}$, that is,
\begin{align}\label{equ::resloss}
L_{\overline{h}_0}(x,t) 
:= \eins_{B_{R,\sqrt{d} \cdot \overline{h}_0}^+}(x) L(x,t),
\end{align}
where $L(x,t)$ is the negative log loss.

\subsection{Convergence Rates for GBHT in $C^{0,\alpha}$} \label{sec::c0}

\begin{theorem}\label{thm::tree}
Let $f_{\mathrm{D},\lambda}$ be as in \eqref{equ::fdlambda} and the density function $f \in C^{0,\alpha}(B_r)$. Then for all $\tau > 0$ and for any $\delta \in (0, 1)$, there exists a constant $N_0$ such that for all $n \geq N_0$, there holds
\begin{align*}
\mathcal{R}_{L, \mathrm{P}}(f_{\mathrm{D},\lambda}) - \mathcal{R}_{L,\mathrm{P}}^* 
\lesssim  n^{-\frac{2\alpha}{(4-2\delta)\alpha+d}}
\end{align*}
with probability $\mathrm{P}^n \otimes \mathrm{P}_H$ at least $1 - 3e^{-\tau}$.
\end{theorem}

Theorem \ref{thm::tree} presents the fast convergence rates of the GBHT density estimator in the sense of ``with high probability'', which is a stronger claim than the convergence results ``in expectation''.
Moreover, convergence rates, a finite sample property of GBHT, also indicate the consistency of $\mathcal{R}_{L, \mathrm{P}}(f_{\mathrm{D},\lambda})$ when $n \to \infty$.

{
With the boosting procedure, the function space $E$ becomes more complicated, e.g. the number of cells increases and their shape becomes irregular. This will affect the VC dimension \citep{vapnik2015uniform} of the function space, and further enlarge the estimation error term. Thus the convergence rate of GBHT turns out to be suboptimal. However, a more complex function space will lead to a smaller approximation error, which means that our GBHT can better estimate smooth density functions than ordinary histograms.
}

\subsection{Convergence Rates for GBHT in $C^{1,\alpha}$} \label{sec::c1}

\begin{theorem}\label{thm::optimalForest}
Let $f_{\mathrm{D},\lambda}$ be as in \eqref{equ::fdlambda} and the density function $f \in C^{1,\alpha}(B_r)$. Moreover, let $L_{\overline{h}_0}(x,t)$ be the restricted negative log loss as in \eqref{equ::resloss}. 
Then for all $\tau > 0$ and $\delta \in (0, 1)$, there exists a constant $N_1$ such that for all $n \geq N_1$, by choosing
$T_n \gtrsim n^{2\alpha/(2(1+\alpha)(2-\delta)+d)}$, there holds
\begin{align} \label{UpperBoundEnsemble}
\mathcal{R}_{L_{\overline{h}_0},\mathrm{P}}(f_{\mathrm{D},\lambda}) - \mathcal{R}_{L_{\overline{h}_0},\mathrm{P}}^* 
\lesssim n^{-\frac{2(1+\alpha)}{2(1+\alpha)(2-\delta)+d}}
\end{align}
with probability $\mathrm{P}^n$ not less than $1 - 4 e^{-\tau}$ in expectation with respect to $\mathrm{P}_H$.
\end{theorem}

In Theorem \ref{thm::optimalForest}, the excess risk decreases as $T_n$ grows at first, and when $T_n$ achieves a certain level, the algorithm achieves the best convergence rate.
Moreover, comparing with Theorem \ref{thm::tree}, when the underlying density function turns more smooth, GBHT achieves a better convergence rate with $f \in C^{1,\alpha}(B_r)$ than that with $f \in C^{0,\alpha}(B_r)$, where a relatively large $T_n$ helps the density estimator to achieve asymptotic smoothness.

\subsection{Lower Bound for HT Density Estimation in $C^{1,\alpha}$}

\begin{theorem}\label{thm::loglower}
Let $f_{\mathrm{D},H}$ be as in \eqref{equ::fD} and suppose that the density function $f \in C^{1,\alpha}(B_r)$. 
Then there exists a constant $N_2$ such that for all $n \geq N_2$, there holds
\begin{align}\label{LowerBoundSingle}
\sup_{f\in C^{1,\alpha}}\mathcal{R}_{L,\mathrm{P}}(f_{\mathrm{D},H}) -  \mathcal{R}_{L,\mathrm{P}}^*
\gtrsim n^{-\frac{2}{2+d}},
\end{align}
in expectation with respect to $\mathrm{P}^n\otimes \mathrm{P}_H$.
\end{theorem}
Recall that in Theorem \ref{thm::optimalForest}, as $n\to\infty$, the upper bound for our GBHT attains asymptotically convergence rate which is slightly faster than $n^{-2(1+\alpha)/(4(1+\alpha)+d)}$. 
When comparing Theorem \ref{thm::loglower} with Theorem \ref{thm::optimalForest}, we find that for any $\alpha\in (0,1]$, if $d \geq 2(1+\alpha)/\alpha$, the upper bound of the convergence rate \eqref{UpperBoundEnsemble} for GBHT turns out to be smaller than the lower bound \eqref{LowerBoundSingle} for HT density estimators, which explains the benefits of the boosting procedure from the perspective of convergence rates.

\section{Numerical Experiments} \label{sec::numerical_experiments}

\subsection{Generation Methods of Histogram Transforms}\label{sec::H_generation}

Here we describe a practical method for the construction of histogram transforms we are confined to in this study. Starting with a $d \times d$ square matrix $M$, consisting of $d^2$ independent univariate standard normal random variates, a Householder $Q R$ decomposition is applied to obtain a factorization of the form $M = R \cdot W$, with an orthogonal matrix $R$ and an upper triangular matrix $W$ with positive diagonal elements. The resulting matrix $R$ is orthogonal by construction and can be shown to be uniformly distributed. Unfortunately, if $R$ does not feature a positive determinant then it is not a proper rotation matrix according to the definition of $R$. In this case, we can change the sign of the first column of $R$ to construct a new rotation matrix $R^+$.


We apply the well-known Jeffreys prior for scale parameters \cite{jeffreys1946invariant}. To be specific, we draw $\log(s_i)$ from the uniform distribution over intervals $[\log(\underline{h}_0),\log(\overline{s}_0)]$. Recall that $h=s^{-1}$ stands for the bin width vector measured in the input space, we choose $\underline{s}_0$ and $\overline{s}_0$, recommended by \cite{lopez2013histogram}, as $\widehat{h} = 3.5 \sigma n^{-1/(2+d)}$, where $\sigma := \sqrt{\text{trace}(V)/d}$ is the standard deviation defined by $V := \frac{1}{n-1} \sum^n_{i=1} (x_i-\overline{x})(x_i-\overline{x})^{\top}$ and $\overline{x}:=\frac{1}{n}\sum_{i=1}^n x_i$. Then we can transform the bin width vector to obtain this scale parameter $\widehat{s}=(\widehat{h})^{-1}=(3.5\sigma)^{-1} n^{\frac{1}{2+d}}$, which can be further refined as
\begin{align*}
\log (\underline{s}_0) := s_{\min} + \log(\widehat{s}),
\quad 
\log(\overline{s}_0) := s_{\max} + \log(\widehat{s}),
\end{align*}
where $s_{\min} < s_{\max}$ are tunable parameters.

The translation vector $b$ is drawn from the uniform distribution over the hypercube $[0,1]^d$.

\subsection{Evaluation Criteria} \label{subsec::criteria}

\textbf{Mean absolute error (\textit{MAE}).} The first criterion of evaluating the accuracy of density estimator is the mean absolute error, defined by $\textit{MAE}(\widehat{f}) = \frac{1}{M} \sum_{j=1}^M |\widehat{f}(x_j) - f(x_j)|$, where $x_1, \ldots, x_M$ are test samples. It is used in synthetic data experiments where the true density function is known.

\textbf{Average negative log-likelihood (\textit{ANLL}).} Another effective measure of estimation accuracy, especially when facing real data and the true density function is unknown, is the average negative log-likelihood, defined by
$\textit{ANLL}(\widehat{f}) = - \frac{1}{M} \sum_{j=1}^M \log \widehat{f}(x_j)$, where $\widehat{f}(x_j)$ represents the estimated probability density for the test sample $x_j$ and $M$ is the size of test samples. Note that the lower the \textit{ANLL} is, the better estimation we obtain.

\subsection{Empirical Understandings} \label{sec::syn_exp}

In this part, we conduct simulations concerning GBHT for density estimation. 
Based on several synthetic datasets, we show the power of boosting procedure through simulations, and we illustrate a possible explanation for the enhancement in accuracy, i.e. the asymptotic smoothness achieved.
Then we study a pair of important parameters for histogram transforms, $s_{\min}$ and $s_{\max}$.

\subsubsection{Synthetic Data Settings}\label{sec::synthetic_settings}


We base the simulations on four different types of synthetic distributions, each with dimension $d \in \{2, 5, 7\}$, respectively. The premise of constructing data sets is that we assume that the components $X_i \sim f_i$, $i = 1 \ldots, d$, of the random vector $X = (X_1, \ldots, X_d)$ are independent of each other. 
To be specific, Type I density function, representing a bimodal Gaussian distribution, enjoys high order of smoothness, while those for Types II and III are not continuous. Moreover, Types II and III represent density functions with bounded support and unbounded support, respectively. Finally, Type IV represents the case where the marginal distributions of each dimension are not identical.
More detailed descriptions and visual illustrations are shown in Section \ref{DesSynData} of the supplementary material.

In the following experiments, we generate $2,000$ and $10,000$ i.i.d samples as training and testing data respectively from each type of synthetic datasets, and each with dimension $d \in \{ 2, 5, 7 \}$.

\subsubsection{The Power of Boosting}\label{sec::subsec::parameter}

To show the behavior of $T$,
we carry out the experiments with $T \in \{1, 5, 10, 20, 50, 100, 500, 1000\}$, and the other two hyper-parameters are chosen by 3-fold cross-validation. We pick $s_{\min}$ from the set $\{-3+0.5k,k=0,\ldots,12\}$ and $s_{\max}-s_{\min}$ is chosen from the set $\{0.5+0.5k,k=0,\ldots,5\}$. For each $T$ we repeat this procedure for 10 times.

\begin{figure}[!h]
\centering
\begin{minipage}{0.4\linewidth}
\centering
\includegraphics[width=\textwidth, trim= 0 100 0 100, clip]{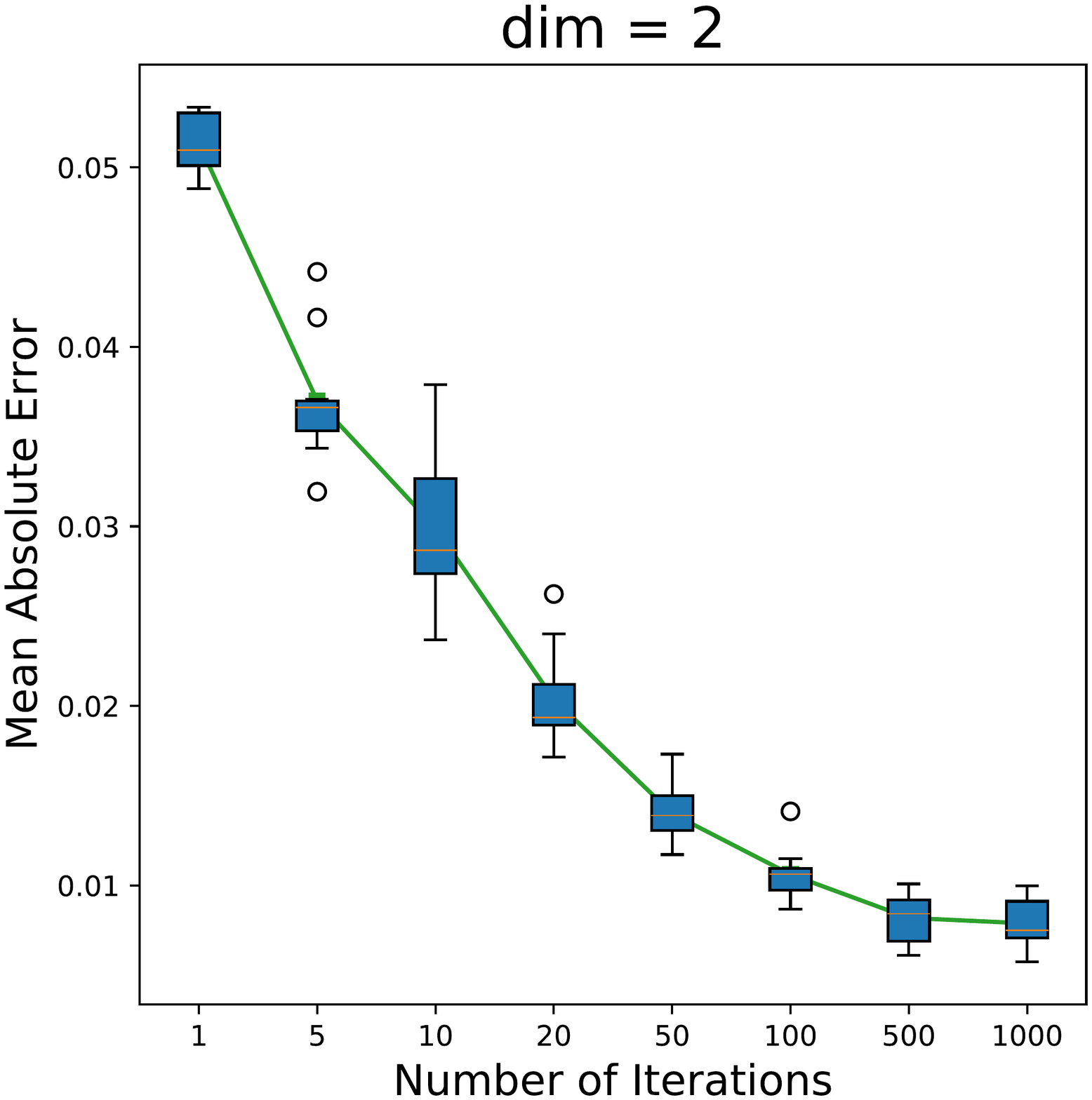}
\end{minipage}
\hspace{0.10in}
\begin{minipage}{0.4\linewidth}
\centering
\includegraphics[width=\textwidth, trim= 0 100 0 100, clip]{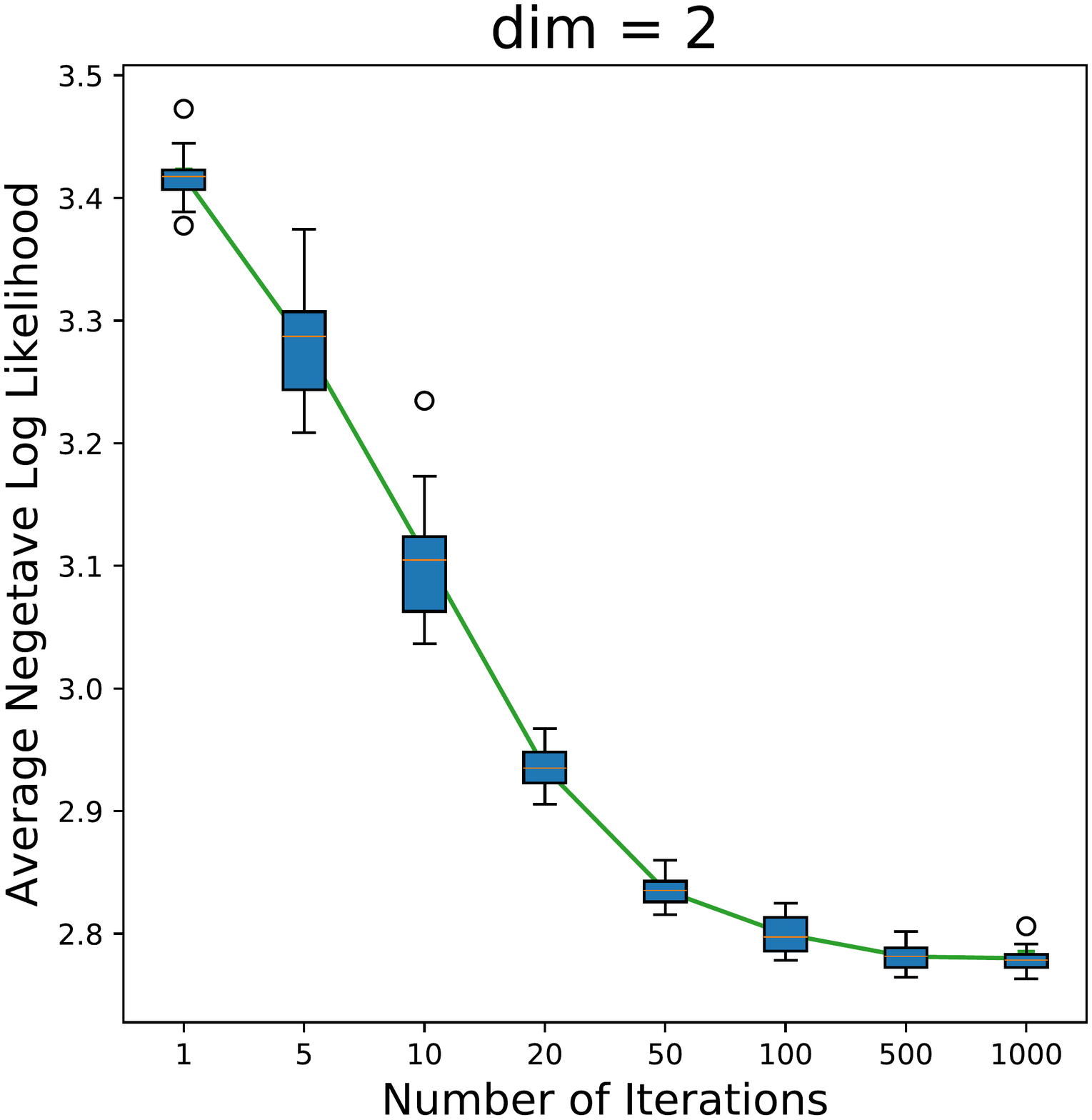}
\end{minipage}
\begin{minipage}{0.4\linewidth}
\centering
\includegraphics[width=\textwidth, trim= 0 100 0 100, clip]{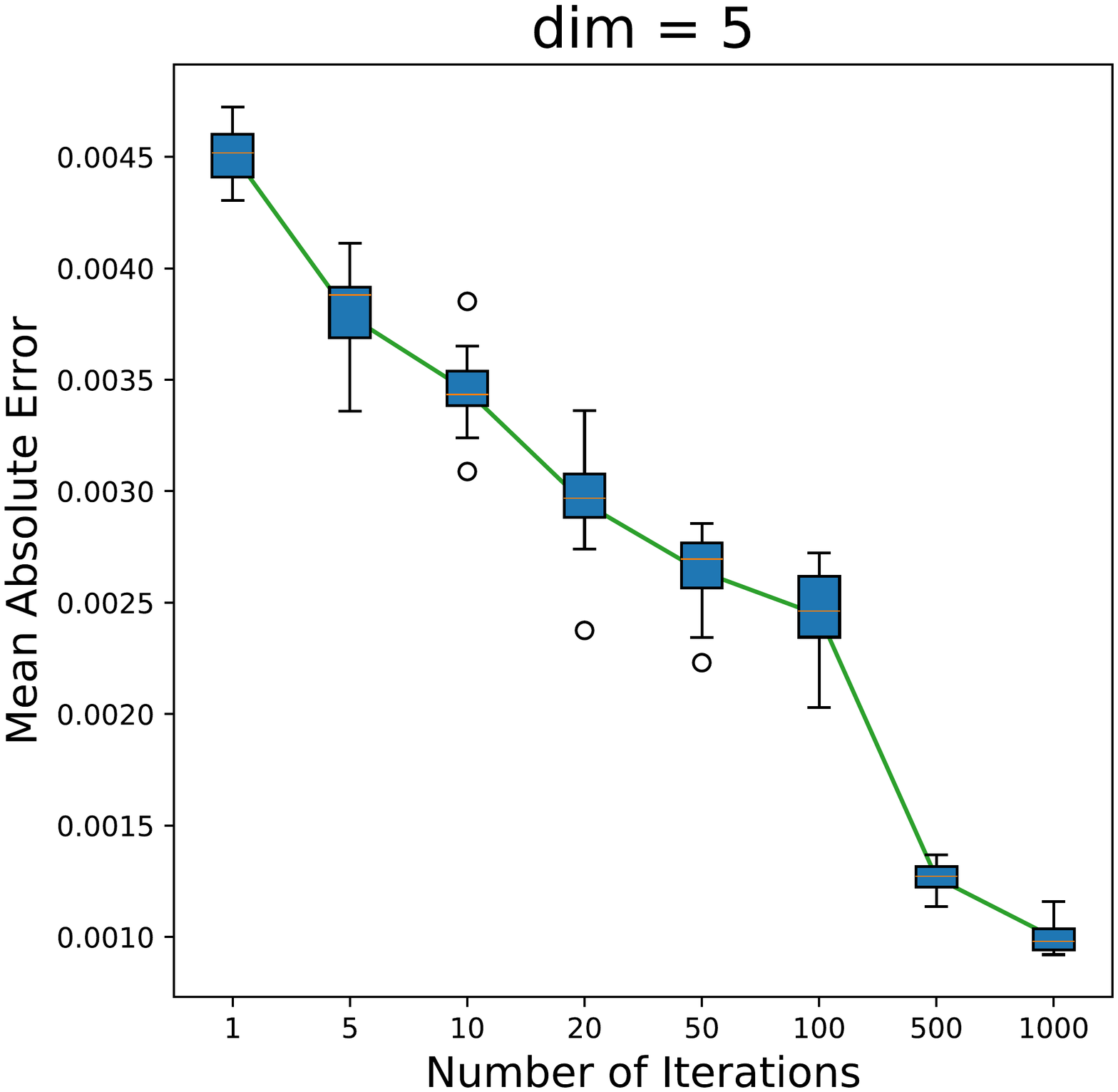}
\end{minipage}
\hspace{0.10in}
\begin{minipage}{0.4\linewidth}
\centering
\includegraphics[width=\textwidth, trim= 0 100 0 100, clip]{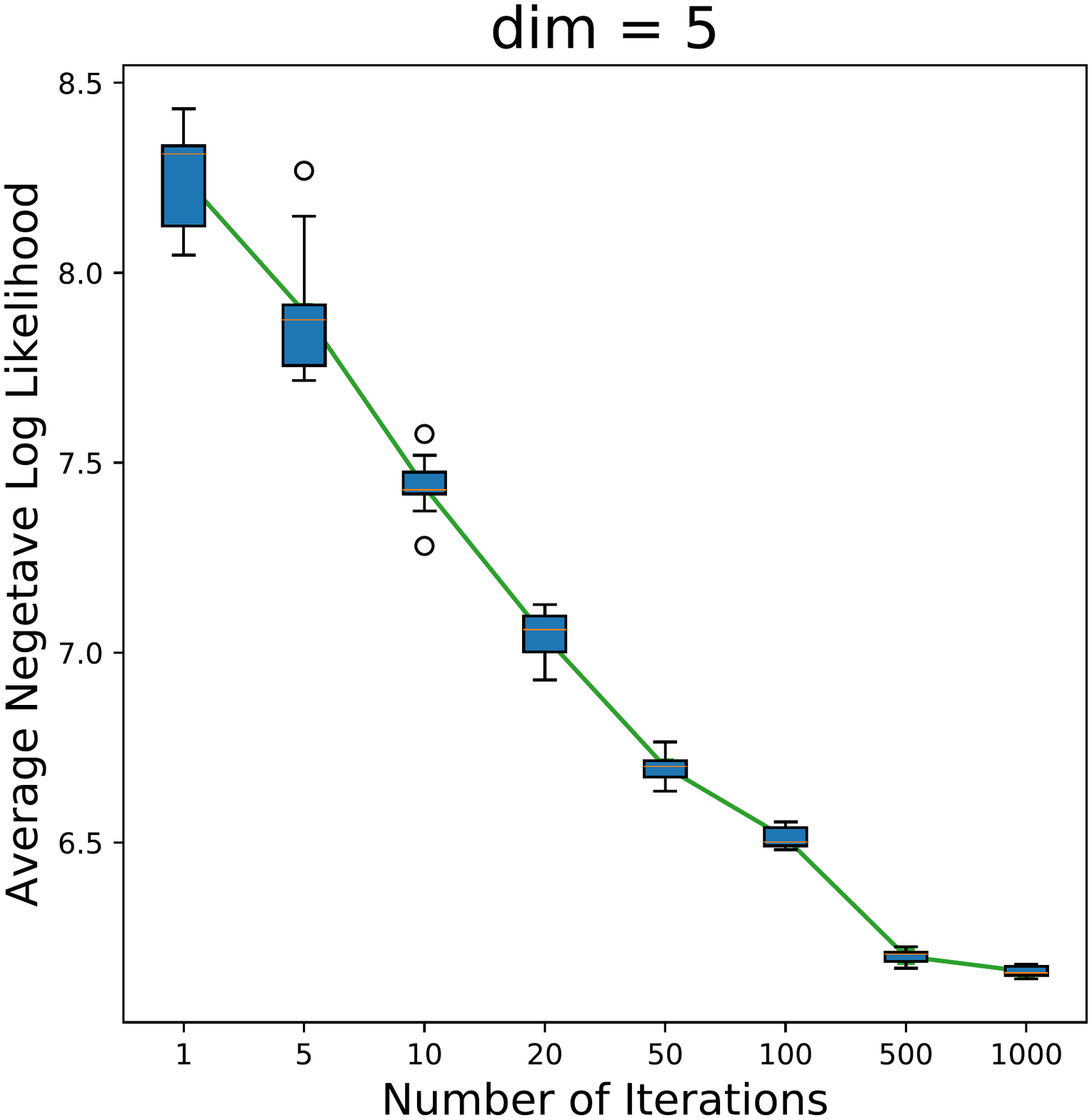}
\end{minipage}
\caption{
The study of parameter $T$ on GBHT of Type I synthetic distribution,
where the first row illustrates the low-dimensional results with dimension $d=2$,
and the second row indicates the high-dimensional results with dimension $d=5$.
The left column indicates how \textit{MAE} varies along parameters $T$, 
and the right column shows the variation of \textit{ANLL}.}
\label{fig::ParameterStudy_T}
\end{figure}

As can be seen in Figure \ref{fig::ParameterStudy_T}, 
as $T$ grows, the accuracy performance of GBHT (both \textit{MAE} and \textit{ANLL}) first enhances dramatically when $T$ grows from $1$ to $1,000$, but as $T$ continues to grow, a steady state will be reached.
This coincides with Theorem \ref{thm::optimalForest}, where the convergence rate attains the optimum when $T_n$ is greater than a certain value.
Moreover, fewer iterations are required to make GBHT convergence when the dimension of input space is lower.
A large number of iterations lead to a more accurate model but bring about the additional burden of computation.

For a possible explanation of the enhancement in estimation accuracy under the boosting procedure, we conduct simulations to show that GBHT achieves asymptotic smoothness with $T$ increasing. 
For the sake of more clear visualization, we utilize a toy example with $2,000$ samples i.i.d. generated from the one-dimensional standard normal distribution, and use GBHT to conduct density estimation, where the number of trees $T$ is set to $1, 5, 20, 50$, respectively.

\begin{figure}[!h]
	\centering
	\subfigure[$T=1$.]{
		\begin{minipage}{0.45\linewidth}
			\centering
			\includegraphics[width=\textwidth]{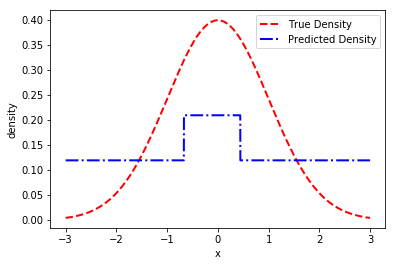}
		\end{minipage}
		\label{T=1}
	}\hspace{0.08in}
	\subfigure[$T=5$.]{
		\begin{minipage}{0.45\linewidth}
			\centering
			\includegraphics[width=\textwidth]{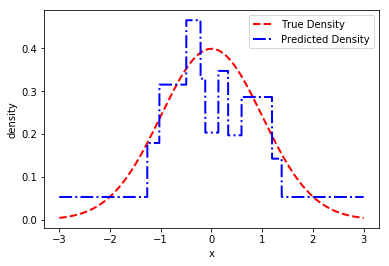}
		\end{minipage}
		\label{T=5}
	}
	\subfigure[$T=20$.]{
		\begin{minipage}{0.45\linewidth}
			\centering
			\includegraphics[width=\textwidth]{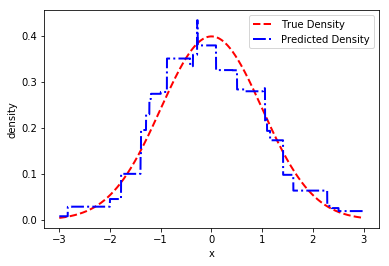}
		\end{minipage}
		\label{T=20}
	}\hspace{0.08in}
	\subfigure[$T=50$.]{
		\begin{minipage}{0.45\linewidth}
			\centering
			\includegraphics[width=\textwidth]{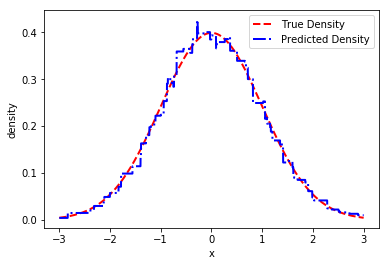}
		\end{minipage}
		\label{T=50}
	}
	\caption{
		The study of parameter $T$ on GBHT of the Standard Normal distribution. The red line represents the underlying density while the blue one represents density estimator returned by GBHT.}
	\label{fig::Iteration}
\end{figure}

From Figure \ref{fig::Iteration} we see that with $T = 1$, the base estimator turns out to be a step function with discontinuous boundaries, and the estimation is far from satisfactory.  Nevertheless, as the iteration $T$ increases, the boosting estimator becomes more continuous and smooth with the corresponding accuracy enhancing greatly. With $T = 50$, our GBHT is nearly smooth and achieves high estimation accuracy.

\subsubsection{Parameter Analysis}
\label{sec::subsec::parahistogram}

Here we mainly conduct experiments concerning the parameters of histogram transforms, namely the lower and upper scale parameters $s_{\min}, s_{\max}\in \mathbb{R}$. To this end, for the sake of clear visualization, we consider the Type I synthetic dataset of 1 dimension to see how these parameters affect the performance of GBHT.

Recall that the scale parameters $s_{\min}$ and $s_{\max}$ of the stretching matrix $S$ control the size of histogram bins. Smaller bins are required for the regions with complex structures of the density function while those with simple structure calls for larger bins. A narrower range of bin size is accommodated to cope with the varying scales while preserving a homogeneous structure. 
We conduct experiments over four pairs of scale parameters $(s_{\min},s_{\max})\in\{(-2.5,-1.5),(-2,-1),(-1.5,-0.5),(-1,0)\}$. We select $T=500$ to make the density estimator convergence with sufficient boosting iterations.

As is shown in Figure \ref{fig::ParameterStudy_h}, lower values of these parameters {(larger bin width)} lead to a coarser approximation of the underlying density function, which results in the loss of precision. Figure \ref{fig::studyofpara1_0}
implies that the density estimator is underfitting when the bin width is too large. On the contrary, if the bin width is too small, then there are few samples lying in most of the histogram bins and thus overfitting occurs as shown in Figure \ref{fig::studyofpara1_c}. Therefore, it is of great importance to choose $s_{\min}$ and $s_{\max}$ properly.

\begin{figure}[!h]
	\centering
	\subfigure[{$(-3, -2)$.}]{
		\begin{minipage}{0.45\linewidth}
			\centering
			\includegraphics[width=\textwidth]{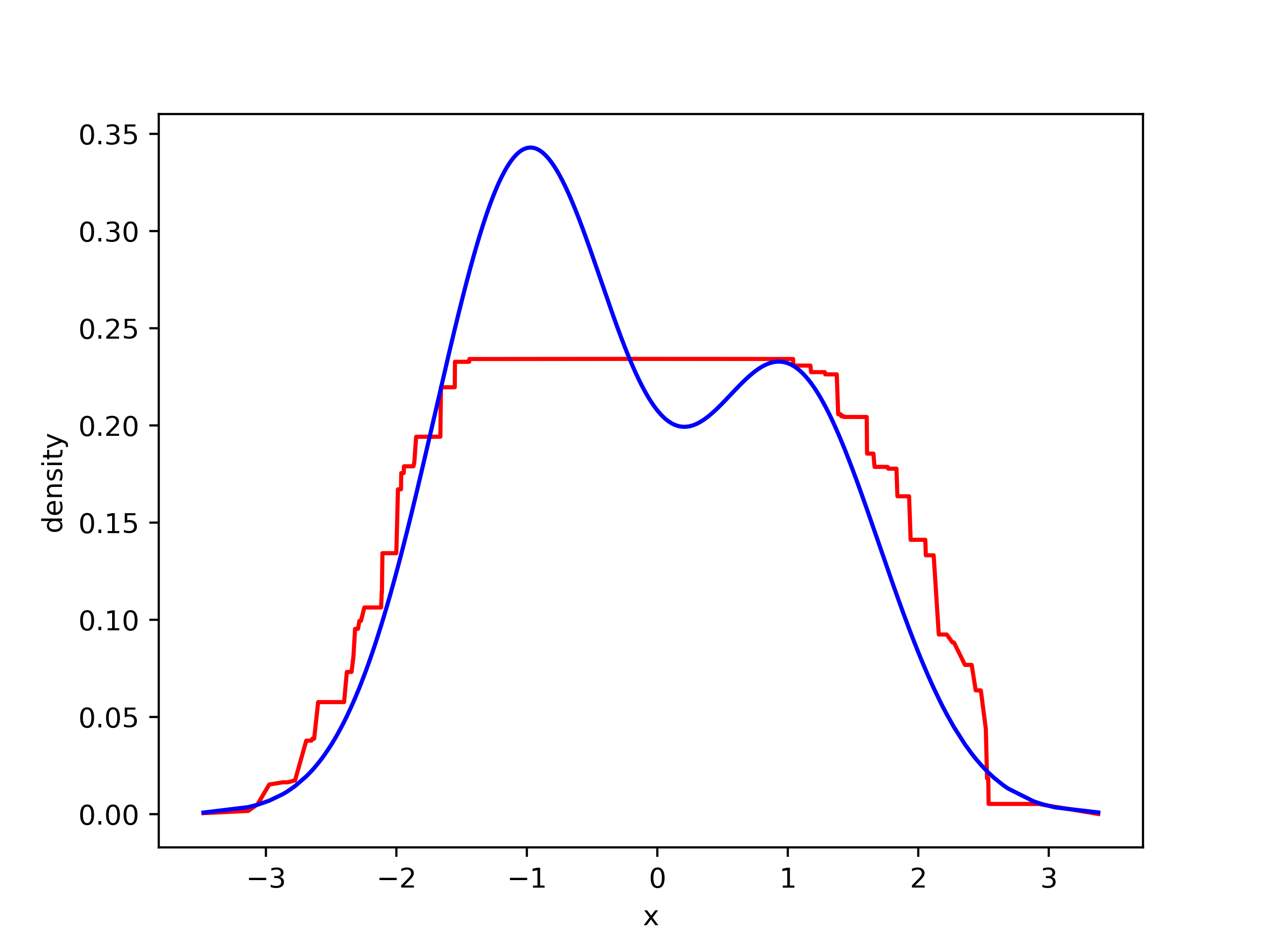}
		\end{minipage}
		\label{fig::studyofpara1_0}
	}\hspace{0.08in}
	\subfigure[$(-2.5, -1.5)$.]{
		\begin{minipage}{0.45\linewidth}
			\centering
			\includegraphics[width=\textwidth]{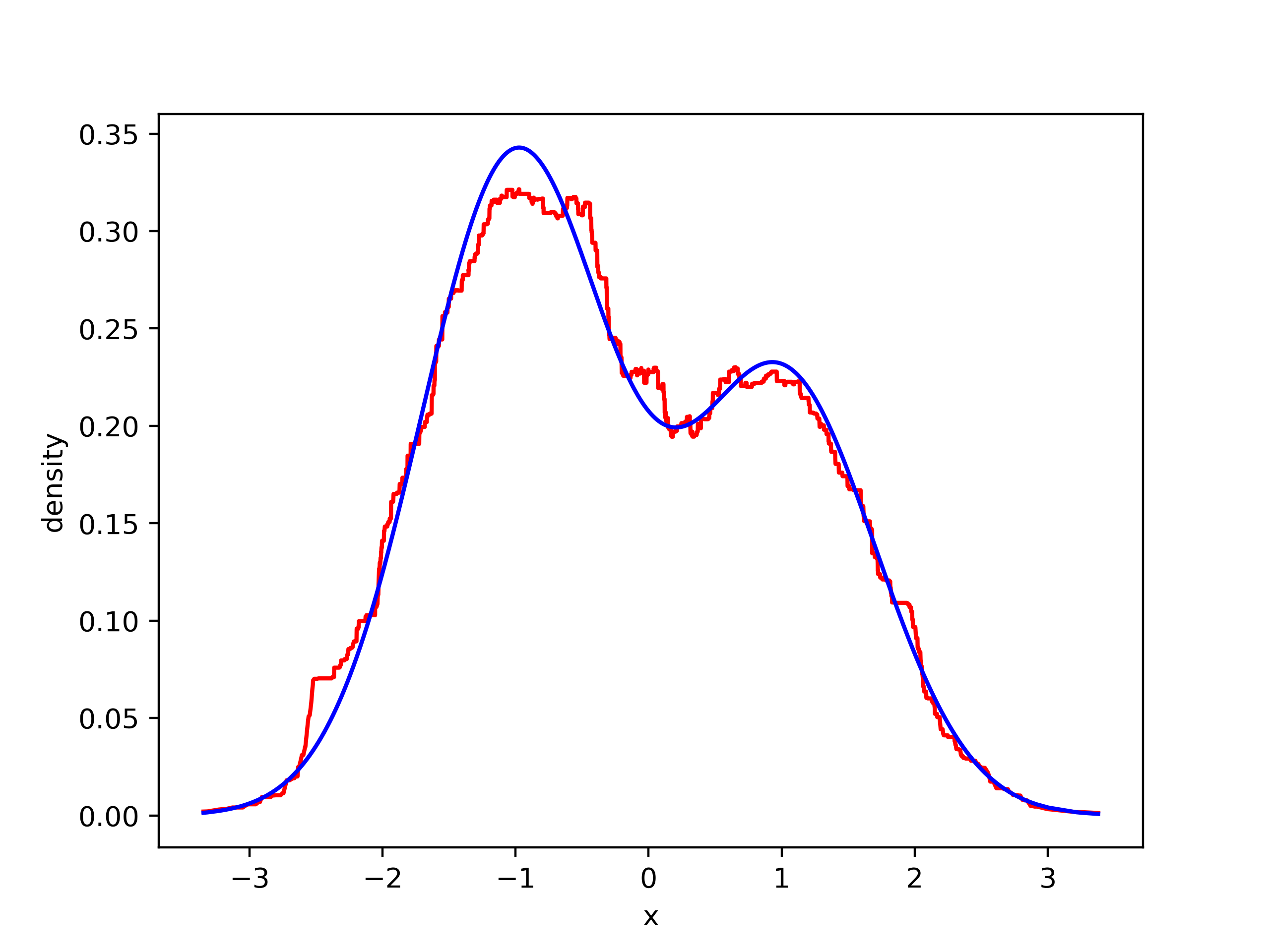}
		\end{minipage}
		\label{fig::studyofpara1_a}
	}\hspace{0.08in}
	\subfigure[$(-2, -1)$.]{
		\begin{minipage}{0.45\linewidth}
			\centering
			\includegraphics[width=\textwidth]{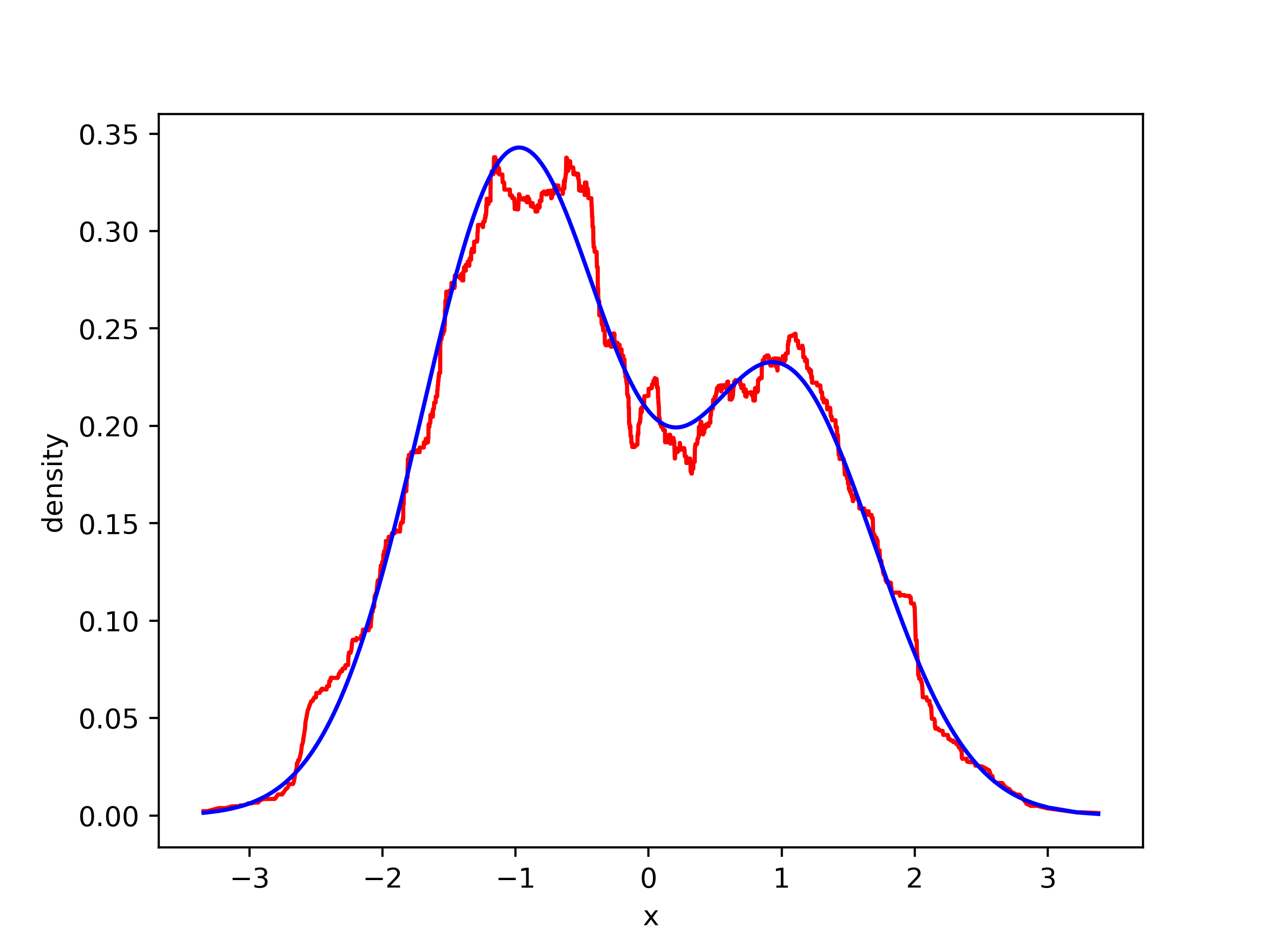}
		\end{minipage}
		\label{fig::studyofpara1_b}
	}
	\subfigure[$(-1.5, -0.5)$.]{
		\begin{minipage}{0.45\linewidth}
			\centering
			\includegraphics[width=\textwidth]{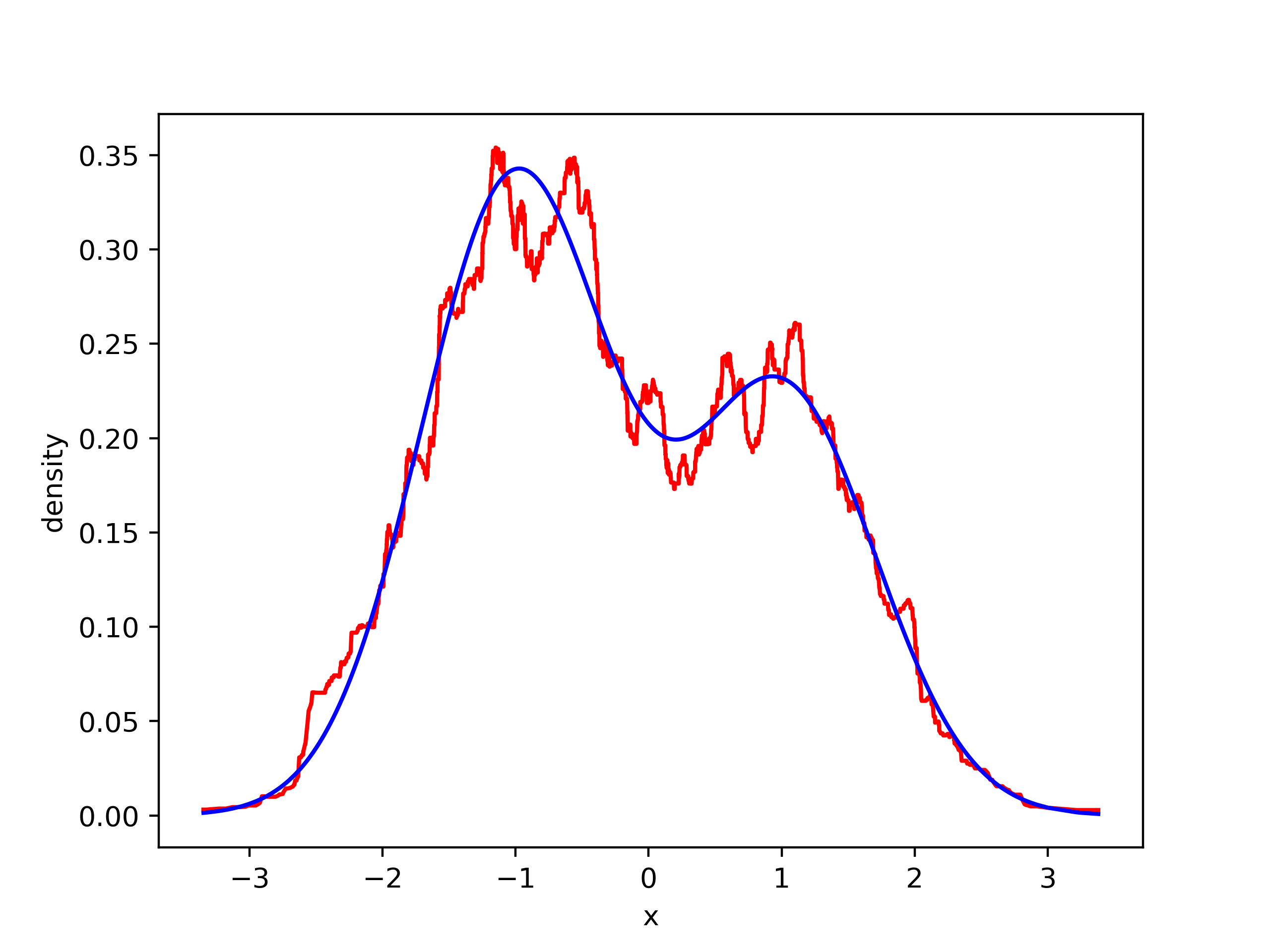}
		\end{minipage}
		\label{fig::studyofpara1_c}
	}
	\caption{
		The study of parameter $s_{\min}$ and $s_{\max}$ on GBHT of the Type I synthetic distribution. The red line represents the density estimator returned by GBHT algorithm while the blue one represents the underlying density function. And the tuples in subtitle represent $(s_{\min},s_{\max})$.}
	\label{fig::ParameterStudy_h}
\end{figure}

\subsection{Performance Comparisons}

In this section, we conduct performance comparisons on both synthetic and real datasets. 
{
Recall that both our theoretical results (shown in Theorems \ref{thm::optimalForest} and \ref{thm::loglower}) and empirical illustrations (shown in Figure \ref{fig::Iteration}) demonstrate that boosting improves the performance of histogram-based methods by enhancing the smoothness of the estimator. 
Therefore, we compare our GBHT with the kernel density estimator (KDE) which enjoys high order of smoothness.
We also compare our GBHT with MIX \citep{ridgeway2002looking}, a boosting method for density estimation using mixtures.
We also consider the histogram density estimator (HDE), which can be viewed as a special case of our GBHT when $T=1$ and $H=I$ (identity matrix). 
We run HDE on synthetic datasets with the bin width chosen by Sturges' rule \citep{sturges1926choice}.
}

\subsubsection{Synthetic Data Comparisons}

\begin{table*}[!h] 
	\centering
	\caption{Average \textit{ANLL} and \textit{MAE} over simulated datasets}
	\label{tab::simudata}
		\begin{tabular}{c c cc cc cc cc}
			\toprule
			\multirow{2}*{$d$} & \multirow{2}*{Method} 
			& \multicolumn{2}{c}{Type I}
			& \multicolumn{2}{c}{Type II}
			& \multicolumn{2}{c}{Type III}
			& \multicolumn{2}{c}{Type IV}
			\\
			\cline{3-10}
			& & \textit{ANLL} & \textit{MAE} & \textit{ANLL} & \textit{MAE} & \textit{ANLL} & \textit{MAE} & \textit{ANLL} & \textit{MAE}
			\\
			\midrule
			\multirow{4}*{$5$} & GBHT (Ours) & $\mathbf{6.26}$  & $\mathbf{2.41e}$-$\mathbf{3}$ & $\mathbf{-0.80}$ & $\mathbf{10.31}$ & $\mathbf{8.23}$ & $\mathbf{6.61e}$-$\mathbf{4}$ & $\mathbf{3.85}$ & $\mathbf{0.14}$
			\\
			&  KDE & $6.33$ & $2.36e$-$3$ & $-0.32$ & $12.40$ & $8.65$ & $8.27e$-$4$ & $3.86$ & $0.15$ \\
			& {MIX} &  $6.53$ & $3.08e$-$3$ & $1.82$ & $13.91$ & $9.64$ & $9.54e$-$4$ & $5.35$ & $0.14$ \\
			& {HDE} &  $9.33$ & $4.86e$-$3$ & $10.17$ & $19.70$ & $10.77$ & $1.33e$-$3$ & $6.09$ & $0.17$ \\
			\hline
			\multirow{4}*{$7$} & GBHT (Ours) & $\mathbf{8.36}$ & $\mathbf{4.33e}$-$\mathbf{4}$ & $\mathbf{-0.45}$ & $\mathbf{34.91}$ & $\mathbf{10.81}$ & $\mathbf{5.30e}$-$\mathbf{5}$ & $\mathbf{5.10}$ & $\mathbf{0.18}$ 
			\\
			& KDE &  $8.77$ & $5.13e$-$4$ & ${0.03}$ & $40.74$ & $12.48$ & $6.05e$-$5$ & $5.16$ & $0.18$ \\
			& {MIX} &  $8.65$ & $5.38e$-$4$ & $2.61$ & $42.13$ & $11.34$ & $6.32e$-$5$ & $7.02$ & $0.19$ \\
			& {HDE} &  $11.35$ & $1.45e$-$3$ & $11.48$ & $73.97$ & $11.49$ & $1.05e$-$4$ & $9.88$ & $0.20$ \\
			\bottomrule	
		\end{tabular}
	\begin{tablenotes}
		\item{~~~~~~~~~~~~~~*} The best results are marked in \textbf{bold}.
	\end{tablenotes}
\end{table*}

Following the experimental settings in Section \ref{sec::syn_exp}, we conduct empirical comparisons between GBHT and the prevailing KDE to further demonstrate the desirable performance of GBHT under synthetic datasets.
Table \ref{tab::simudata} records average \textit{ANLL} and \textit{MAE} over simulation data sets for KDE and GBHT with $T=1,000$. For higher dimensions $d=5$ and $d=7$, our GBHT always outperforms KDE in terms of \textit{ANLL} and \textit{MAE}.

\subsubsection{Real Data Comparisons} \label{sec::subsec::realdata}

\begin{table*}[!h] 
	\centering
	\caption{Average \textit{ANLL} over real data sets}
	\label{tab::Realdata}
		\begin{tabular}{lcrrr|lcrrr}
			\toprule
			Datasets & $d'$ & GBHT \hspace{0mm} & KDE \hspace{1mm} & {MIX} \hspace{1mm} & Datasets & $d'$ & GBHT \hspace{0mm} & KDE \hspace{1mm} & {MIX} \hspace{1mm} \\
			\midrule
			\multirow{8}*{Adult}  & \multirow{2}*{2}  & $\mathbf{-1.2371}$ &  $-0.7402$ & $1.3572$
			& \multirow{8}*{Diabetes}
			& \multirow{2}*{1}  & $\mathbf{-0.7057}$ &  $-0.2627$ & $0.7131$ \\
			& & ${(0.0312)}$ & $(0.0027)$ & $(0.0050)$ & & & ${(0.1253)}$ & $(0.0111)$ & $(0.0186)$ \\
			& \multirow{2}*{4}  & $\mathbf{-1.9312}$ &  $-0.3075$ & $1.7609$ &
			& \multirow{2}*{3}  & $\mathbf{-1.5982}$ &  $-0.4042$ & $0.5193$ \\
			& & ${(0.0667)}$ & $(0.0032)$ & $(0.0059)$ & & & ${(0.1011)}$ & $(0.0403)$ & $(0.0600)$ \\
			& \multirow{2}*{8}  & $\mathbf{-5.5922}$ &  $-2.2970$ & $0.8562$ &
			& \multirow{2}*{4}  & $\mathbf{-1.8605}$ &  $-0.8353$ & $0.0403$ \\
			& & ${(0.1097)}$ & $(0.0108)$ & $(0.3183)$ & & & ${(0.1424)}$ & $(0.0773)$ & $(0.0771)$ \\
			& \multirow{2}*{10}  & $\mathbf{-6.0740}$ &  $-3.4372$ & $-0.8975$ &
			& \multirow{2}*{6}  & $\mathbf{-2.6134}$ &  $-1.9693$ & $-1.2393$ \\
			& & ${(0.1044)}$ & $(0.0110)$ & $(0.0982)$ & & & ${(0.2310)}$ & $(0.1550)$ & $(0.1087)$ \\
			\hline
			\multirow{8}*{Australian}
			& \multirow{2}*{2}  & $\mathbf{-0.7966}$ &  $1.3155$ & $1.8577$ &
			\multirow{8}*{Ionosphere}
			& \multirow{2}*{3}  & $\mathbf{2.8681}$ &  $2.9544$ & $3.4988$ \\
			& & ${(0.0904)}$ & $(0.0234)$ & $(0.0263)$ & & & ${(0.0917)}$ & $(0.0423)$ & $(0.0776)$ \\
			& \multirow{2}*{4}  & $\mathbf{-5.8510}$ &  $0.8518$ & $3.0147$ &
			& \multirow{2}*{10}  & $\mathbf{4.1625}$ &  $4.6447$ & $-$ \\
			& & ${(0.2947)}$ & $(0.0291)$ & $(0.0370)$ & & & ${(0.2150)}$ & $(0.4448)$ & $-$ \\
			& \multirow{2}*{8}  & $\mathbf{-3.7957}$ &  $0.6879$ & $2.6446$ &
			& \multirow{2}*{17}  & $\mathbf{3.8920}$ &  ${5.3236}$ & $-$ \\
			& & ${(0.5823)}$ & $(0.1056)$ & $(0.6659)$ & & & ${(0.4198)}$ & $(0.9654)$ & $-$ \\
			& \multirow{2}*{10}  & $\mathbf{-1.3659}$ &  ${0.4995}$ & $2.2421$ &
			& \multirow{2}*{24}  & $\mathbf{2.1412}$ &  ${4.5570}$ & $-$ \\
			& & ${(0.4382)}$ & $(0.1748)$ & $(0.4280)$ & & & ${(0.6710)}$ & $(1.3684)$ & $-$ \\
			\hline
			\multirow{8}*{Breast-cancer}
			& \multirow{2}*{1}  & $\mathbf{0.3580}$ &  $0.6907$ & $1.3141$ &
			\multirow{8}*{Parkinsons}
			& \multirow{2}*{2}  & $\mathbf{-0.9465}$ &  $-0.0847$ & $1.0913$ \\
			& & ${(0.0561)}$ & $(0.0394)$ & $(0.0246)$ & & & ${(0.0402)}$ & $(0.0094)$ & $(0.0172)$ \\
			& \multirow{2}*{3}  & $\mathbf{-0.5446}$ &  $0.1743$ & $0.7889$ &
			& \multirow{2}*{7}  & $\mathbf{-5.7700}$ &  ${-2.1513}$ & $0.1867$ \\
			& & ${(0.1887)}$ & $(0.1268)$ & $(0.0626)$ & & & ${(0.1439)}$ & $(0.0189)$ & $(0.0538)$ \\
			& \multirow{2}*{6}  & $\mathbf{-3.2099}$ &  $-1.1397$ & $-0.7526$ &
			& \multirow{2}*{11}  & $\mathbf{-10.0932}$ &  ${-7.8291}$ & $-5.6844$ \\
			& & ${(0.6068)}$ & $(0.2788)$ & $(0.4959)$ & & & ${(0.1492)}$ & $(0.0340)$ & $(0.0906)$ \\
			& \multirow{2}*{8}  & $\mathbf{-6.4362}$ &  $-2.1110$ & $-3.1482$ &
			& \multirow{2}*{15}  & $\mathbf{-16.9316}$ &  ${-16.8767}$ & $-15.6404$ \\
			& & ${(0.8144)}$ & $(0.3906)$ & $(0.6501)$ & & & ${(0.2151)}$ & $(0.1025)$ & $(0.1163)$ \\
			\bottomrule 
		\end{tabular}
	\begin{tablenotes}
		\footnotesize
		\item{*} The best results are marked in \textbf{bold}, and the standard deviation is reported in the parenthesis. 
		{
		The results of MIX on Ionosphere with $d'=10,17,24$ is corrupted due to numerical problems.}
	\end{tablenotes}
\end{table*}

We conduct real data comparisons on real datasets from the UCI repository. We put the detailed description of datasets in Section \ref{DesRealData} of the supplement. 


\textbf{Experimental Settings.} 
In order to evaluate the performance of density estimators on datasets with various dimensions, we apply the following data preprocessing pipeline. Firstly, we remove duplicate observations as well as those with missing values. Then each dimension of the datasets is scaled to $[0, 1]$ and each dataset is reduced to lower dimensions $d'$ through PCA, e.g.~to $10\%$, $30\%$, $50\%$ and $70\%$ of the original dimension $d$, respectively. Finally, in each dataset, we randomly select $70\%$ of the samples for training and the remaining $30\%$ for testing.

The number of iterations $T$ is set to be $100$ and the other two hyper-parameters $s_{\min}$ and $s_{\max} - s_{\min}$ are chosen from $\{ -2 + 0.5k, k = 0, \ldots, 8\}$ and $\{ 0.5 + 0.5k, k = 0, \ldots, 5\}$, respectively, by $3$-fold cross-validation. We repeat this procedure $10$ times to evaluate the standard deviation for \textit{ANLL}. The average \textit{ANLL} on test sets are recorded in Table \ref{tab::Realdata}.

Since real density often resides in a low-dimensional manifold instead of filling the whole high-dimensional space, it is reasonable to study the density estimation problem after dimensionality reduction. Therefore, in data preprocessing, all data sets are reduced to various lower dimensions through PCA. However, we need to take the to-be-reduced dimension as a hyper-parameter, since in general, the dimension of the manifold is unknown.

\textbf{Experimental Results.} In Table \ref{tab::Realdata}, we summarize the comparisons with the state-of-the-art density estimator KDE on six real datasets, which demonstrates the accuracy of our GBHT algorithm. For most of the redacted datasets, GBHT shows its superiority on accuracy, whereas the standard deviation of GBHT is slightly larger than that of KDE due to the randomness of histogram transforms.

\subsection{Gradient Boosted Histogram Transform (GBHT) for Anomaly Detection} \label{sec::anomaly}

To showcase a potential application of GBHT, we propose a density-based method for anomaly detection. Given a density level $\rho$, we regard the sample points with low density estimation $\{ x_i \in D \ | \ f_{\mathrm{D},\lambda}(x_i) \leq \rho \}$ as anomaly points. Based on GBHT density estimation, we are able to present the \textit{Gradient Boosting Histogram Transform} (\textit{GBHT}) for anomaly detection in Algorithm \ref{alg::anomalydetection}.

\begin{algorithm}[!h]
\caption{GBHT for Anomaly Detection}
\label{alg::anomalydetection}
\begin{algorithmic}
		\STATE {\bfseries Input:} Training data $D := \{ x_1, \ldots, x_n \}$; \\
		\quad\quad\quad \ Density threshold parameters $\rho$. \\	
		\STATE Compute GBHT $f_{\mathrm{D},\lambda}$ \eqref{equ::fdlambda}. \\	
		\STATE {\bfseries Output:} Recognize anomalies as 
		$$\{ x_i \in D \ | \ f_{\mathrm{D},\lambda}(x_i) \leq \rho \}.$$
\end{algorithmic}
\end{algorithm}  

We conduct numerical experiments to make a comparison between our GBHT and several popular anomaly detection algorithms such as the forest-based  Isolation Forest (iForest) \cite{liu2008isolation}, the distance-based $k$-Nearest Neighbor ($k$-NN) \cite{ramaswamy2000efficient} and Local Outlier Factor (LOF) \cite{Breunig2000LOF}, and the kernel-based one-class SVM (OCSVM) \cite{scholkopf2001estimating}, on $20$ real-world benchmark outlier detection datasets from the ODDS library. 
{
We perform ranking according to the best AUC performance when parameters go through their parameter grids.}
Detailed experimental settings and comparison results are shown in Section \ref{sec::Aanomaly}.

In the aspect of best performance, our method GBHT wins in $7$ out of $20$ datasets, while the iForest and OCSVM win both $4$ out of $20$ datasets, respectively. Moreover, our GBHT ranks the second on $5$ datasets. Finally, in the aspect of the average performance of benchmark datasets, our method has the lowest rank-sum.
Overall, our experiments on benchmark datasets show that our method has favorable performance among competitive anomaly detection algorithms.

\section{Conclusion}\label{sec::Conclusion}

In this paper, we propose an algorithm called \textit{Gradient Boosting Histogram Transform} (\textit{GBHT}) for density estimation with novel theoretical analysis under the RERM framework. It is well-known that boosting methods are hard to apply in unsupervised learning. Therefore, we turn the density estimation into a supervised learning problem by changing the loss function to \textit{Negative Log Likelihood} loss, which measures the proximity between the estimated density and the true one. In each iteration of boosting methods, histogram transform first randomly stretches, rotates, and translates the feature space for acquiring more information, and then an additional density function is attached to the estimated one with weights, which guarantees that the result is a density function with integral equals to $1$. For theoretical achievements, we prove the convergence properties of our algorithm under mild assumptions. It should be highlighted that we are the first to explain the benefits of the boosting procedure for density estimation algorithms. Last but not least, numerical experiments of both synthetic data and real data are carried out to verify the promising performance of GBHT with applications to anomaly detection. 

\section*{Acknowledgement}
Yisen Wang is supported by the National Natural Science Foundation of China under Grant No. 62006153, CCF-Baidu Open Fund (No. OF2020002), and Project 2020BD006 supported by PKU-Baidu Fund. Zhouchen Lin is supported by the National Natural Science Foundation of China (Grant No.s 61625301 and 61731018), Project 2020BD006 supported by PKU-Baidu Fund, Major Scientific Research Project of Zhejiang Lab (Grant No.s 2019KB0AC01 and 2019KB0AB02), and Beijing Academy of Artificial Intelligence.

%
%

\bibliographystyle{icml2021}
\bibliography{GBHT}

\begin{thebibliography}{40}
\providecommand{\natexlab}[1]{#1}
\providecommand{\url}[1]{\texttt{#1}}
\expandafter\ifx\csname urlstyle\endcsname\relax
  \providecommand{\doi}[1]{doi: #1}\else
  \providecommand{\doi}{doi: \begingroup \urlstyle{rm}\Url}\fi

\bibitem[Amarbayasgalan et~al.(2018)Amarbayasgalan, Jargalsaikhan, and
  Ryu]{amarbayasgalan2018unsupervised}
Amarbayasgalan, T., Jargalsaikhan, B., and Ryu, K.~H.
\newblock Unsupervised novelty detection using deep autoencoders with density
  based clustering.
\newblock \emph{Applied Sciences}, 8\penalty0 (9):\penalty0 1468, 2018.

\bibitem[Biau et~al.(2019)Biau, Cadre, and
  Rouv{\`\i}{\`e}re]{biau2014accelerated}
Biau, G., Cadre, B., and Rouv{\`\i}{\`e}re, L.
\newblock Accelerated gradient boosting.
\newblock \emph{Machine Learning}, 108\penalty0 (6):\penalty0 971--992, 2019.

\bibitem[Blanchard et~al.(2003)Blanchard, Lugosi, and
  Vayatis]{blanchard2003rate}
Blanchard, G., Lugosi, G., and Vayatis, N.
\newblock On the rate of convergence of regularized boosting classifiers.
\newblock \emph{The Journal of Machine Learning Research}, 4\penalty0
  (Oct):\penalty0 861--894, 2003.

\bibitem[Blaser \& Fryzlewicz(2016)Blaser and Fryzlewicz]{blaser2016random}
Blaser, R. and Fryzlewicz, P.
\newblock Random rotation ensembles.
\newblock \emph{The Journal of Machine Learning Research}, 17\penalty0
  (1):\penalty0 126--151, 2016.

\bibitem[Breiman(2000)]{breiman2000some}
Breiman, L.
\newblock Some infinity theory for predictor ensembles.
\newblock Technical report, Technical Report 579, Statistics Dept. UCB, 2000.

\bibitem[Breunig et~al.(2000)Breunig, Kriegel, Ng, and Sander]{Breunig2000LOF}
Breunig, M.~M., Kriegel, H.-P., Ng, R.~T., and Sander, J.
\newblock Lof: identifying density-based local outliers.
\newblock In \emph{ACM Sigmod Record}, volume~29, pp.\  93--104. ACM, 2000.

\bibitem[B{\"u}hlmann \& Yu(2003)B{\"u}hlmann and Yu]{buhlmann2003boosting}
B{\"u}hlmann, P. and Yu, B.
\newblock Boosting with the {L2} loss: regression and classification.
\newblock \emph{Journal of the American Statistical Association}, 98\penalty0
  (462):\penalty0 324--339, 2003.

\bibitem[Cai et~al.(2020)Cai, Hang, Yang, and Lin]{cai2020boosted}
Cai, Y., Hang, H., Yang, H., and Lin, Z.
\newblock Boosted histogram transform for regression.
\newblock In \emph{International Conference on Machine Learning}, pp.\
  1251--1261. PMLR, 2020.

\bibitem[Chen \& Guestrin(2016)Chen and Guestrin]{chen2016xgboost}
Chen, T. and Guestrin, C.
\newblock Xgboost: {A} scalable tree boosting system.
\newblock In \emph{Proceedings of the 22nd ACM SIGKDD International Conference
  on Knowledge Discovery and Data Mining}, pp.\  785--794, 2016.

\bibitem[Chen et~al.(2020)Chen, Hu, Fan, Shen, Zhang, Liu, Du, Li, Chen, and
  Li]{chen2020fast}
Chen, Y., Hu, X., Fan, W., Shen, L., Zhang, Z., Liu, X., Du, J., Li, H., Chen,
  Y., and Li, H.
\newblock Fast density peak clustering for large scale data based on knn.
\newblock \emph{Knowledge-Based Systems}, 187:\penalty0 104824, 2020.

\bibitem[Cortes et~al.(2019)Cortes, Mohri, and
  Storcheus]{cortes2019regularized}
Cortes, C., Mohri, M., and Storcheus, D.
\newblock Regularized gradient boosting.
\newblock \emph{Advances in Neural Information Processing Systems},
  32:\penalty0 5449--5458, 2019.

\bibitem[Criminisi \& Shotton(2013)Criminisi and
  Shotton]{criminisi2013decision}
Criminisi, A. and Shotton, J.
\newblock \emph{Decision Forests for Computer Vision and Medical Image
  Analysis}.
\newblock Springer Science \& Business Media, 2013.

\bibitem[Criminisi et~al.(2011)Criminisi, Shotton, and
  Konukoglu]{criminisi2011decision}
Criminisi, A., Shotton, J., and Konukoglu, E.
\newblock Decision forests for classification, regression, density estimation,
  manifold learning and semi-supervised learning.
\newblock \emph{Microsoft Research Technical Report 2011--114}, 2011.

\bibitem[Duan et~al.(2020)Duan, Anand, Ding, Thai, Basu, Ng, and
  Schuler]{duan2020ngboost}
Duan, T., Anand, A., Ding, D.~Y., Thai, K.~K., Basu, S., Ng, A., and Schuler,
  A.
\newblock Ngboost: Natural gradient boosting for probabilistic prediction.
\newblock In \emph{International Conference on Machine Learning}, pp.\
  2690--2700. PMLR, 2020.

\bibitem[Freund \& Schapire(1997)Freund and Schapire]{freund1997decision}
Freund, Y. and Schapire, R.~E.
\newblock A decision-theoretic generalization of on-line learning and an
  application to boosting.
\newblock \emph{Journal of Computer and System Sciences}, 55\penalty0
  (1):\penalty0 119--139, 1997.

\bibitem[Friedman(2001)]{friedman2001greedy}
Friedman, J.~H.
\newblock Greedy function approximation: a gradient boosting machine.
\newblock \emph{The Annals of Statistics}, pp.\  1189--1232, 2001.

\bibitem[Ghaffari et~al.(2019)Ghaffari, Lattanzi, and
  Mitrovi{\'c}]{ghaffari2019improved}
Ghaffari, M., Lattanzi, S., and Mitrovi{\'c}, S.
\newblock Improved parallel algorithms for density-based network clustering.
\newblock In \emph{International Conference on Machine Learning}, pp.\
  2201--2210. PMLR, 2019.

\bibitem[Jang \& Jiang(2019)Jang and Jiang]{jang2019dbscanpp}
Jang, J. and Jiang, H.
\newblock {DBSCAN}++: Towards fast and scalable density clustering.
\newblock In \emph{International Conference on Machine Learning}, pp.\
  3019--3029. PMLR, 2019.

\bibitem[Jeffreys(1946)]{jeffreys1946invariant}
Jeffreys, H.
\newblock An invariant form for the prior probability in estimation problems.
\newblock \emph{Proceedings of the Royal Society of London. Series A.
  Mathematical and Physical Sciences}, 186\penalty0 (1007):\penalty0 453--461,
  1946.

\bibitem[Klemel\"{a}(2009)]{klemela2009multivariate}
Klemel\"{a}, J.
\newblock Multivariate histograms with data-dependent partitions.
\newblock \emph{Statistica Sinica}, 19\penalty0 (1):\penalty0 159--176, 2009.

\bibitem[Liu et~al.(2008)Liu, Ting, and Zhou]{liu2008isolation}
Liu, F.~T., Ting, K.~M., and Zhou, Z.-H.
\newblock Isolation forest.
\newblock In \emph{Proceedings of the IEEE International Conference on Data
  Mining}, pp.\  413--422, 2008.

\bibitem[Liu \& Wong(2014)Liu and Wong]{liu2014multivariate}
Liu, L. and Wong, W.~H.
\newblock Multivariate density estimation via adaptive partitioning {(I)}:
  sieve {MLE}.
\newblock \emph{arXiv preprint arXiv:1401.2597}, 2014.

\bibitem[L{\'o}pez-Rubio(2013)]{lopez2013histogram}
L{\'o}pez-Rubio, E.
\newblock A histogram transform for probability density function estimation.
\newblock \emph{IEEE Transactions on Pattern Analysis and Machine
  Intelligence}, 36\penalty0 (4):\penalty0 644--656, 2013.

\bibitem[Mathiasen et~al.(2019)Mathiasen, Larsen, and
  Gr{\o}nlund]{mathiasen2019optimal}
Mathiasen, A., Larsen, K.~G., and Gr{\o}nlund, A.
\newblock Optimal minimal margin maximization with boosting.
\newblock In \emph{International Conference on Machine Learning}, pp.\
  4392--4401. PMLR, 2019.

\bibitem[Nachman \& Shih(2020)Nachman and Shih]{nachman2020anomaly}
Nachman, B. and Shih, D.
\newblock Anomaly detection with density estimation.
\newblock \emph{Physical Review D}, 101\penalty0 (7):\penalty0 075042, 2020.

\bibitem[Parmar et~al.(2019)Parmar, Wang, Zhang, Tan, Miao, Jiang, and
  Zhou]{parmar2019redpc}
Parmar, M., Wang, D., Zhang, X., Tan, A.-H., Miao, C., Jiang, J., and Zhou, Y.
\newblock Redpc: A residual error-based density peak clustering algorithm.
\newblock \emph{Neurocomputing}, 348:\penalty0 82--96, 2019.

\bibitem[Parnell et~al.(2020)Parnell, Anghel, {\L}azuka, Ioannou, Kurella,
  Agarwal, Papandreou, and Pozidis]{parnell2020snapboost}
Parnell, T., Anghel, A., {\L}azuka, M., Ioannou, N., Kurella, S., Agarwal, P.,
  Papandreou, N., and Pozidis, H.
\newblock Snapboost: A heterogeneous boosting machine.
\newblock \emph{Advances in Neural Information Processing Systems}, 33, 2020.

\bibitem[Ram \& Gray(2011)Ram and Gray]{ram2011density}
Ram, P. and Gray, A.~G.
\newblock Density estimation trees.
\newblock In \emph{Proceedings of the 17th ACM SIGKDD International Conference
  on Knowledge Discovery and Data Mining}, pp.\  627--635. ACM, 2011.

\bibitem[Ramaswamy et~al.(2000)Ramaswamy, Rastogi, and
  Shim]{ramaswamy2000efficient}
Ramaswamy, S., Rastogi, R., and Shim, K.
\newblock Efficient algorithms for mining outliers from large data sets.
\newblock In \emph{Proceedings of the ACM SIGMOD International Conference on
  Management of Data}, pp.\  427--438, 2000.

\bibitem[Ridgeway(2002)]{ridgeway2002looking}
Ridgeway, G.
\newblock Looking for lumps: Boosting and bagging for density estimation.
\newblock \emph{Computational Statistics \& Data Analysis}, 38\penalty0
  (4):\penalty0 379--392, 2002.

\bibitem[Rosset \& Segal(2003)Rosset and Segal]{rosset2003boosting}
Rosset, S. and Segal, E.
\newblock Boosting density estimation.
\newblock In \emph{Advances in Neural Information Processing Systems}, pp.\
  657--664, 2003.

\bibitem[Schapire \& Freund(1995)Schapire and Freund]{schapire1995decision}
Schapire, R. and Freund, Y.
\newblock A decision-theoretic generalization of on-line learning and an
  application to boosting.
\newblock In \emph{Second European Conference on Computational Learning
  Theory}, pp.\  23--37, 1995.

\bibitem[Sch{\"o}lkopf et~al.(2001)Sch{\"o}lkopf, Platt, Shawe-Taylor, Smola,
  and Williamson]{scholkopf2001estimating}
Sch{\"o}lkopf, B., Platt, J.~C., Shawe-Taylor, J., Smola, A.~J., and
  Williamson, R.~C.
\newblock Estimating the support of a high-dimensional distribution.
\newblock \emph{Neural Computation}, 13\penalty0 (7):\penalty0 1443--1471,
  2001.

\bibitem[Scott(2015)]{scott2015multivariate}
Scott, D.~W.
\newblock \emph{Multivariate Density Estimation}.
\newblock John Wiley \& Sons, Inc., Hoboken, NJ, second edition, 2015.

\bibitem[Steinwart \& Christmann(2008)Steinwart and Christmann]{StCh08}
Steinwart, I. and Christmann, A.
\newblock \emph{Support Vector Machines}.
\newblock Information Science and Statistics. Springer, New York, 2008.

\bibitem[Sturges(1926)]{sturges1926choice}
Sturges, H.~A.
\newblock The choice of a class interval.
\newblock \emph{Journal of the American Statistical Association}, 21\penalty0
  (153):\penalty0 65--66, 1926.

\bibitem[Suggala et~al.(2020)Suggala, Liu, and
  Ravikumar]{suggala2020generalized}
Suggala, A., Liu, B., and Ravikumar, P.
\newblock Generalized boosting.
\newblock \emph{Advances in Neural Information Processing Systems}, 33, 2020.

\bibitem[Van~der Vaart \& Wellner(1996)Van~der Vaart and Wellner]{van1996weak}
Van~der Vaart, A.~W. and Wellner, J.~A.
\newblock \emph{Weak Convergence and Empirical Processes}.
\newblock Springer Series in Statistics. Springer-Verlag, New York, 1996.

\bibitem[Vapnik \& Chervonenkis(2015)Vapnik and
  Chervonenkis]{vapnik2015uniform}
Vapnik, V.~N. and Chervonenkis, A.~Y.
\newblock On the uniform convergence of relative frequencies of events to their
  probabilities.
\newblock In \emph{Measures of Complexity}, pp.\  11--30. Springer, 2015.

\bibitem[Zhang et~al.(2018)Zhang, Lin, and Karim]{zhang2018adaptive}
Zhang, L., Lin, J., and Karim, R.
\newblock Adaptive kernel density-based anomaly detection for nonlinear
  systems.
\newblock \emph{Knowledge-Based Systems}, 139:\penalty0 50--63, 2018.

\end{thebibliography}

\clearpage
\appendix

\section*{Appendix}

This file consists of supplementaries for both theoretical analysis and experiments. 
In Section \ref{sec::ErrorAnalysis}, we divide the general risk into approximation error and estimation error term for the underlying density function residing in space $C^{0,\alpha}$ and $C^{1,\alpha}$, respectively. 
The corresponding proofs of Section \ref{sec::ErrorAnalysis} and Section \ref{sec::mainresults} are shown in Section \ref{sec::proofs}.
In Section \ref{sec::supple_exp} we show the supplementaries for numerical experiments.

\section{Error Analysis}\label{sec::ErrorAnalysis}
This section provides a more comprehensive error analysis for the theoretical results in Section \ref{sec::mainresults}. To be specific, we conduct approximation error analysis for the boosted density estimators $f_{\mathrm{D},\lambda}$ under the assumption that the density function $f^*_{L,\mathrm{P}}$ lying in the H\"{o}lder spaces $C^{0,\alpha}$ and $C^{1,\alpha}$.

To conduct the theoretical analysis, we also need the infinite sample version of Definition \ref{def::RBHT}. To this end, we fix a distribution $\mathrm{P}$ on $\mathcal{X} \times \mathcal{Y}$ and let the function space $E$ be as in \eqref{equ::En}. Then every $f_{\mathrm{P},\lambda} \in E$ satisfying
\begin{align*}
	\Omega(h) + \mathcal{R}_{L,\mathrm{P}}(f_{\mathrm{P},\lambda}) 
	= \inf_{f \in E} \Omega(h) + \mathcal{R}_{L,\mathrm{P}}(f)
\end{align*}
is called an infinite sample version of GBHT with respect to $E$ and $L$. Moreover, the approximation error function $A(\lambda)$ is defined by
\begin{align}\label{equ::approximationerror}
	A(\lambda) 
	= \inf_{f \in E} \Omega(h) + \mathcal{R}_{L,\mathrm{P}}(f) - \mathcal{R}^*_{L,\mathrm{P}}.
\end{align}

\subsection{Error Analysis for $f \in C^{0,\alpha}$} \label{subsec::analysisc0alpha}

First of all, we introduce some definitions and notations which will be used in the supplementary material. 
Recall that the $L_p$-distance between $g_1, g_2 \in L_p(\mu)$, $p \in [1, \infty)$, is defined by
\begin{align*}
	\|g_1 - g_2\|_{L_p(\mu)} 
	:= \biggl( \int_{\mathcal{X}} (g_1(x) - g_2(x))^p \, d\mu(x) \biggr)^{1/p}.
\end{align*}

For a given histogram transform $H$, let the function set $\mathcal{F}_H$ be defined by \eqref{equ::functionFn}.
We write
\begin{align} \label{def::fPH}
	f_{\mathrm{P},H} := \argmin_{\hat{f} \in \mathcal{F}_H} \|\hat{f} - f\|_{L_2(\mu)}^2.
\end{align}
In other words, $f_{\mathrm{P},H}$ is the function that minimizes the $L_2$-distance over the function set $\mathcal{F}_H$ with the bin width $h \in [\underline{h}_0,\overline{h}_0]$. Then, elementary calculation yields
\begin{align}\label{eq::fPH}
	f_{\mathrm{P},H}(x)
	& = \mathbb{E}_{\mu}(f(X) | A_H(x))
	\nonumber\\
	& = \sum_{j \in \mathcal{I}_H} \frac{\int_{A_j} f(x) \, d\mu(z)}{\mu(A_j)} \cdot \eins_{A_j}(x)
	\nonumber\\
	&= \sum_{j \in \mathcal{I}_H} \frac{\mathrm{P}(A_j)}{\mu(A_j)} \cdot \eins_{A_j}(x)
\end{align}
Moreover, we write
\begin{align} \label{def::fDH}
	f_{\mathrm{D},H} 
	= \sum_{j \in \mathcal{I}_H} \frac{\sum_{i=1}^n \eins_{A_j}(x)}{n\mu(A_j)} \cdot \eins_{A_j}(x)
\end{align}
for the empirical version, which can be further presented as
\begin{align*}
	f_{\mathrm{D},H} = \sum_{j \in \mathcal{I}_H}  \frac{\mathrm{D}(A_j)}{\mu(A_j)}\cdot \eins_{A_j}.
\end{align*}

\begin{lemma}\label{lem::relationlogL2}
	Let $f$ be the underlying probability density function and $\mathrm{P}$ is the corresponding distribution of $f$. 
	Moreover, let $L : \mathcal{X} \times [0, \infty) \to \mathbb{R}$ be
	the \textit{Negative Log Likelihood} loss defined by \eqref{eq::L}.
	Then $f$ is exactly the minimizer of $\mathcal{R}_{L,\mathrm{P}}(\cdot)$ among all density functions. 
	For fixed constants $\underline{c}_f, \overline{c}_f \in (0, \infty)$, let $\mathcal{A}_f^0$ denote the set 
	\begin{align} \label{DegenerateSetF0}
		\mathcal{A}_f^0 := 
		\bigl\{ x \in \mathbb{R}^d : f(x) \in [\underline{c}_f, \overline{c}_f] \bigr\}.
	\end{align}
	Then for any $x \in \mathcal{A}_f^0$, there holds
	\begin{align*}
		\frac{\|g - f\|_{L_2(\mu)}^2}{2 \underline{c}_f}  &- \frac{\|g - f\|^3_{L_3(\mu)}}{3 \overline{c}_f^2} \leq
		\\
		& \mathcal{R}_{L,\mathrm{P}}(g) -  \mathcal{R}_{L,\mathrm{P}}(f) 
		\leq \frac{\|g - f\|_{L_2(\mu)}^2}{2 \underline{c}_f}.
	\end{align*}
\end{lemma}

\subsubsection{Bounding the Approximation Error Term} \label{sec::boundapproerror}

The following proposition shows that the $L_2$ distance between $f_{\mathrm{P},H}$ and $f$ behaves polynomial in the regularization parameter $\lambda$ if we choose the bin width $\underline{h}_0$  appropriately.

\begin{proposition}\label{pro::c0alphasingle}
	Let the histogram transform $H$ be defined as in \eqref{equ::HT} with bin width $h$ satisfies Assumption \ref{assumption::h}. Furthermore, suppose that the density function $f\in C^{0,\alpha}$. Then, for any fixed $\lambda > 0$, there holds
	\begin{align*}
		\lambda h^{-2d} + \mathcal{R}_{L,\mathrm{P}}(f_{\mathrm{P},H}) - \mathcal{R}^*_{L,\mathrm{P}}
		\leq c \cdot \lambda^{\frac{\alpha}{\alpha+d}},
	\end{align*}
	where $c$ is some constant depending on $\alpha$, $d$, and $c_0$ as in Assumption \ref{assumption::h}.
\end{proposition}

\subsubsection{Bounding the Sample Error Term}
\label{sec::boundsamplerror}

To derive bounds on the sample error of regularized empirical risk minimizers, let us briefly recall the definition of VC dimension measuring the complexity of the underlying function class.

\begin{definition}[VC dimension] \label{def::VC dimension}
	Let $\mathcal{B}$ be a class of subsets of $\mathcal{X}$ and $A \subset \mathcal{X}$ be a finite set. The trace of $\mathcal{B}$ on $A$ is defined by $\{ B \cap A : B \subset \mathcal{B}\}$. Its cardinality is denoted by $\Delta^{\mathcal{B}}(A)$. We say that $\mathcal{B}$ shatters $A$ if $\Delta^{\mathcal{B}}(A) = 2^{\#(A)}$, that is, if for every $\tilde{A} \subset A$, there exists a $B \subset \mathcal{B}$ such that $\tilde{A} = B \cap A$. For $k \in \mathrm{N}$, let
	\begin{align}\label{equ::VC dimension}
		m^{\mathcal{B}}(k) := \sup_{A \subset \mathcal{X}, \, \#(A) = k} \Delta^{\mathcal{B}}(A).
	\end{align}
	Then, the set $\mathcal{B}$ is a Vapnik-Chervonenkis class if there exists $k<\infty$ such that $m^{\mathcal{B}}(k) < 2^k$ and the minimal of such $k$ is called the VC dimension of $\mathcal{B}$, and abbreviate as $\mathrm{VC}(\mathcal{B})$.
\end{definition}

To prove Lemma \ref{VCIndex}, we need the following fundamental lemma concerning with the VC dimension of purely random partitions, which follows the idea put forward by \cite{breiman2000some} of the construction of purely random forest. To this end, let $p \in \mathbb{N}$ be fixed and $\pi_p$ be a partition of $\mathcal{X}$ with number of splits $p$ and $\pi_{(p)}$ denote the collection of all partitions $\pi_p$.

\begin{lemma}\label{VCIndex}
	Let $\mathcal{B}_p$ be defined by
	\begin{align} \label{Bp}
		\mathcal{B}_p := \biggl\{ B : B = \bigcup_{j \in J} A_j, J \subset \{ 0, 1, \ldots, p \}, A_j \in \pi_p \biggr\}.
	\end{align}
	Then the VC dimension of $\mathcal{B}_p$ can be upper bounded by $d p + 2$. 
\end{lemma}

To investigate the capacity property of continuous-valued functions, we need to introduce the concept 
\textit{VC-subgraph class}. To this end, the \emph{subgraph} of a function $f : \mathcal{X} \to \mathbb{R}$ is defined by 
\begin{align*}
	\textit{sg}(f) := \{ (x, t) : t < f(x) \}.
\end{align*}
A class $\mathcal{F}$ of functions on $\mathcal{X}$ is said to be a VC-subgraph class, if the collection of all subgraphs of functions in $\mathcal{F}$, which is denoted by $\textit{sg}(\mathcal{F}) := \{ \textit{sg}(f) : f \in \mathcal{F} \}$ is a VC class of sets in $\mathcal{X} \times \mathbb{R}$. Then the VC dimension of $\mathcal{F}$ is defined by the VC dimension of the collection of the subgraphs, that is, 
$\mathrm{VC}(\mathcal{F}) = \mathrm{VC}(\textit{sg}(\mathcal{F}))$.

Before we proceed, we also need to recall the definitions of the convex hull and VC-hull class. 
The symmetric \textit{convex hull} $\mathrm{Co}(\mathcal{F})$ of a class of functions $\mathcal{F}$ is defined as the set of functions $\sum_{i=1}^m \alpha_i f_i$ with $\sum_{i=1}^m |\alpha_i| \leq 1$ and each $f_i$ contained in $\mathcal{F}$. 
A set of measurable functions is called a \textit{VC-hull class}, if it is in the pointwise sequential closure of the symmetric convex hull of a VC-class of functions.

We denote the function set $\mathcal{F}$ as
\begin{align}\label{equ::functionFH}
	\mathcal{F} := \bigcup_{H \sim \mathrm{P}_H} \mathcal{F}_H,
\end{align}
which contains all the functions of $\mathcal{F}_H$ induced by histogram transforms $H$ with bin width $\underline{h}_0$.

The following lemma presents the upper bound for the VC dimension of the function set $\mathcal{F}$.

\begin{lemma}\label{lem::VCFn}
	Let $\mathcal{F}$ be the function set defined as in \eqref{equ::functionFH}. Then $\mathcal{F}$ is a $\mathrm{VC}$-subgraph class with 
	\begin{align*}
		\mathrm{VC}(\mathcal{F}) 
		\leq (d+1) 2^{d+1} \bigl( \lfloor 2 R \sqrt{d} / \underline{h}_0 \rfloor + 1 \bigr)^d.
	\end{align*}
\end{lemma}

To further bound the capacity of the function sets, we need to introduce the following fundamental descriptions which enables an approximation of an infinite set by finite subsets.

\begin{definition}[Covering Numbers]\label{def::Covering Numbers}
	Let $(\mathcal{X}, d)$ be a metric space, $A \subset \mathcal{X}$ and $\varepsilon > 0$. We call $A' \subset A$ an $\varepsilon$-net of $A$ if for all $x \in A$ there exists an $x' \in A'$ such that $d(x, x') \leq \varepsilon$. Moreover, the $\varepsilon$-covering number of $A$ is defined as
	\begin{align*}
		\mathcal{N}(A, d, \varepsilon)
		= \inf \biggl\{ n \geq 1 &: \exists x_1, \ldots, x_n \in \mathcal{X}, \\
		&\text{ such that } A \subset \bigcup_{i=1}^n B_d(x_i, \varepsilon) \biggr\},
	\end{align*}
	where $B_d(x, \varepsilon)$ denotes the closed ball in $\mathcal{X}$ centered at $x$ with radius $\varepsilon$.
\end{definition}

The following lemma follows directly from Theorem 2.6.9 in \cite{van1996weak}. For the sake of completeness, we present the proof in Section \ref{sec::proofrelatsample}.

\begin{lemma}\label{thm::vart}
	Let $\mathrm{Q}$ be a probability measure on $\mathcal{X}$ and 
	\begin{align*}
		\mathcal{F} := \bigl\{ f : \mathcal{X} \to \mathbb{R} : f \in [-M, M] \bigr\}.
	\end{align*}
	Assume that for some fixed $\varepsilon > 0$ and $v > 0$, the covering number of $\mathcal{F}$ satisfies
	\begin{align} \label{CoverAssumption}
		\mathcal{N}(\mathcal{F}, L_2(\mathrm{Q}), M \varepsilon)
		\leq c \, (1/\varepsilon)^v.
	\end{align}
	Then there exists a universal constant $c'$ such that
	\begin{align*}
		\log \mathcal{N}(\mathrm{Co}(\mathcal{F}), L_2(\mathrm{Q}), M \varepsilon)
		\leq c' c^{2/(v+2)} \varepsilon^{-2v/(v+2)}.
	\end{align*}
\end{lemma}

The next theorem shows that covering numbers of $\mathcal{F}$ grow at a polynomial rate.

\begin{theorem}\label{the::Fncovering}
	Let $\mathcal{F}$ be a function set defined as in \eqref{equ::functionFH}. Then there exists a universal constant $c < \infty$ such that for any $\varepsilon \in (0, 1)$ and any probability measure $\mathrm{Q}$,
	we have
	\small
	\begin{align*}
		\mathcal{N}(\mathcal{F},L_2(\mathrm{Q}),M\varepsilon)
		\leq c_0 (c_d/\underline{h}_0)^d \cdot (16e)^{(c_d/\underline{h}_0)^d}        
		\varepsilon^{2(\underline{h}_0/c_d)^d-2},
	\end{align*}
	\normalsize
	where the constant $c_d:=2^{1+4/d}\cdot d^{1/2+1/d}$.
\end{theorem}

The following theorem gives an upper bound on the covering number of the $\mathrm{VC}$-hull class  $\mathrm{Co}(\mathcal{F})$.

\begin{theorem}\label{the::convexFn}
	Let $\mathcal{F}$ be the function set defined as in \eqref{equ::functionFH}. Then there exists a constant $c_1$ such that for any $\varepsilon \in (0, 1)$ and any probability measure $\mathrm{Q}$, there holds
	\begin{align}\label{equ::convexFn}
		\log \mathcal{N}(\mathrm{Co}(\mathcal{F}),L_2(\mathrm{Q}),M\varepsilon)
		\leq c_1 \varepsilon^{2(\underline{h}_0/c_d)^d-2}.
	\end{align}
\end{theorem}

Next, let us recall the definition of entropy numbers.

\begin{definition}[Entropy Numbers] \label{def::entropy numbers}
	Let $(\mathcal{X}, d)$ be a metric space, $A \subset \mathcal{X}$ and $m \geq 1$ be an integer. The $m$-th entropy number of $(A, d)$ is defined as
	\begin{align*}
		e_m(A, d) 
		= \inf \biggl\{ \varepsilon > 0 &: \exists x_1, \ldots, x_{2^{m-1}} \in \mathcal{X} \\
		&\text{ such that } A \subset \bigcup_{i=1}^{2^{m-1}} B_d(x_i, \varepsilon) \biggr\}.
	\end{align*}
	Moreover, if $(A, d)$ is a subspace of a normed space $(E, \|\cdot\|)$ and the metric $d$ is given by $d(x, x') = \|x - x'\|$, $x, x' \in A$, we write $e_m(A, \|\cdot\|) := e_m(A, E) := e_m(A, d)$.
	Finally, if $S : E \to F$ is a bounded, linear operator between the normed space $E$ and $F$, we denote
	$e_m(S) := e_m(S B_E, \|\cdot\|_F)$.
\end{definition}

For a finite set $D \in \mathcal{X}^n$, we define the norm of an empirical $L_2$-space by
\begin{align*}
	\|f\|^2_{L_2(\mathrm{D})}
	= \mathbb{E}_{\mathrm{D}} |f|^2
	:= \frac{1}{n} \sum_{i=1}^n |f(x_i)|^2.
\end{align*}
If $E$ is the function space \eqref{equ::En} and $\mathrm{D}_X \in \mathcal{X}^n$, then the entropy number $e_m(\mathrm{id} : E \to L_2(\mathrm{D}_X))$ equals the $m$-th entropy number of the symmetric convex hull of the family $\{ (f_i), f_i \in \mathcal{F}_i \}$, where $\mathrm{id} : E \to L_2(\mathrm{D}_X)$ denotes the identity map that assigns to every $f \in E$ the corresponding equivalence class in $L_2(\mathrm{D}_X)$.

Now, we are able to present an oracle inequality for GBHT, which gives an upper bound for the sample error term.

\begin{theorem}\label{thm::OracalBoost}
	Let the histogram transform $H_n$ be defined as in \eqref{equ::HT} with bin width $h_n$ satisfying Assumption \ref{assumption::h}. Furthermore, let $f_{\mathrm{D},\lambda}$ be the GBHT defined by \eqref{equ::fdlambda} and $A(\lambda)$ be the corresponding approximation error defined by \eqref{equ::approximationerror}. Then for all $\tau>0$, with probability $\mathrm{P}^n \otimes \mathrm{P}_H$ not less than $1-3e^{-\tau}$, we have
	\begin{align*}
		\Omega(h) &+ \mathcal{R}_{L,\mathrm{D}}(f_{\mathrm{D},\lambda}) - \mathcal{R}_{L,\mathrm{P}}^*
		\leq \\
		&12 A(\lambda) + 3456 M^2 \tau / n+ 3 c_0' \lambda^{-\frac{1}{1+2\delta'}}n^{-\frac{2}{1+2\delta'}},
	\end{align*}
	where $c_0'$ is a constant.
\end{theorem}

\subsection{Error Analysis for $f \in  C^{1,\alpha}$} \label{subsec::analysisc1alpha}

A drawback to the analysis in $C^{0,\alpha}$ is that the usual Taylor expansion involved techniques for error estimation may not apply directly. As a result, we fail to prove the exact benefits of the boosting procedure. Therefore, in this subsection, we turn to the function space $C^{1,\alpha}$ consisting of smoother functions. To be specific, we study the convergence rates of $f_{\mathrm{D},\lambda}$ to the density function $f \in C^{1,\alpha}$. To this end, there is a point in introducing some notations.

For fixed $\underline{h}_0, \overline{h}_0 > 0$, let $\{ H_t \}_{t=1}^T$ be histogram transforms with bin width $h_t \in [\underline{h_0},\overline{h}_0]$, $t = 1, \ldots, T$. 
Moreover, let $\{ f_{\mathrm{P},H_t} \}_{t=1}^T$  and $\{ f_{\mathrm{D},H_t} \}_{t=1}^T$ be defined as in \eqref{def::fPH} and \eqref{def::fDH}, respectively. 
For $x \in \mathcal{X}$, we define
\begin{align}\label{fpeensemble}
	f_{\mathrm{P},\mathrm{E}}(x)
	:= \frac{1}{T} \sum_{t=1}^T f_{\mathrm{P},H_t}(x)
\end{align}
and 
\begin{align}\label{fempensemble}
	f_{\mathrm{D},\mathrm{E}}(x)
	:= \frac{1}{T} \sum_{t=1}^T f_{\mathrm{D},H_t}(x).
\end{align}
Then we make the error decomposition
\begin{align} \label{equ::L2de}
	\mathbb{E}_{\nu_n} \|f_{\mathrm{D},\mathrm{E}} &- f \|_{L_2(\mu)}^2 = 
	\nonumber \\
	&\mathbb{E}_{\nu_n} \|f_{\mathrm{D},\mathrm{E}} - f_{\mathrm{P},\mathrm{E}} \|_{L_2(\mu)}^2
	+ \mathbb{E}_{\nu_n} \|f_{\mathrm{P},\mathrm{E}} - f \|_{L_2(\mu)}^2,
\end{align}
where $\nu_n := \mathrm{P}^n \otimes \mathrm{P}_H$. In particular, in the case that $T = 1$, i.e., for the base histogram transform density estimator, we are concerned with the lower bound for $f_{\mathrm{D}, H}$. 
We make the error decomposition 
\small
\begin{align} \label{equ::L2Decomposition}
	\mathbb{E}_{\nu_n} \|f_{\mathrm{D},H} &- f \|_{L_2(\mu)}^2 =
	\nonumber \\
	&\mathbb{E}_{\nu_n} \|f_{\mathrm{D},H} - f_{\mathrm{P},\mathrm{H}} \|_{L_2(\mu)}^2
	+ \mathbb{E}_{\nu_n} \|f_{\mathrm{P},H} - f \|_{L_2(\mu)}^2
\end{align}
\normalsize
and
\small
\begin{align}     
	\mathbb{E}_{\nu_n} &\|f_{\mathrm{D},H} - f \|_{L_3(\mu)}^3
	\nonumber\\
	&= \mathbb{E}_{\nu_n} \|f_{\mathrm{D},H} - f_{\mathrm{P},H} + f_{\mathrm{P},H} - f \|_{L_3(\mu)}^3
	\nonumber\\
	&= \mathbb{E}_{\nu_n} \|f_{\mathrm{D},H} - f_{\mathrm{P},H} \|_{L_3(\mu)}^3
	+ \mathbb{E}_{\nu_n} \|f_{\mathrm{P},H} - f \|_{L_3(\mu)}^3 
	\nonumber\\
	&\quad + 3\mathbb{E}_{\nu_n}\int_{\mathcal{X}}(f_{\mathrm{D},H}(x) - f_{\mathrm{P},H}(x))^2 (f_{\mathrm{P},H}(x) - f(x)) \,dx .
	\label{equ::L3Decomposition}
\end{align}
\normalsize
It is important to note that both of the two terms on the right-hand side of \eqref{equ::L2de} and \eqref{equ::L2Decomposition} are data- and partition-independent due to the expectation with respect to $\mathrm{D}$ and $H$. Loosely speaking, the first error term corresponds to the expected estimation error of the estimators $f_{\mathrm{D},\mathrm{E}}$ or $f_{\mathrm{D},H}$, while the second one demonstrates the expected approximation error.

\subsubsection{Upper Bound for Convergence Rate of GBHT} \label{sec::upperboundconve}

The following Lemma presents the explicit representation of $A_H(x)$ which will be used later in the proofs of Proposition \ref{prop::biasterm}.

\begin{lemma}\label{binset}
	Let the histogram transform $H$ be defined as in \eqref{equ::HT}
	and $A'_H$, $A_H$ be as in Section \ref{sub::histogram}. 
	Then for any $x \in \mathbb{R}^d$, the set $A_H(x)$ can be represented as 
	\begin{align*}
		A_H(x) = \bigl\{ x + (R \cdot S)^{-1} z :  z \in [-b', 1 - b'] \bigr\},
	\end{align*}
	where $b' \sim \mathrm{Unif}(0, 1)^d$.
\end{lemma}

The next proposition presents the upper bound of the $L_2$ distance between GBHT $f_{\mathrm{P},\mathrm{E}}$ \eqref{fpeensemble} and the density function $f$ in the H\"{o}lder space $C^{1,\alpha}$. 

\begin{proposition}\label{prop::biasterm}
	Let the histogram transform $H$ be defined as in \eqref{equ::HT} with bin width $h$ satisfying Assumption \ref{assumption::h} and $T$ be the number of iterations.
	Furthermore, let $\mathrm{P}_X$ be the uniform distribution and $L_{\overline{h}_0}(x,y,t)$ be the restricted negative log-likelihood loss defined as in \eqref{equ::resloss}.
	Moreover, let the density function satisfy $f \in C^{1,\alpha}$. 
	For fixed constants $\underline{c}_f, \overline{c}_f \in (0, \infty)$, let $\mathcal{A}_f^0$ be as in \eqref{DegenerateSetF0}.
	Then for any $x \in \mathcal{A}_f^0$, there holds
	\small
	\begin{align}\label{eq::PE}
		\mathcal{R}_{L_{\overline{h}_0},\mathrm{P}}(f_{\mathrm{P},\mathrm{E}})-\mathcal{R}_{L_{\overline{h}_0},\mathrm{P}}^* 
		\leq \frac{c_L^2 \mu(B_R)}{2\underline{c}_f} \cdot \biggl( \overline{h}_0^{2(1+\alpha)} 
		+ \frac{d}{T} \cdot \overline{h}_0^2 \biggr)
	\end{align}
	\normalsize
	in expectation with respect to $\mathrm{P}_H$.
\end{proposition}

\subsubsection{Lower Bound of $L_2$-Convergence Rate of HT} \label{sec::lowerboundconve}

\begin{theorem} \label{thm::LowerBoundSingles}
	Let the histogram transform $H_n$ be defined as in \eqref{equ::HT} with bandwidth $h_n$ satisfying Assumption \ref{assumption::h}. Furthermore, let the density function $f \in C^{1,\alpha}$. For fixed constants $\underline{c}'_f, \underline{c}_f, \overline{c}_f \in (0, \infty)$, let $\mathcal{A}_f^1$ denote the set 
	\begin{align} \label{DegenerateSetF}
		\mathcal{A}_f^1 := 
		\biggl\{ x \in \mathbb{R}^d : 
		\|\nabla f\|_{\infty} \geq \underline{c}'_f
		\text{ and }
		f(x) \in [\underline{c}_f, \overline{c}_f]\biggr\}.
	\end{align}
	If $\mu(B_{r, \sqrt{d} \cdot \overline{h}_0}^+ \cap \mathcal{A}_f^1)>0$, then for all $n > N_0$ with
	\begin{align} \label{MinimalNumber}
		N_0 := \min \biggl\{
		n \in \mathbb{N} :
		&\overline{h}_{0,n} \leq  
		\min \biggl\{ 
		\biggl( \frac{\sqrt{d} \underline{c}'_f c_{0,n}}{4 \sqrt{3} c_L} \biggr)^{\frac{1}{\alpha}},
		\nonumber \\
		&\biggl( \frac{d \sqrt{d}}{2} \biggr)^{\frac{1}{\alpha}},
		\frac{\underline{c}_f}{2 d \sqrt{d} c_L},
		\biggl( \frac{1}{4 \overline{c}_f} \biggr)^{\frac{1}{d}}
		\biggr\} \biggr\},
	\end{align}
	by choosing 
	\begin{align*}
		\overline{h}_{0,n} := n^{-\frac{1}{2+d}}, 
	\end{align*}
	there holds
	\begin{align} \label{RatesLInftyConeSingle}
		\|f_{\mathrm{D},H_n} - f\|_{L_{2}(\mu)}^2
		\gtrsim n^{-\frac{2}{2+d}}
	\end{align}
	in the sense of $L_2(\nu_n)$-norm.
\end{theorem}

In order to prove Theorem \ref{thm::LowerBoundSingles}, we prove the following two propositions presenting the lower bound of approximation error and sample error of HT respectively.

\begin{proposition}\label{counterapprox}
	Let the histogram transform $H$ be defined as in \eqref{equ::HT} with bin width $h$ satisfying Assumption \ref{assumption::h} and $\overline{h}_0 \leq 1$.
	Moreover, let the density function $f\in C^{1,\alpha}(B_R)$. For a fixed constant $\underline{c}_f \in (0, \infty)$, let $\mathcal{A}_f^1$ be the set \eqref{DegenerateSetF}.
	Let $N_1$ be defined as 
	\begin{align} \label{NOne}
		N_1 := \min \biggl\{ n \in \mathrm{N} : \overline{h}_{0,n} \leq \biggl( \frac{\sqrt{d} \underline{c}'_f c_0}{4 \sqrt{3} c_L} \biggr)^{\frac{1}{\alpha}} \biggr\}.
	\end{align}
	Then for all $n > N_1$,  there holds
	\begin{align*}
		\bigl\| f_{\mathrm{P},H} - f \bigr\|^2_2 \geq \frac{d}{16} \mu(\mathcal{A}_f^1\cap B_{R,\sqrt{d}\overline{h}_{0}}^+) c_0^2\underline{c}_f'^{2} \cdot \overline{h}_0^2.
	\end{align*}
	in expectation with respect to $\mathrm{P}_H$.
\end{proposition}

\begin{proposition} \label{OracleInequality::LTwoCounter}
	Let the histogram transform $H_n$ be defined as in \eqref{equ::HT} with bandwidth $h_n$ satisfying Assumption \ref{assumption::h}. Moreover, let the density function $f \in C^{1,\alpha}$ and $\mathcal{A}_f^1$ be the set \eqref{DegenerateSetF}. Then for all $x \in B_{r, \sqrt{d} \cdot \overline{h}_{0,n}}^+ \cap \mathcal{A}_f^1$ and all $n \geq N'$ with 
	\begin{align} \label{MinimalNumberSample}
		N' := \min \biggl\{ n \in \mathbb{N} :
		\overline{h}_{0,n} &\leq  
		\min \biggl\{
		\biggl( \frac{d \sqrt{d}}{2} \biggr)^{\frac{1}{\alpha}},
		\nonumber\\
		&\frac{\underline{c}_f}{2 d \sqrt{d} c_L},
		\biggl( \frac{1}{4 \overline{c}_f} \biggr)^{\frac{1}{d}}
		\biggr\}
		\biggr\},
	\end{align}
	there holds
	\begin{align}	\label{variancetermP}
		\|f_{\mathrm{D},H} - f_{\mathrm{P},H}\|_{L_{2}(\mu)}^2
		\geq \mu(\mathcal{A}_f^1\cap B_{R,\sqrt{d}\overline{h}_{0}}^+)\frac{\underline{c}_f}{4} 
		\cdot \overline{h}_{0,n}^{-d} \cdot n^{-1} 
	\end{align}
	in expectation with respect to $\mathrm{P}^n$.
\end{proposition}

\subsubsection{Upper Bound of $L_3$-Convergence Rate of HT} \label{sec::L3upperbound}

\begin{proposition}\label{prop::L3}
	Let the histogram transform $H_n$ be defined as in \eqref{equ::HT} with bandwidth $h_n$ satisfying Assumption \ref{assumption::h}. Furthermore, let the density function $f \in C^{1,\alpha}$ and 
	for fixed constants $\underline{c}'_f, \underline{c}_f, \overline{c}_f \in (0, \infty)$, 
	let $\mathcal{A}_f^1$ be the set \eqref{DegenerateSetF}. 
	Then for all $n > N_0$ with $N_0$ as in \eqref{MinimalNumber}, there holds	
	\begin{align*}
		\|f_{\mathrm{D},H} - f\|_{L_3(\mu)}^3 &\leq \mu(B_{R, \sqrt{d} \cdot \overline{h}_0}^+ \cap \mathcal{A}_f^1) 
		\cdot \biggl( \frac{d c_L^3}{4} \cdot \overline{h}_0^{3+\alpha} 
		\\
		&\quad+ c_{\alpha}^3 \cdot \overline{h}_0^{3(1+\alpha)}
		+ \frac{\overline{c}_f}{c_0^2} n^{-2} \overline{h}_0^{-2d}
		\\
		&\quad+ \frac{3c_L^2}{c_0^2} \cdot n^{-1} \cdot \overline{h}_{0}^{-d+1+\alpha} \biggr),
	\end{align*}
	where $c_{\alpha}$ is some constant depending on $\alpha$.
\end{proposition}

\section{Proofs}\label{sec::proofs}

It is well-known that entropy numbers are closely related to the covering numbers. To be specific, entropy and covering numbers are in some sense inverse to each other. 
More precisely, for all constants $a > 0$ and $q > 0$, the implication
\begin{align} \label{EntropyCover}
	&e_i (T, d) \leq a i^{-1/q}, \quad \forall \, i \geq 1
	\quad \\
	&\Longrightarrow 
	\quad 
	\ln \mathcal{N}(T, d, \varepsilon) \leq \ln(4) (a/\varepsilon)^q, \quad \forall \, \varepsilon > 0
\end{align}
holds by Lemma 6.21 in \cite{StCh08}. 
Additionally, Exercise 6.8 in \cite{StCh08} yields the opposite implication, namely
\begin{align} \label{CoverEntropy}
	\ln \mathcal{N}(T, d, \varepsilon) < (a/\varepsilon)^q, \quad \forall \, \varepsilon > 0 
	\quad 
	\Longrightarrow 
	\quad 
	e_i(T, d) \leq 3^{1/q} a i^{-1/q}, \quad \forall \, i \geq 1.
\end{align}

\subsection{Proof for $f \in C^{0,\alpha}$}
\subsubsection{Proof Related to Section \ref{sec::boundapproerror}}
\begin{proof}[Proof of Lemma \ref{lem::relationlogL2}]
	For any density function $g$, there holds
	\begin{align*}
		\mathcal{R}_{L,\mathrm{P}}(g) -  \mathcal{R}_{L,\mathrm{P}}(f) 
		& = - \mathbb{E}_{\mathrm{P}} \log g(X)  + \mathbb{E}_{\mathrm{P}} \log f(X)
		\\
		& = - \mathbb{E}_{\mathrm{P}} \log \frac{g(X)}{f(X)}
		\\
		& = - \mathbb{E}_{\mathrm{P}} \log \biggl( 1 + \frac{g(X) - f(X)}{f(X)} \biggr).
	\end{align*}
	Using $x - x^2/2 \leq \log(1+x) \leq x$, $x>-1$, we get
	\begin{align}\label{eq::relation1}
		-\mathbb{E}_{\mathrm{P}}&\frac{g(X)-f(X)}{f(X)} \leq \mathcal{R}_{L,\mathrm{P}}(g) -  \mathcal{R}_{L,\mathrm{P}}(f) 
		\nonumber\\
		&\leq -\mathbb{E}_{\mathrm{P}}\frac{g(X)-f(X)}{f(X)} +\mathbb{E}_{\mathrm{P}} \frac{\big(g(X)-f(X)\big)^2}{2f(X)^2}.
	\end{align}
	Since $g$ is a density function, we have
	\begin{align}\label{eq::first0}
		\mathbb{E}_{\mathrm{P}}&\frac{g(X)-f(X)}{f(X)} = \int_{\mathcal{X}} \frac{g(x)-f(x)}{f(x)} f(x)\, dx 
		\nonumber\\
		&=\int_{\mathcal{X}} g(x)\, dx - \int_{\mathcal{X}}f(x)\, dx = 1-1 = 0.
	\end{align}
	On the one hand, \eqref{eq::first0} together with the first inequality in \eqref{eq::relation1} yields
	\begin{align*}
		\mathcal{R}_{L,\mathrm{P}}(g) -  \mathcal{R}_{L,\mathrm{P}}(f) \geq 0.
	\end{align*}
	Moreover, the equation holds if and only if $g=f$. On the other hand, 
	combining the second inequality \eqref{eq::relation1} and \eqref{eq::first0}, we obtain
	\begin{align*}
		\mathcal{R}&_{L,\mathrm{P}}(g) -  \mathcal{R}_{L,\mathrm{P}}(f) \\
		&\leq \mathbb{E}_{\mathrm{P}} \frac{\big(g(X)-f(X)\big)^2}{2f(X)^2}
		= \int_{\mathcal{X}} \frac{\big(g(x)-f(x)\big)^2}{2f(x)} \,d\mu(x).
	\end{align*}
	Thus, for all $x$ satisfying $f(x)\geq \underline{c}_f$, we have
	\begin{align*}
		\mathcal{R}_{L,\mathrm{P}}(g) -  \mathcal{R}_{L,\mathrm{P}}(f) 
		\leq \frac{\|f - g\|_{L_2(\mu)}^2}{2\underline{c}_f}.
	\end{align*}
	Using $\log(1+x) \leq x-x^2/2+x^3/3$, $x>-1$, we get
	\begin{align} \label{eq::taylor2}
		\mathbb{E}_{\mathrm{P}} &\log \biggl( 1 + \frac{g(X) - f(X)}{f(X)} \biggr) 
		\nonumber\\
		& \leq \mathbb{E}_{\mathrm{P}} \frac{g(X) - f(X)}{f(X)} 
		- \frac{1}{2} \mathbb{E}_{\mathrm{P}} \biggl( \frac{g(X) - f(X)}{f(X)} \biggr)^2 
		\nonumber\\
		&+ \frac{1}{3} \mathbb{E}_{\mathrm{P}} \biggl( \frac{g(X) - f(X)}{f(X)} \biggr)^3.
	\end{align}
	Combining \eqref{eq::taylor2} with \eqref{eq::first0}, we obtain
	\begin{align*}
		- \mathbb{E}_{\mathrm{P}} &\log \biggl( 1 + \frac{g(X) - f(X)}{f(X)} \biggr) \\
		&\geq \frac{1}{2} \mathbb{E}_{\mathrm{P}} \biggl( \frac{g(X) - f(X)}{f(X)} \biggr)^2 
		-  \frac{1}{3} \mathbb{E}_{\mathrm{P}} \biggl( \frac{g(X) - f(X)}{f(X)} \biggr)^3.
	\end{align*}
	Consequently, for any $x$ satisfying $f(x) \in [\underline{c}_f, \overline{c}_f]$, there holds 
	\begin{align*}
		\mathcal{R}_{L,\mathrm{P}}(g) -  \mathcal{R}_{L,\mathrm{P}}(f) 
		\geq \frac{\|g-f\|_{L_2(\mu)}^2}{2\underline{c}_f} - \frac{\|g-f\|_{L_3(\mu)}^3}{3 \overline{c}_f^2},
	\end{align*}
	which completes the proof.
\end{proof}

\begin{proof}[Proof of Proposition \ref{pro::c0alphasingle}]
	Lemma \ref{lem::relationlogL2} together with the definition of $f_{\mathrm{P},H}$ implies
	\begin{align}\label{eq::c0approx}
		\mathcal{R}&_{L,\mathrm{P}}(f_{\mathrm{P},H}) - \mathcal{R}_{L,\mathrm{P}}^*
		\leq \frac{\|f_{\mathrm{P},H} - f\|_{L_2(\mu)}^2}{2\underline{c}_f} 
		\nonumber\\
		& = \frac{1}{2\underline{c}_f} \biggl\| \sum_{j \in \mathcal{I}_H} \frac{\eins_{A_j}(x)}{\mu(A_j)} \int_{A_j} f(x') 
		- f(x)  \, d \mu(x') \biggr\|_2^2 
		\nonumber\\
		& \leq \frac{1}{2\underline{c}_f} \biggl\| \sum_{j \in \mathcal{I}_H} \frac{\eins_{A_j}(x)}{\mu(A_j)} \int_{A_j} \big|f(x') 
		- f(x) \big| \ d \mu(x') \biggr\|_2^2 
		\nonumber\\
		& \leq \frac{1}{2\underline{c}_f} \biggl\| \sum_{j \in \mathcal{I}_H} \frac{\eins_{A_j}(x)}{\mu(A_j)} \int_{A_j} c_L\|x' - x\|^{\alpha}  
		\, d \mathrm{P}_X(x') \biggr\|_2^2 
		\nonumber\\
		& \leq \frac{1}{2\underline{c}_f} \biggl\| \sum_{j \in \mathcal{I}_H} \frac{\eins_{A_j}(x)}{\mu(A_j)} c_L(\sqrt{d}\cdot\overline{h}_0)^{\alpha} 
		\, \mu(A_j) \biggr\|_2^2 
		\nonumber\\
		& \leq \frac{c_L^2}{2\underline{c}_f}(\sqrt{d} \cdot \overline{h}_0)^{2\alpha} \mu(B_R)
		\nonumber\\
		& \leq (2\overline{c}_f)^{-1} \mu(B_R)  d^{\alpha} c_0^{- 2 \alpha}c_L^2 \underline{h}_0^{2\alpha}
		\nonumber\\
		&= c_{\alpha,d,R}\underline{h}_0^{2\alpha},
	\end{align}
	where the second last inequality is due to assumption $f \in C^{0,\alpha}$ and the last inequality follows from Assumption \ref{assumption::h}.
	Consequently we obtain
	\begin{align*}
		\lambda h^{-2d} + &\mathcal{R}_{L,\mathrm{P}}(f_{\mathrm{P},H}) - \mathcal{R}_{L,\mathrm{P}}^* \leq
		\\
		&\lambda \underline{h}_0^{-2d} + (2\overline{c}_f)^{-1} \mu(B_R)  d^{\alpha} c_0^{- 2 \alpha}c_L^2 \underline{h}_0^{2\alpha}
	\end{align*}
	Taking 
	\begin{align*}
		\overline{h}_0 := c_{\alpha,d,R}^{-\frac{1}{2d+2\alpha}} \lambda^{\frac{1}{2d+2\alpha}},
	\end{align*}
	we have 
	\begin{align*}
		\lambda h^{-2d} + \mathcal{R}_{L,\mathrm{P}}(f_{\mathrm{P},H}) - \mathcal{R}_{L,\mathrm{P}}^*  \leq 2c_{\alpha,d,R}^{\frac{d}{d+\alpha}} \lambda^{\frac{\alpha}{d+\alpha}} := c\lambda^{\frac{\alpha}{d+\alpha}},
	\end{align*}
	which yields the assertion.
\end{proof}

\subsubsection{Proof Related to Section \ref{sec::boundsamplerror}} \label{sec::proofrelatsample}

\begin{proof}[Proof of Lemma \ref{VCIndex}]
	This proof is conducted from the perspective of geometric constructions. 
	
	\begin{figure*}[htbp]
		\centering
		\resizebox{0.7\textwidth}{!}{
			\begin{minipage}[b]{0.18\textwidth}
				\centering
				\includegraphics[width=\textwidth]{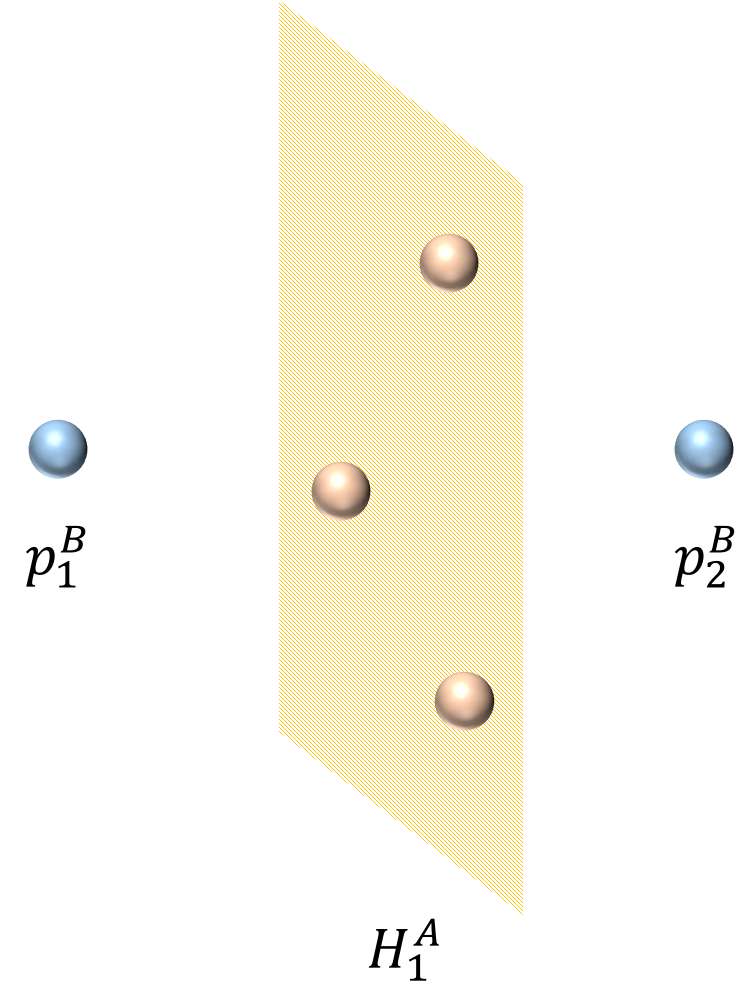}
				$p=1$
				\centering
				\label{fig::p=1}
			\end{minipage}
			\qquad
			\begin{minipage}[b]{0.25\textwidth}
				\centering
				\includegraphics[width=\textwidth]{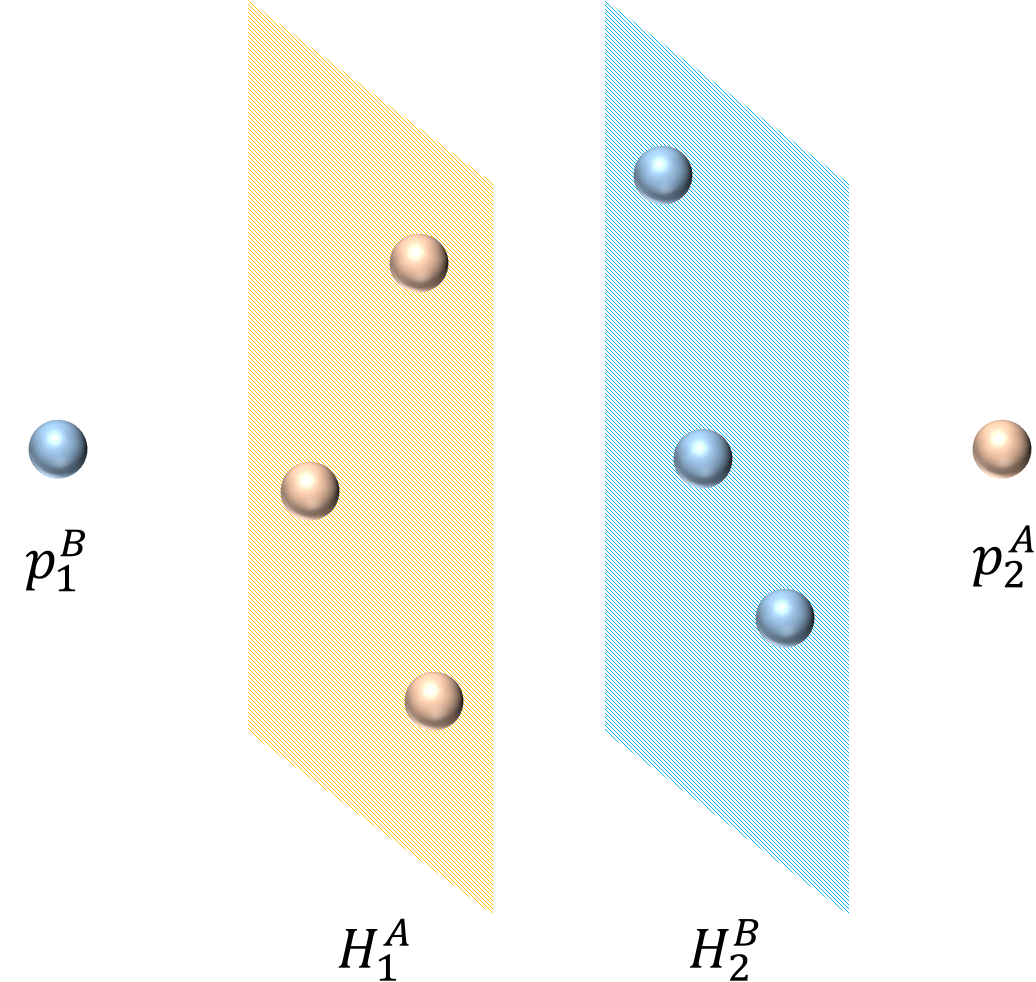}
				$p=2$
				\label{fig::p=2}
			\end{minipage}
			\qquad
			\begin{minipage}[b]{0.43\textwidth}
				\centering
				\includegraphics[width=\textwidth]{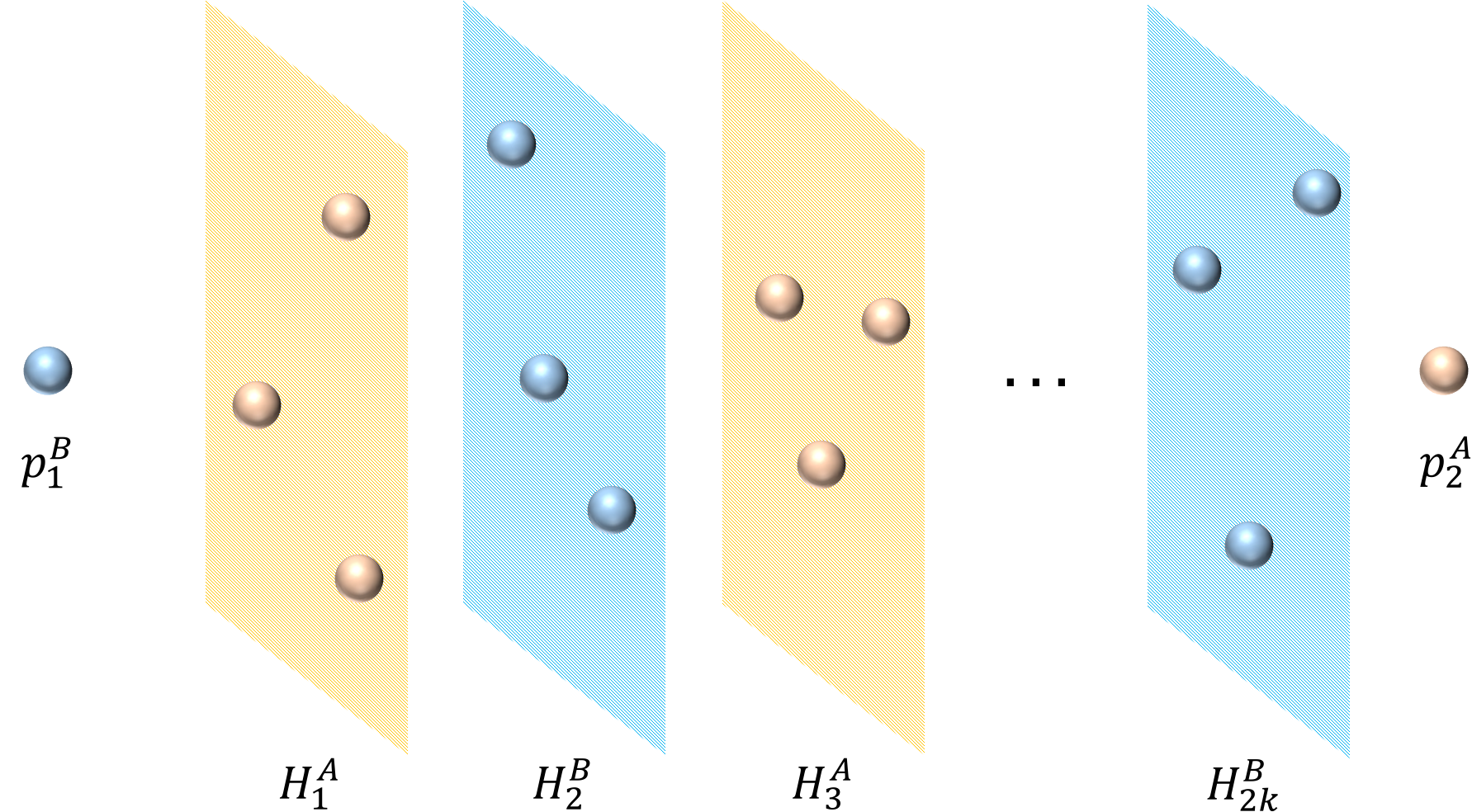}
				$p=2k$
				\label{fig::p=2k}
			\end{minipage}
		}
		\caption{We take one case with $d=3$ as an example to illustrate the geometric interpretation of the VC dimension. The yellow balls represent samples from class $A$, blue ones are from class $B$ and slices denote the hyper-planes formed by samples. }
		\label{fig::VC}
	\end{figure*}
	
	We proceed by induction.
	Firstly, we concentrate on partition with the number of splits $p=1$. Because of the dimension of the feature space is $d$,  the smallest number of sample points that cannot be divided by $p=1$ split is $d+2$. Concretely, owing to the fact that $d$ points can be used to form $d-1$ independent vectors and hence a hyperplane in a $d$-dimensional space, we might take the following case into consideration: There is a hyperplane consisting of $d$ points all from one class, say class $A$, and two points $p_1^B$, $p_2^B$ from the opposite class $B$ located on the opposite sides of this hyperplane, respectively. We denote this hyperplane by $H_1^A$. In this case, points from two classes cannot be separated by one split (since the positions are $p_1^B, H_1^A, p_2^B$), so that we have $\mathrm{VC}(\mathcal{B}_1) \leq d + 2$.

	Next, when the partition is with the number of splits $p=2$, we analyze in the similar way only by extending the above case a little bit. Now, we pick either of the two single sample points located on opposite side of the $H_1^A$, and add $d-1$ more points from class $B$ to it. Then, they together can form a hyperplane $H_2^B$ parallel to $H_1^A$. After that, we place one more sample point from class $A$ to the side of this newly constructed hyperplane $H_2^B$. In this case, the location of these two single points and two hyperplanes are $p_1^B, H_1^A, H_2^B, p_2^A$. Apparently, $p=2$ splits cannot separate these $2d+2$ points. As a result, we have $\mathrm{VC}(\mathcal{B}_2) \leq 2d + 2$.

	Inductively, the above analysis can be extended to the general case of number of splits $p \in \mathbb{N}$. In this manner, we need to add points continuously to form $p$ mutually parallel hyperplanes where any two adjacent hyperplanes should be constructed from different classes. Without loss of generality, we consider the case for $p=2k+1$, $k \in \mathbb{N}$, where two points (denoted as $p_1^B$, $p_2^B$) from class $B$ and $2k+1$ alternately appearing hyperplanes form the space locations: $p_1^B, H_1^A, H_2^B, H_3^A, H_4^B, \ldots, H_{(2k+1)}^A, p_2^B$. 
	Accordingly, the smallest number of points that cannot be divided by $p$ splits is $dp+2$, leading to $\mathrm{VC}(\mathcal{B}_p) \leq d p + 2$. This completes the proof.
\end{proof}

\begin{proof}[Proof of Lemma \ref{lem::VCFn}]
	Recall that for a histogram transform $H$, the set $\pi_H =(A_j)_{j\in \mathcal{I}_H}$  is a partition of $B_R$ with the index set $\mathcal{I}_{H}$ induced by $H$. 
	The choice $k := \lfloor 2 R \sqrt{d} / \underline{h}_0 \rfloor + 1$ leads to the partition of $B_R$ of the form $\pi_k := \{ A_{i_1, \ldots, i_d} \}_{i_j = 1,\ldots,k}$ with
	\begin{align} \label{def::cells} 
		A_{i_1, \ldots, i_d} 
		&:= \prod_{j=1}^d A_j 
		\nonumber\\
		&:= \prod_{j=1}^d \biggl[ - R + \frac{2R(i_j-1)}{k}, -R+\frac{2R i_j}{k} \biggr).
	\end{align}
	Obviously, we have $|A_{i_j}| \leq \frac{\underline{h}_0}{\sqrt{d}}$. Let $D$ be a data set of the form
	\begin{align*}
		D := \{ (x_i, t_i) : x_i \in B_R, t_i \in [-M, M], i = 1, \cdots, m \}
	\end{align*}
	with 
	\begin{align*}
		m := \#(D) = 2^{d+1}(d+1) \bigl( \lfloor 2 R \sqrt{d} / \underline{h}_0 \rfloor + 1 \bigr)^d. 
	\end{align*}
	Then there exists at least one cell $A$ with 
	\begin{align} \label{DcapANo}
		\#(D \cap (A\times [-M,M])) \geq 2^{d+1}(d+1).
	\end{align}
	Moreover, for any $x, x' \in A$, the construction of the partition \eqref{def::cells} implies $\|x - x'\| \leq \underline{h}_0$.
	Consequently, for any arbitrary histogram transform $H$ and $A_j\in \pi_H$, at most one vertex of $A_j$ lies in $A$, since the bin width of $A_j$ is larger than $\underline{h}_0$. Therefore, 
	\begin{align*} 
		\Pi_{H|A} 
		:= &\biggl\{ \bigcup_{j \in I} \bigl( (A_j \cap A) \times [-M, c_j] \bigr), I \subset \mathcal{I}_H \biggr\}
		\cup \\
		&\biggl\{ \bigcup_{j \in I} \bigl( (A_j \cap A) \times (c_j, M] \bigr), I \subset \mathcal{I}_H \bigg\}
	\end{align*}
	forms a partition of $A\times [-M,M]$ with $\#({\Pi}_{H|A}) \leq 2^{d+1}$. It is easily seen that this partition can be generated by $2^{d+1}-1$ splitting hyperplanes on the space $A\times [-M,M]$. In this way, Lemma \ref{VCIndex} implies that ${\Pi}_{H|A}$ can only shatter a dataset with at most $(d+1)(2^{d+1}-1)+1$ elements. Thus \eqref{DcapANo} indicates that ${\Pi}_{H|A}$ fails to shatter $D \cap (A\times[-M,M])$. Therefore, the subgraphs of $\mathcal{F}$
	\begin{align*}
		\big\{\{(x,t):t<f(x)\},f\in \mathcal{F}\big\}
	\end{align*}
	cannot shatter the data set $D$ as well. By Definition \ref{def::VC dimension}, we immediately get 
	\begin{align*} 
		\mathrm{VC}(\mathcal{F}) \leq 
		2^{d+1}(d+1) \bigl( \lfloor 2 R \sqrt{d} / \underline{h}_0 \rfloor + 1 \bigr)^d
	\end{align*}
	and the assertion is thus proved.
\end{proof}

\begin{proof}[Proof of Lemma \ref{thm::vart}]
	Let $\mathcal{F}_{\varepsilon}$ be an $\varepsilon$-net over $\mathcal{F}$. 
	Then, for any $f \in \mathrm{Co}(\mathcal{F})$, there exists an $f_{\varepsilon} \in \mathrm{Co}(\mathcal{F}_{\varepsilon})$
	such that $\|f - f_{\varepsilon}\|_{L_2(\mathrm{Q})} \leq \varepsilon$.
	Therefore, we can assume without loss of generality that $\mathcal{F}$ is finite.
	
	Obviously, \eqref{CoverAssumption} holds for $1 \leq \varepsilon \leq c^{1/v}$.
	Let $v' := 1/2 + 1/v$ and $M' := c^{1/v}M$. Then \eqref{CoverAssumption} implies that
	for any $n \in \mathbb{N}$, there exists $f_1, \ldots, f_n \in \mathcal{F}$ such that
	for any $f \in \mathcal{F}$, there exists an $f_i$ such that
	\begin{align*}
		\|f - f_i\|_{L_2(\mathrm{Q})} \leq M' n^{-1/v}.
	\end{align*}
	Therefore, for each $n \in \mathbb{N}$, we can find
	sets $\mathcal{F}_1 \subset \mathcal{F}_2 \subset \cdots \subset \mathcal{F}$ such that 
	the set $\mathcal{F}_n$ is a $M' n^{-1/v}$-net over $\mathcal{F}$ and $\#(\mathcal{F}_n) \leq n$.

	In the following, we show by induction that for $q \geq 3 + v$ and $n, k \geq 1$, there holds
	\begin{align}\label{eq::main}
		\log \mathcal{N} \bigl( \mathrm{Co}(\mathcal{F}_{nk^q}), L_2(\mathrm{Q}), c_k M' n^{-v'} \bigr)
		\leq c'_k n,
	\end{align}
	where $c_k$ and $c'_k$ are constants depending only on $c$ and $v$ such that $\sup_k \max \{ c_k, c'_k \} < \infty$.
	The proof of \eqref{eq::main} will be conducted by a nested induction argument.

	Let us first consider the case $k = 1$. For a fixed $n_0$, let $n \leq n_0$. Then for $c_1$ satisfying $c_1 M' n_0^{-v'} \geq M$, there holds 
	\begin{align*}
		\log \mathcal{N} \bigl( \mathrm{Co}(\mathcal{F}_{nk^q}), L_2(\mathrm{Q}), c_k M' n^{-v'} \bigr) = 0,
	\end{align*}
	which immediately implies \eqref{eq::main}. For a general $n \in \mathbb{N}$, let $m := n/\ell$ for large enough $\ell$ to be chosen later. Then for any $f \in \mathcal{F}_n \setminus \mathcal{F}_m$, there exists an $f^{(m)} \in \mathcal{F}_m$ such that
	\begin{align*}
		\|f - f^{(m)}\|_{L_2(\mathrm{Q})} \leq M' m^{-1/v}.
	\end{align*}
	Let $\pi_m : \mathcal{F}_n \setminus \mathcal{F}_m \to \mathcal{F}_m$ be the projection operator. Then for any $f \in \mathcal{F}_n \setminus \mathcal{F}_m$, there holds
	\begin{align*}
		\|f - \pi_m f\|_{L_2(\mathrm{Q})} \leq M' m^{-1/v}
	\end{align*}
	Therefore, for $\lambda_i, \mu_j \geq 0$ and $\sum_{i=1}^n \lambda_i = \sum_{j=1}^m \mu_j = 1$, we have
	\begin{align*}
		\sum_{i=1}^n \lambda_i f^{(n)}_i 
		= \sum_{j=1}^m \mu_j f^{(m)}_j 
		+ \sum_{k=m+1}^n \lambda_k \bigl( f^{(n)}_k - \pi_m f^{(n)}_k \bigr).
	\end{align*}
	Let $\mathcal{G}_n$ be the set 
	\begin{align*}
		\mathcal{G}_n := \{ 0 \} \cup \{ f - \pi_m f : f \in \mathcal{F}_n \setminus \mathcal{F}_m \}.
	\end{align*}
	Then we have $\#(\mathcal{G}_n) \leq n$ and for any $g \in \mathcal{G}_n$, there holds
	\begin{align*}
		\|g\|_{L_2(\mathrm{Q})} 
		\leq M'm^{-1/v}.
	\end{align*}
	Moreover, we have
	\begin{align} \label{SpaceSplit}
		\mathrm{Co}(\mathcal{F}_n) \subset \mathrm{Co}(\mathcal{F}_m) + \mathrm{Co}(\mathcal{G}_n).
	\end{align}
	
	Applying Lemma 2.6.11 in \cite{van1996weak} with $\varepsilon := \frac{1}{2} c_1 m^{1/v} n^{-v'}$ 
	to $\mathcal{G}_n$, we can find a $\frac{1}{2}c_1 M'n^{-v'}$-net over $\mathrm{Co}(\mathcal{G}_n)$ consisting of at most 
	\begin{align} \label{CapacityCoGn}
		(e + e n \varepsilon^2)^{2/\varepsilon^2} 
		\leq \biggl( e + \frac{e c_1^2}{\ell^{2/v}} \biggr)^{8 \ell^{2/v} c_1^{-2} n}
	\end{align}
	elements.
	
	Suppose that \eqref{eq::main} holds for $k = 1$ and $n = m$.
	In other words, there exists a $c_1 M' m^{-v'}$-net over $\mathrm{Co}(\mathcal{F}_m)$ consisting of at most $e^m$ elements, which partitions $\mathrm{Co}(\mathcal{F}_m)$ into $m$-dimensional cells of diameter at most $2 c_1 M' m^{-v'}$. 
	Each of these cells can be isometrically identified with a subset of a ball of radius $c_1 M' m^{-v'}$ in $\mathbb{R}^m$
	and can be therefore further partitioned into
	\begin{align*}
		\bigg(\frac{3c_1 M' m^{-v'}}{\frac{1}{2} c_1 M' n^{-v'}} \bigg)^m = (6\ell^{v'})^{n/\ell}
	\end{align*}
	cells of diameter $\frac{1}{2}c_1 M' n^{-v'}$. As a result, we get a $\frac{1}{2}c_1 M' n^{-v'}$-net of $\mathrm{Co}(\mathcal{F}_m)$ containing at most 
	\begin{align} \label{CapacityCoFm}
		e^m \cdot (6\ell^{v'})^{n/\ell}
	\end{align}
	elements.

	Now, \eqref{SpaceSplit} together with \eqref{CapacityCoGn} and \eqref{CapacityCoFm}
	yields that there exists a $c_1 M' n^{-v'}$-net of $\mathrm{Co}(\mathcal{F}_n)$ 
	whose cardinality can be bounded by 
	\begin{align*}
		e^{n/\ell} \bigl( 6 \ell^{v'} \bigr)^{n/\ell} 
		\biggl( e + \frac{e c_1^2}{\ell^{2/v}} \biggr)^{8 \ell^{2/v} c_1^{-2} n}
		\leq e^n,
	\end{align*}
	for suitable choices of $c_1$ and $\ell$ depending only on $v$. 
	This concludes the proof of \eqref{eq::main} for $k=1$ and every $n \in \mathbb{N}$.

	Let us consider a general $k \in \mathbb{N}$. 
	Similarly as above, there holds
	\begin{align} \label{SpaceSplitGeneral}
		\mathrm{Co}(\mathcal{F}_{nk^q}) \subset \mathrm{Co}(\mathcal{F}_{n(k-1)^q}) + \mathrm{Co}(\mathcal{G}_{n,k}),
	\end{align}
	where the set $\mathcal{G}_{n,k}$ contains at most $n k^q$ elements with norm smaller than $M'(n(k-1)^q)^{-1/v}$.
	Applying Lemma 2.6.11 in \cite{van1996weak} to $\mathcal{G}_{n,k}$, we can find an $M'k^{-2}n^{-v'}$-net over $\mathrm{Co}(\mathcal{G}_{n,k})$ consisting of at most 
	\begin{align} \label{CapacityCoGnk}
		\bigl( e + e k^{2q/v-4+q} \bigr)^{2^{2q/v+1}k^{4-2q/v}n}
	\end{align}
	elements. Moreover, by the induction hypothesis, we have a $c_{k-1}M'n^{-v'}$-net over $\mathrm{Co}(\mathcal{F}_{n(k-1)^q})$ consisting of at most 
	\begin{align} \label{CapacityCoFnk-1q}
		e^{c'_{k-1}n}
	\end{align}
	elements. Using \eqref{SpaceSplitGeneral}, \eqref{CapacityCoGnk}, and \eqref{CapacityCoFnk-1q},
	we obtain a $c_k M' n^{-v'}$-net over $\mathrm{Co}(\mathcal{F}_{nk^q})$ consisting of at most $e^{c'_k n}$ elements, where
	\begin{align*}
		c_k &= c_{k-1} + \frac{1}{k^2},
		\\
		c'_k &= c'_{k-1} + 2^{2q/v+1}\frac{1+\log(1+k^{2q/v-4+q})}{k^{2q/v-4}}.
	\end{align*}
	Form the elementary analysis we know that if $2q/v - 5 = 2$, 
	then there exist constants $c''_1$, $c''_2$, and $c''_3$ such that
	\begin{align*}
		\lim_{k \to \infty} c_k 
		& = c^{-1/v} n_0^{(v+2)/2v} + \sum_{i=2}^{\infty} 1/i^2 
		\leq c''_1 c^{-1/v} + c''_2,
		\\
		\lim_{k \to \infty} c'_k 
		& = 1 + c \sum_{i=1}^{\infty} 2 (2/i)^{2q/v}i^5 
		\leq c''_3.
	\end{align*}
	Thus \eqref{eq::main} is proved. 
	Taking $\varepsilon := c_k M' n^{-v'} / M$ in \eqref{eq::main}, we get 
	\begin{align*}
		\log \mathcal{N} ( \mathrm{Co}&(\mathcal{F}_{nk^q}), L_2(\mathrm{Q}), M \varepsilon)
		\leq \\
		&c'_k c_k^{1/v'} (M')^{1/v'}M^{-1/v'}\varepsilon^{-1/v'}.
	\end{align*}
	This together with 
	\begin{align*}(M')^{1/v'} = (c^{1/v}M)^{1/v'}=c^{2/(v+2)}M^{1/v'}
	\end{align*}
	yields
	\begin{align*}
		\log \mathcal{N} ( \mathrm{Co}(\mathcal{F}), L_2(\mathrm{Q}), M \varepsilon)
		\leq c' c^{2/(v+2)}\varepsilon^{-2v/(v+2)},
	\end{align*}
	where the constant $c'$ depends on the constants $c''_1$, $c''_2$ and $c''_3$.
	This finishes the proof.
\end{proof}

\begin{proof}[Proof of Theorem \ref{the::Fncovering}]
	We find the upper bound of $\mathrm{VC}(\mathcal{F})$ satisfies
	\begin{align*} 
		2^{d+1}(d+1)&(2R\sqrt{d}/\underline{h}_0+2)^d
		\leq \\
		&d\cdot2^{d+2} (4R\sqrt{d}/\underline{h}_0)^d
		= (c_d R/\underline{h}_0)^d,
	\end{align*}
	where $c_d:= 2^{1+4/d}\cdot d^{1/2+1/d}$. 
	Then Theorem 2.6.7 in \cite{van1996weak} yields the assertion.
\end{proof}

\begin{proof}[Proof of Theorem \ref{the::convexFn}]
	The assertion follows directly from Lemma \ref{thm::vart} with 
	\begin{align*}
		c := c_0 (c_d/\underline{h}_0)^d \cdot (16e)^{(c_d/\underline{h}_0)^d}, 
		\quad
		v := 2((c_d/\underline{h}_0)^d-1).
	\end{align*}
	Let $\delta := (\underline{h}_0/c_d)^d$, then we have
	\small
	\begin{align*}
		c^{2/(v+2)} 
		= (c_0\delta^{-1}(16e)^{1/\delta})^{\delta}
		= 16 e (c_0\delta^{-1})^{\delta}
		= 16e(c_0\delta^{-1})^{\delta}.
	\end{align*}
	\normalsize
	Note that the function $f$ defined by $f(\delta) := (c_0\delta^{-1})^{\delta}$ is continuous and 
	\begin{align*}
		\lim_{\delta\to 0} f(\delta) = 1.
	\end{align*}
	Then there exists a constant $M_{d} > 0$ such that $f(\delta) \leq M_d$ for all $0<\delta\leq(1/c_d)^d$ if $\underline{h}_0\leq 1$.
	Consequently, we have
	\begin{align*}
		\log \mathcal{N} ( \mathrm{Co}(\mathcal{F}), L_2(\mathrm{Q}), M \varepsilon) 
		\leq 16 e c' M_{d} \varepsilon^{2(\underline{h}_{0.n}/c_d)^d-2}.
	\end{align*}
	With $c_1 := 16 e c' M_d$ we obtain the assertion.
\end{proof}

\begin{definition}\label{def::variancebound}
	Let $f$ be density function and $\mathrm{P}$ be the corresponding probability distribution on $\mathcal{X}$.
	For a loss function $L: \mathcal{X} \times [0,\infty] \to \mathbb{R}$ and denote $L\circ g:= L(x, g(x))$, Then $L$ satisfies the supreme bound and variance bound if there exist constants $B>0$, $\theta\in [0,1]$ and $V\geq B^{2-\theta}$ such that for any function $g$, there holds 
	\begin{align*}
		\|L\circ g - L\circ f\|_{\infty} & \leq B,
		\\
		\mathbb{E}_{\mathrm{P}}(L\circ g - L\circ f)^2 & \leq V \cdot( \mathbb{E}_{\mathrm{P}} (L\circ g - L\circ f))^{\theta}.
	\end{align*}
\end{definition}

\begin{lemma}\label{lem::variancebound}
	Let $L$ be the negative log-likelihood loss defined in \eqref{eq::L}. Moreover, 
	let $f$ be the underlying density function of the probability distribution $\mathrm{P}$ on $B_R$ satisfying
	$\underline{c}_f \leq f(x) \leq \overline{c}_f$ for all $x \in B_R$.
	Then for any $g$ with  $\underline{c}_f \leq g(x) \leq \overline{c}_f$,
	$L$ satisfies the supreme bound and variance bound in Definition \ref{def::variancebound} with $B=2\max\{|\log\underline{c}_f|, |\log\overline{c}_f|\}$ and $V = 2\max\{1,|\log\underline{c}_f|, |\log\overline{c}_f|\}$, $\theta = 1$.
\end{lemma}

\begin{proof}[Proof of Lemma \ref{lem::variancebound}]
	First any $x \in B_R$, there holds
	\begin{align*}
		\|L \circ g - L \circ f\|_{\infty} 
		& \leq \max_{x\in B_R} \log |f(x)| + \max_{x\in B_R} \log |g(x)| 
		\\
		& \leq 2 \max\{|\log\underline{c}_f|, |\log\overline{c}_f|\} =: B.
	\end{align*}
	Using Taylor's expansion, we get
	\begin{align*}
		\mathbb{E}_{\mathrm{P}}(&L\circ g - L\circ f)^2
		= \mathbb{E}_{\mathrm{P}} \bigl( - \log g(x) + \log f(x) \bigr)^2
		\\
		& = \mathbb{E}_{\mathrm{P}} \biggl( - \log \biggl( 1 + \frac{g(x) - f(x)}{f(x)} \biggr) \biggr)^2
		\\
		& \leq \mathbb{E}_{\mathrm{P}} \biggl( \frac{g(x) - f(x)}{f(x)} - \frac{(g(x) - f(x))^2}{2 f(x)^2} \biggr)^2
		\\
		&= \mathbb{E}_{\mathrm{P}} \biggl( \biggl( \frac{g(x) - f(x)}{f(x)} \biggr)^2 
		- \biggl( \frac{g(x) - f(x)}{f(x)} \biggr)^3 
		\\
		&\quad+ o \biggl( \biggl( \frac{g(x) - f(x)}{f(x)} \biggr)^3 \biggr) \biggr)
	\end{align*}
	and
	\begin{align*}
		& \mathbb{E}_{\mathrm{P}}(L \circ g - L \circ f) 
		= \mathbb{E}_{\mathrm{P}} \biggl( - \log \biggl( 1 + \frac{g(x) - f(x)}{f(x)} \biggr) \biggr)
		\\
		& = \mathbb{E}_{\mathrm{P}} \biggl( - \frac{g(x) - f(x)}{f(x)} 
		+ \frac{1}{2} \biggl( \frac{g(x) - f(x)}{f(x)} \biggr)^2 
		\\
		&\quad- \frac{1}{3} \biggl( \frac{g(x) - f(x)}{f(x)} \biggr)^3 
		+ o \biggl( \biggl( \frac{g(x) - f(x)}{f(x)} \biggr)^3 \biggr) \biggr)
		\\
		& =  \mathbb{E}_{\mathrm{P}} \biggl( \frac{1}{2} \biggl( \frac{g(x) - f(x)}{f(x)} \biggr)^2 
		- \frac{1}{3} \biggl( \frac{g(x) - f(x)}{f(x)} \biggr)^3 
		\\
		&\quad+ o \biggl( \biggl( \frac{g(x) - f(x)}{f(x)} \biggr)^3 \biggr) \biggr),
	\end{align*}
	where the last inequality follows from 
	\begin{align*}
		\mathbb{E}_{\mathrm{P}} &\biggl( \frac{g(x) - f(x)}{f(x)} \biggr) 
		= \int_{B_R} \frac{g(x)-f(x)}{f(x)} f(x) \, dx 
		\\
		& = \int_{B_R} g(x)-f(x) \,dx 
		\\
		& =  \int_{B_R} g(x) \,dx - \int_{B_R} f(x) \,dx = 0.
	\end{align*}
	Consequently we have 
	\begin{align*}
		\mathbb{E}_{\mathrm{P}}(L\circ g - L\circ f)^2 \leq 2\mathbb{E}_{\mathrm{P}}(L\circ g - L\circ f).
	\end{align*}
	Choosing $V := \max\{ 2, B \} = 2 \max\{1, |\log\underline{c}_f|, |\log\overline{c}_f|\}$,
	we obtain the assertion.
\end{proof}

\begin{proof}[Proof of Theorem \ref{thm::OracalBoost}]
	Denote
	\begin{align*}
		r^*:=\Omega (h)+\mathcal{R}_{L,\mathrm{P}}(f)-R^*_{L,\mathrm{P}},
	\end{align*}
	and for $r > r^*$, we write
	\begin{align*}
		\mathcal{F}_r & := \{ f \in E : \Omega(h) + \mathcal{R}_{L,\mathrm{P}}(f) - \mathcal{R}^*_{L,\mathrm{P}} \leq r \},
		\\
		\mathcal{H}_r & := \{ L \circ f - L \circ f^*_{L,\mathrm{P}} : f \in \mathcal{F}_r \}.
	\end{align*}
	Note that for $f \in \mathcal{F}_r$, we have $f=\sum^T_{t=1}w_t f_t$, where $f_t\in \mathcal{F}$ and $\sum^T_{t=1}w_t=1$, 
	Consequently, we have $\mathcal{F}_r \subset co(\mathcal{F})$. 
	Since $L$ is Lipschitz continuous with $|L|_1 \leq \underline{c}_f^{-1}$, we find
	\begin{align*}
		\mathbb{E}_{D \sim \mathrm{P}^n} e_m&(\mathcal{H}_r, L_2(\mathrm{D}))
		\leq \underline{c}_f^{-1} \mathbb{E}_{D \sim \mathrm{P}_X^n} e_m(\mathcal{F}_r, L_2(\mathrm{D}))
		\\
		& \leq 2\underline{c}_f^{-1}  \mathbb{E}_{D \sim \mathrm{P}_X^n} e_m(\mathrm{Co}(\mathcal{F}), L_2(\mathrm{D})).
	\end{align*}
	Let $\delta := (\underline{h}_0/c_d)^d$, $\delta' := 1 - \delta$, and $a := c_1^{1/(2\delta')} M$. 
	Then \eqref{equ::convexFn} together with \eqref{CoverEntropy}  implies that
	\begin{align*}
		e_m(\mathrm{Co}(\mathcal{F}), L_2(\mathrm{D}))
		\leq (3c_1)^{1/(2\delta')}M i^{-1/(2\delta')}
	\end{align*}
	Taking expectation with respect to $\mathrm{P}^n$, we get
	\begin{align}\label{equ::emcofH}
		\mathbb{E}_{D \sim \mathrm{P}_X^n} e_m(\mathrm{Co}(\mathcal{F}),L_2(\mathrm{D})) \leq c_2 i^{-1/(2\delta')},
	\end{align}
	where $c_2 := (3c_1)^{1/(2\delta')}M $. Moreover, we easily find
	\begin{align*}
		\lambda h^{-2d}
		= \Omega(h)
		\leq \Omega_{\lambda}(f) + \mathcal{R}_{L,\mathrm{P}}(f)-\mathcal{R}^*_{L,\mathrm{P}}
		\leq r,
	\end{align*}
	which yields
	\begin{align*}
		\underline{h}_0^{-1} 
		\leq (r/\lambda)^{1/(2d)}.
	\end{align*}
	Therefore, if $\underline{h}_0 \leq 1$, then we have $r \geq \lambda \geq 1$ and \eqref{equ::emcofH} can be further estimated by
	\begin{align*}
		\mathbb{E}_{D \sim \mathrm{P}_X^n} e_m(\mathrm{Co}(\mathcal{F}_H), L_2(\mathrm{D}))
		\leq c_2 (r/\lambda)^{1/(4\delta')} i^{-1/(2\delta')},
	\end{align*}
	which leads to 
	\begin{align*}
		\mathbb{E}_{D \sim \mathrm{P}_X^n} e_m(\mathcal{H}_r, L_2(\mathrm{D}))
		\leq 2 c_2 \underline{c}_f^{-1} (r/\lambda)^{1/(4\delta')}i^{-1/(2\delta')}.
	\end{align*}
	For the negative log-likelihood loss $L$, Lemma \ref{lem::variancebound} implies the supreme bound 
	\small
	\begin{align*}
		L(x,t) \leq 2\max\{|\log\underline{c}_f|, |\log\overline{c}_f|\},
		\ \forall \, x \in B_R,  \, t \in [\underline{c}_f,\overline{c}_f],
	\end{align*}
	\normalsize
	and the variance bound
	\begin{align*}
		\mathrm{E}(L \circ g - L \circ f)^2
		\leq V (\mathrm{E}(L \circ g - L \circ f^*_{L,\mathrm{P}}))^{\vartheta}
	\end{align*}
	holds for $V = 2\max\{1,|\log\underline{c}_f|, |\log\overline{c}_f|\}$ and $\vartheta = 1$. Therefore, for $h \in \mathcal{H}_r$, we have
	\begin{align*}
		\|h\|_{\infty} & \leq 4\max\{|\log\underline{c}_f|, |\log\overline{c}_f|\},
		\\
		\mathbb{E}_{\mathrm{P}} h^2 & \leq 2\max\{1,|\log\underline{c}_f|, |\log\overline{c}_f|\} \cdot r.
	\end{align*}
	Then Theorem 7.16 in \cite{StCh08} with $a := 2c_2\underline{c}_f^{-1} (r/\lambda)^{1/(4\delta')}$ yields that there exist a constant $c_0' > 0$ such that
	\begin{align*}
		\mathbb{E}_{D \sim \mathrm{P}^n} &\mathrm{Rad}_D(\mathcal{H}_r,n)
		\leq c_0' \max \Bigl\{ r^{5/4-\delta'}\lambda^{-1/4}n^{-1/2},
		\\
		&r^{1/{2(1+\delta')}} \lambda^{-1/{2(1+\delta')}}n^{-1/(1+\delta')} \Bigr\}
		\\
		& =: \varphi_n(r).
	\end{align*}
	Simple algebra shows that the condition $\varphi_n(4r) \leq 2\sqrt{2} \varphi_n(r)$ is satisfied. Since $2\sqrt{2} < 4$, similar arguments show that there still hold the statements of the Peeling Theorem 7.7 in \cite{StCh08}. Consequently, Theorem 7.20 in \cite{StCh08} can also be applied, if the assumptions on $\varphi_n$ and $r$ are modified to $\varphi_n(4r) \leq 2\sqrt{2} \varphi_n(r)$ and
	$r \geq \max\{75\varphi_n(r), 1152M^2\tau/n, r^*\}$, respectively. It is easy to verify that the condition $r\geq 75\varphi_n(r)$ is satisfied if
	\begin{align*}
		r \geq c_0'\lambda^{-1/(1+2\delta')}n^{-2/(1+2\delta')},
	\end{align*}
	where $c_0'$ is a constant, which yields the assertion.
\end{proof}

\subsubsection{Proof Related to Section \ref{sec::c0}}

\begin{proof}[Proof of Theorem \ref{thm::tree}]
	It is easy to see that $f_{\mathrm{P},\mathrm{E}}$ defined by \eqref{fpeensemble} satisfies $f_{\mathrm{P},\mathrm{E}}\in E$. Moreover, by Jensen's inequality and Proposition \ref{pro::c0alphasingle}, we have
	\begin{align*}
		\mathcal{R}_{L,\mathrm{P}}(f_{\mathrm{P},\mathrm{E}})-\mathcal{R}^*_{L,\mathrm{P}}
		& = \int_{\mathcal{X}} \biggl( \frac{1}{T} \sum_{t=1}^T f_{\mathrm{P},H_t} - f \biggr)^2 \, d\mathrm{P}_X
		\\
		& \leq \frac{1}{T} \sum_{t=1}^T \int_{\mathcal{X}} (f_{\mathrm{P},H_t} - f)^2 \, d\mathrm{P}_X
		\\
		& = \frac{1}{T} \sum_{t=1}^T \mathcal{R}_{L, \mathrm{P}}(f_{\mathrm{P},H_t}) - \mathcal{R}_{L,\mathrm{P}}^*
		\\ 
		& \leq d^{\alpha} c_0^{-2\alpha} \underline{h}_0^{2\alpha}.
	\end{align*}
	Consequently we get
	\begin{align*}
		A(\lambda)
		&= \inf_{f \in E}\Omega(h) + \mathcal{R}_{L,\mathrm{P}}(f) - \mathcal{R}^*_{L,\mathrm{P}}
		\\
		&\leq  \Omega(h) + \mathcal{R}_{L,\mathrm{P}}(f_{\mathrm{P},\mathrm{E}}) - \mathcal{R}^*_{L,\mathrm{P}}
		\leq  c \lambda^{\frac{\alpha}{\alpha+d}}.
	\end{align*}
	Then, Theorem \ref{thm::OracalBoost} implies that with probability $\mathrm{P} \otimes \mathrm{P}_H$ not less than $1-3e^{-\tau}$, there holds
	\begin{align} \label{equ::c0lambda1}
		\lambda \Omega(h) &+ \mathcal{R}_{L,\mathrm{D}}(f_{\mathrm{D},\lambda}) - \mathcal{R}^*_{L,\mathrm{P}}
		\leq 
		\nonumber\\
		& 6 c \lambda^{\frac{\alpha}{\alpha+d}} + 3 c_0' \lambda^{-\frac{1}{1+2\delta'}} n^{-\frac{2}{1+2\delta'}}
		+ 3456 M^2 \tau / n,
	\end{align}
	where $c$ and $c_0'$ are constants defined as in Proposition \ref{pro::c0alphasingle} and Theorem \ref{thm::OracalBoost}. Minimizing the right hand side of \eqref{equ::c0lambda1}, we get 
	\begin{align*}
		\mathcal{R}_{L,\mathrm{P}}(f_{\mathrm{D},\lambda}) - \mathcal{R}^*_{L,\mathrm{P}}
		\leq c'' n^{-\frac{2\alpha}{(4-2\delta)\alpha+d}},
	\end{align*}
	if we choose 
	\begin{align*}
		\lambda_{n} :=n^{-\frac{2(\alpha+d)}{(4-2\delta)\alpha+d}},
		\quad
		h_{0,n} := n^{-\frac{1}{(4-2\delta)\alpha+d}},
	\end{align*}
	where $c''$ is a constant depending on $c$, $c_0'$, $d$, $M$, $R$ and $T$. Thus, the assertion is proved.
\end{proof}

\subsection{Proof for $f\in C^{1,\alpha}$}

\subsubsection{Proof Related to Section \ref{sec::upperboundconve}}

\begin{proof}[Proof of Lemma \ref{binset}]
	For any $x \in \mathbb{R}^d$, we define $b' := H(x) - \lfloor H(x) \rfloor \in \mathbb{R}^d$. Then we have $b' \sim \mathrm{Unif}(0,1)^d$ according to the definition of $H$. For any $x' \in A'_H(x)$, we define
	\begin{align*}
		z := H(x') - H(x) = (R \cdot S) (x' - x).
	\end{align*}
	Then we have
	\begin{align*}
		x' = x + (R \cdot S)^{-1} z.
	\end{align*}
	Moreover, since 
	\begin{align*}
		\lfloor H(x') \rfloor = \lfloor H(x) \rfloor,
	\end{align*}
	we have $z \in [-b', 1 - b']$. 
\end{proof}

\begin{proof}[Proof of Proposition \ref{prop::biasterm}]
	Lemma \ref{lem::relationlogL2} implies that
	the excess risk $\mathcal{R}_{L,\mathrm{P}}(f_{\mathrm{D},\mathrm{E}}) - \mathcal{R}_{L,\mathrm{P}}^*$ can be controlled by considering the $L_2$-distance $\|f_{\mathrm{D},\mathrm{E}} - f\|_{L_2(\mu)}$.
	According to the generation process, the histogram transforms $\{H_t\}_{t=1}^T$ are i.i.d. Therefore, for any $x \in B_R$, the expected approximation error term can be decomposed as follows:
	\small
	\begin{align}
		\mathbb{E}_{\mathrm{P}}  &\bigl( f_{\mathrm{P},\mathrm{E}}(x)- f(x) \bigr)^2 
		\nonumber\\
		&= \mathbb{E}_{\mathrm{P}_H} \bigl( 
		(f_{\mathrm{P},\mathrm{E}}(x) - \mathbb{E}_{\mathrm{P}_H}(f_{\mathrm{P},\mathrm{E}}(x)) )
		\nonumber\\
		&\quad+(\mathbb{E}_{\mathrm{P}_H}(f_{\mathrm{P},\mathrm{E}}(x)) - f(x)) \bigr)^2
		\nonumber\\
		& =  \mathrm{Var}(f_{\mathrm{P},\mathrm{E}}(x))
		+ (\mathbb{E}_{\mathrm{P}_H}(f_{\mathrm{P},\mathrm{E}}(x))-f(x))^2
		\nonumber\\
		& = \frac{1}{T} \cdot \mathrm{Var}_{\mathrm{P}_H}(f_{\mathrm{P}, H_1}(x))
		+ \bigl( \mathbb{E}_{\mathrm{P}_H} ( f_{\mathrm{P},H_1}(x) ) - f(x) \bigr)^2.
		\label{equ::biasvarianceDecom}
	\end{align}
	\normalsize
	In the following, for the simplicity of notations, we drop the subscript of $H_1$ and write $H$ instead of $H_1$ when there is no confusion.
	
	For the first term in \eqref{equ::biasvarianceDecom}, the assumption $f \in C^{1,\alpha}$ implies
	\begin{align}\label{equ::first}
		\mathrm{Var}_{\mathrm{P}_H} &\bigl( f_{\mathrm{P},H}(x) \bigr)
		= \mathbb{E}_{\mathrm{P}_H} \bigl( f_{\mathrm{P},H}(x) - \mathbb{E}_{\mathrm{P}_H}(f_{\mathrm{P},H}(x)) \bigr)^2
		\nonumber\\
		& \leq \mathbb{E}_{\mathrm{P}_H} \bigl( f_{\mathrm{P},H}(x) - f(x) \bigr)^2
		\nonumber\\
		& = \mathbb{E}_{\mathrm{P}_H} \biggl( \frac{1}{\mu(A_H(x))} \int_{A_H(x)} f(x') \, dx' 
		- f(x) \biggr)^2
		\nonumber\\
		& = \mathbb{E}_{\mathrm{P}_H} \biggl( \frac{1}{\mu(A_H(x))} \int_{A_H(x)} \bigl( f(x') - f(x) \bigr) \, dx' \biggr)^2
		\nonumber\\
		& \leq \mathbb{E}_{\mathrm{P}_H} \bigl( c_L \mathrm{diam} \bigl( A_H(x) \bigr) \bigr)^2
		\nonumber\\
		& \leq c_L^2 d \overline{h}_0^2.
	\end{align}

	We now consider the second term in \eqref{equ::biasvarianceDecom}. Lemma \ref{binset} implies that for any $x' \in A_H(x)$, there exist a random vector $u \sim \mathrm{Unif}[0,1]^d$ and a vector $v \in [0,1]^d$ such that 
	\begin{align} \label{xPrimex}
		x' = x + S^{-1} R^{\top} (- u + v).
	\end{align}
	Therefore, we have
	\begin{align}
		dx' 
		& = \det \biggl( \frac{dx'}{dv} \biggr) \, dv
		\nonumber\\
		& = \det \biggl( \frac{d(x + S^{-1} R^{\top}(- u + v))}{dv} \biggr) \, dv
		\nonumber\\
		& = \det (R S^{-1}) \, dv
		\nonumber\\
		& = \biggl( \prod_{i=1}^d h_i \biggr) \, dv.
		\label{JacobiTrans}
	\end{align}
	Taking the first-order Taylor expansion of $f(x')$ at $x$, we get
	\begin{align} \label{TaylorExpansion}
		f(x') - f(x) = \int_0^1 \bigl( \nabla f(x + t(x' - x)) \bigr)^{\top} (x' - x) \, dt. 
	\end{align}
	Moreover, we obviously have
	\begin{align} \label{Trivial}
		\nabla f(x)^{\top} (x' - x)
		= \int_0^1 \nabla f(x)^{\top} (x' - x) \, dt. 
	\end{align}
	Thus, \eqref{TaylorExpansion} and \eqref{Trivial} imply that for any $f \in C^{1, \alpha}$, there holds
	\small
	\begin{align*}
		\bigl| f(x')& - f(x) - \nabla f(x)^{\top} (x' - x) \bigr|
		\\
		&= \biggl| \int_0^1 \bigl( \nabla f(x + t(x' - x)) - \nabla f(x) \bigr)^{\top} (x' - x) \, dt \biggr|
		\\
		& \leq \int^1_0 c_L (t \|x' - x\|_2)^{\alpha} \|x' - x\|_2 \, dt
		\\
		& \leq c_L \|x' - x\|^{1+\alpha}. 
	\end{align*}
	\normalsize
	This together with \eqref{xPrimex} yields
	\begin{align*}
		\bigl| f(x') - f(x) - \nabla f(x)^{\top} S^{-1} R^{\top} (- u + v) \bigr|
		\leq c_L \overline{h}_0^{1+\alpha}
	\end{align*}
	and consequently there exists a constant $c_{\alpha} \in [-c_L, c_L]$ such that
	\begin{align} \label{TaylorEntwicklung}
		f(x') - f(x)
		= \nabla f(x)^{\top} S^{-1} R^{\top} (- u + v) + c_{\alpha} \overline{h}_0^{1+\alpha}.
	\end{align}
	The definition \eqref{eq::fPH} of $	f_{\mathrm{P},H}$ shows
	\begin{align*}
		f_{\mathrm{P},H}(x)
		& = \frac{1}{\mu(A_H(x))} \int_{A_H(x)} f(x') \, dx'.
	\end{align*}
	This together with \eqref{TaylorEntwicklung} and \eqref{JacobiTrans} yields
	\small
	\begin{align}
		f_{\mathrm{P},H}&(x) - f(x)
		= \frac{1}{\mu(A_H(x))} \int_{A_H(x)} f(x') \, dx' - f(x)
		\nonumber\\
		& = \frac{1}{\mu(A_H(x))} \int_{A_H(x)} \bigl( f(x') - f(x) \bigr) \, dx'
		\nonumber\\
		& = \frac{\prod_{i=1}^d h_i}{\mu(A_H(x))}\cdot
		\nonumber\\
		&\qquad \int_{[0,1]^d} \Bigl( \nabla f(x)^{\top} S^{-1} R^{\top} (- u + v)
		+ c_{\alpha} \overline{h}_0^{1+\alpha} \Bigr) \, dv 
		\nonumber\\
		& = \bigg( \int_{[0,1]^d} (- u + v)^{\top} \, dv \biggr) R S^{-1} \nabla f(x) 
		+ c_{\alpha} \overline{h}_0^{1+\alpha}
		\nonumber\\
		& = \biggl( \frac{1}{2} - u \biggr)^{\top} R S^{-1} \nabla f(x) + c_{\alpha} \overline{h}_0^{1+\alpha}.
		\label{StepOne}
	\end{align}
	\normalsize
	Since the random variables $(u_i)_{i=1}^d$ are independent and identically distributed as $\mathrm{Unif}[0, 1]$, we have
	\begin{align} \label{CrossTermPropertyy1}
		\mathbb{E}_{\mathrm{P}_H} \biggl( \frac{1}{2} - u_i \biggr) = 0,
		\quad 
		i = 1, \ldots, d.
	\end{align}
	Combining \eqref{StepOne} with \eqref{CrossTermPropertyy1}, we obtain
	\begin{align}
		\mathbb{E}_{\mathrm{P}_H} \bigl( f_{\mathrm{P},H}(x) - f(x) \bigr)
		= c_{\alpha} \overline{h}_0^{1+\alpha}
	\end{align}
	and consequently
	\begin{align}\label{equ::biasbound}
		\bigl( \mathbb{E}_{\mathrm{P}_H} ( f_{\mathrm{P},H_1}(x) ) - f(x) \bigr)^2
		\leq c_L^2 \overline{h}_0^{2(1+\alpha)}.
	\end{align}
	Combining \eqref{equ::biasvarianceDecom} with \eqref{equ::biasbound} and \eqref{equ::first}, we obtain
	\begin{align*}
		\mathbb{E}_{\mathrm{P}_H} \bigl( f_{\mathrm{P},\mathrm{E}}(x) - f(x) \bigr)^2
		\leq c_L^2 \cdot \overline{h}_0^{2(1+\alpha)} + \frac{1}{T} \cdot d c_L^2 \cdot \overline{h}_0^2.
	\end{align*}
	Taking expectation with respect to $\mu$, we get
	\begin{align*}
		\mathbb{E}_{\mathrm{P}_H} &\|f_{\mathrm{P},\mathrm{E}} - f\|_{L_2(\mu)}^2 \\
		&\leq c_L^2\mu(B_R) \cdot \overline{h}_0^{2(1+\alpha)} + \frac{1}{T} \cdot d c_L^2\mu(B_R) \cdot \overline{h}_0^2,
	\end{align*}
	This combines with Lemma \ref{lem::relationlogL2} implies
	\begin{align*}
		\mathbb{E}_{\mathrm{P}_H}&
		\bigl( \mathcal{R}_{L_{\overline{h}_0},\mathrm{P}}(f_{\mathrm{P},\mathrm{E}})-\mathcal{R}_{L_{\overline{h}_0},\mathrm{P}}^* \bigr)
		\leq \frac{\mathbb{E}_{\mathrm{P}_H}\|f_{\mathrm{P},\mathrm{E}} - f\|^2_{L_2(\mu)}}{2\underline{c}_f}
		\\
		&= \frac{c_L^2\mu(B_R)}{2\underline{c}_f} \cdot \overline{h}_0^{2(1+\alpha)} 
		+ \frac{1}{T} \cdot \frac{d c_L^2\mu(B_R)}{2\underline{c}_f} \cdot \overline{h}_0^2,
	\end{align*}
	which completes the proof.
\end{proof}

\subsubsection{Proof Related to Section \ref{sec::lowerboundconve}}

\begin{proof}[Proof of Proposition \ref{counterapprox}]
	Lemma \ref{binset} implies that for any $x' \in A_H(x)$, there exist a random vector $u \sim \mathrm{Unif}[0,1]^d$ and a vector $v \in [0,1]^d$ such that 
	\begin{align*} 
		x' = x + S^{-1} R^{\top} (- u + v).
	\end{align*}
	Then \eqref{StepOne} yields
	\begin{align}    \label{StepOneOne}
		(f_{\mathrm{P},H}(x) &- f(x))^2 
		\nonumber\\
		&= \biggl( \biggl( \frac{1}{2} - u \biggr)^{\top} R S^{-1} \nabla f(x) + c_{\alpha} \overline{h}_0^{1+\alpha} \biggr)^2.
	\end{align}
	The orthogonality of the rotation matrix $R$ in Section \ref{sub::histogram} tells us that
	\begin{align} \label{CrossTermProperty1}
		\sum_{i=1}^d R_{ij} R_{ik} = 
		\begin{cases}
			1, & \text{ if } j = k, \\
			0,& \text{ if } j \neq k
		\end{cases}
	\end{align} 
	and consequently we have
	\begin{align} \label{CrossTermProperty2}
		\sum_{i=1}^d \sum_{j \neq k} &R_{ij} R_{ik} h_j h_k \cdot \frac{\partial f(x)}{\partial x_j} \cdot \frac{\partial f(x)}{\partial x_k}
		\nonumber\\
		&= \sum_{j \neq k} h_j h_k \cdot \frac{\partial f(x)}{\partial x_j} \cdot \frac{\partial f(x)}{\partial x_k} \sum_{i=1}^d R_{ij} R_{ik}
		= 0.
	\end{align}
	Since the random variables $(u_i)_{i=1}^d$ are independent and identically distributed as $\mathrm{Unif}[0, 1]$, we have
	\begin{align} \label{CrossTermProperty3}
		\mathbb{E}_{\mathrm{P}_H} \bigg(\frac{1}{2}-u_i \bigg)=0
	\end{align}
	and
	\begin{align} \label{CrossTermProperty4}
		\mathbb{E}_{\mathrm{P}_H} \bigg(\frac{1}{2}-u_i \bigg)^2=\frac{1}{12}.
	\end{align}
	Then, for all $x \in B_{R,\sqrt{d} \cdot \overline{h}_0}^+ \cap \mathcal{A}_f^1$, \eqref{CrossTermProperty1}, \eqref{CrossTermProperty2}, \eqref{CrossTermProperty3}, and \eqref{CrossTermProperty4} yield
	\begin{align}
		\mathbb{E}_{\mathrm{P}_H}&\biggl( \biggl( \frac{1}{2} - u \biggr)^{\top} R S^{-1} \nabla f(x) \biggr)^2 
		\nonumber\\
		& = \mathbb{E}_{\mathrm{P}_H}\biggl( \sum_{i=1}^d \biggl( \frac{1}{2} - u_i \biggr) \sum_{j=1}^d R_{ij} h_j \frac{\partial f(x)}{\partial x_j} \bigg)^2 
		\nonumber\\
		& = \sum_{i=1}^d \mathbb{E}_{\mathrm{P}_H} \biggl( \frac{1}{2} - u_i \biggr)^2 \biggl(\sum_{j=1}^d R_{ij} h_j \frac{\partial f(x)}{\partial x_j} \bigg)^2
		\nonumber\\
		& = \frac{1}{12} \mathbb{E}_{\mathrm{P}_H} \sum_{i=1}^d \sum_{j=1}^d R_{ij}^2h_j^2\bigg(\frac{\partial f(x)}{\partial x_j} \bigg)^2
		\nonumber\\
		& \geq \frac{d}{12}\underline{c}_f'^2\underline{h}_0^2 \geq \frac{d}{12}\underline{c}_f'^2c_0^2\overline{h}_0^2.
		\label{ErrorTermMain}
	\end{align}
	Combining \eqref{StepOne} with \eqref{ErrorTermMain} and using \eqref{CrossTermProperty3}, we see that for all $x \in B_{R,\sqrt{d} \cdot \overline{h}_0}^+ \cap \mathcal{A}_f^1$, if
	\begin{align*}
		h_0 \leq \biggl( \frac{\sqrt{d} \underline{c}'_f c_0}{4 \sqrt{3} c_L} \biggr)^{\frac{1}{\alpha}},
	\end{align*}
	then we have
	\begin{align} \label{ApproxiamtionSingle}
		\mathbb{E}_{\mathrm{P}_H} (f_{\mathrm{P},H}(x)-f(x))^2 
		\geq \frac{d}{16}\underline{c}_f'^2c_0^2\overline{h}_0^2,
	\end{align}
	where the constant $c_0$ is as in Assumption \ref{assumption::h}. 
	Moreover, we have
	\begin{align*}
		\mathbb{E}_{\mathrm{P}_H}\|f_{\mathrm{P},H}-f\|_2^2 \geq \frac{d}{16}\mu(\mathcal{A}_f^1\cap B_{R,\sqrt{d}\overline{h}_{0}}^+)\underline{c}_f'^2c_0^2\overline{h}_0^2.
	\end{align*}
	This completes the proof.
\end{proof}

\begin{proof}[Proof of Proposition \ref{OracleInequality::LTwoCounter}]
	Recall that for a fixed histogram transform $H$, the set $\pi_H$ is defined as the collection of all cells in the partition induced by $H$, that is, $\pi_H := \{ A_j \}_{j \in \mathcal{I}_H}$. To estimate the first term in \eqref{equ::L2Decomposition}, we observe that for any $x \in B_R$, there holds
	\begin{align}
		\mathbb{E}_{\mathrm{P}^n} &\bigl( (f_{\mathrm{D},H}(x) - f_{\mathrm{P},H}(x))^2 | \pi_H \bigr)
		= \mathrm{Var}_{\mathrm{P}^n} \bigl( f_{\mathrm{D},H}(x) | \pi_H \bigr)
		\nonumber\\
		& = \mathrm{Var}_{\mathrm{P}^n} \biggl( \frac{1}{n \mu(A_H(x))} \sum_{i=1}^n \eins_{\{ x_i \in A_H(x) \}} \bigg| \pi_H \biggr)
		\nonumber\\
		& \geq \frac{1}{n^2\overline{h}_{0,n}^{2d}} \sum_{i=1}^n \mathrm{P}(A_H(x)) (1 - \mathrm{P}(A_H(x)))
		\nonumber\\
		& = \frac{1}{n\overline{h}_{0,n}^{2d}} \mathrm{P}(A_H(x))(1 - \mathrm{P}(A_H(x))),
		\label{ConditionalSampleError}
	\end{align}
	where $\mathbb{E}_{\mathrm{P}^n}(\cdot | \pi_H)$ and $\mathrm{Var}_{\mathrm{P}^n}(\cdot | \pi_H)$ denote the conditional expectation and conditional variance with respect to $\mathrm{P}^n$ on the partition $\pi_H$, respectively. 
	
	Lemma \ref{binset} implies that for any $x' \in A_H(x)$, there exist a random vector $u \sim \mathrm{Unif}[0,1]^d$ and a vector $v \in [0,1]^d$ such that 
	\begin{align*} 
		x' = x + S^{-1} R^{\top} (- u + v).
	\end{align*}
	By \eqref{TaylorEntwicklung} and \eqref{JacobiTrans}, there exists a constant $\theta \in (0,1)$ such that
	\small
	\begin{align}
		&\displaystyle \mathrm{P}(A_H(x)) 
		= \int_{A_H(x)} f(x') \, dx'
		\nonumber\\
		& = \biggl( \prod_{i=1}^d h_i \biggr) 
		\biggl( \int_{[0,1]^d}  f(x) + \nabla f(x+\theta S^{-1} R^{\top} (- u + v))^{\top}\cdot
		\nonumber\\
		&\quad S^{-1} R^{\top} (- u + v) \, dv \biggl)
		\nonumber\\
		& = \biggl( \prod_{i=1}^d h_i \biggr) \biggl( 
		f(x) + \int_{[0,1]^d} \nabla f(x+\theta S^{-1} R^{\top}(- u + v))^{\top}\cdot
		\nonumber\\
		&\quad S^{-1} R^{\top} (- u + v) \, dv  \biggr)
		\nonumber\\
		& = \biggl( \prod_{i=1}^d h_i \biggr) \biggl( 
		f(x) + \bigg(\int_{[0,1]^d} (-u+v)^{\top} \ dv\bigg) R S^{-1}\cdot
		\nonumber\\
		&\quad  \nabla f(x+\theta S^{-1} R^{\top} (- u + v))  \biggr)
		\nonumber\\
		& = \biggl( \prod_{i=1}^d h_i \biggr) \biggl( 
		f(x) + \biggl( \frac{1}{2} - u \biggr)^{\top} R S^{-1}\cdot
		\nonumber\\
		&\quad  \nabla f(x+\theta S^{-1} R^{\top} (- u + v))  \biggr).
		\label{PAHx}
	\end{align}
	\normalsize
	Elementary Analysis tells us that for any $a_1, \ldots, a_d \in \mathbb{R}$, there holds
	\begin{align*} 
		\frac{a_1 + \ldots + a_d}{d} 
		\leq \sqrt{\frac{a_1^2 + \ldots + a_d^2}{d}},
	\end{align*}
	which implies that
	\begin{align*} 
		\biggl| \biggl( \frac{1}{2} - u \biggr)^{\top} R S^{-1} \nabla f(x+&\theta S^{-1} R^{\top} (- u + v)) \biggr|
		\\
		&\leq d \cdot \frac{3}{2}  \cdot  \overline{h}_0 \cdot c_L
		= \frac{3dc_L}{2} \cdot  \overline{h}_0.
	\end{align*}
	This together with \eqref{PAHx} yields that for all $x \in B_{r, \sqrt{d} \cdot \overline{h}_0}^+ \cap \mathcal{A}_f^1$, there hold
	\begin{align}
		\mathrm{P}(A_H(x)) 
		\leq \overline{h}_0^d \biggl( \overline{c}_f 
		+ \frac{3dc_L}{2} \cdot  \overline{h}_0 \biggr)
	\end{align}
	and
	\begin{align}
		\mathrm{P}(A_H(x)) 
		\geq \overline{h}_0^d \biggl( \underline{c}_f
		- \frac{3dc_L}{2} \cdot  \overline{h}_0 \biggr).
	\end{align}
	Then for any $n > N'$ with $N'$ as in \eqref{MinimalNumberSample}, we have
	\begin{align} \label{PAHxSize}
		\frac{1}{2} \underline{c}_f \overline{h}_0^d
		\leq \mathrm{P}(A_H(x)) 
		\leq 2 \overline{c}_f \overline{h}_0^d
		\leq \frac{1}{2}.
	\end{align}
	Combining \eqref{ConditionalSampleError} with \eqref{PAHxSize}, we obtain
	\begin{align*}
		\mathbb{E}_{\mathrm{P}^n} &\bigl( ( f_{\mathrm{D},H}(x) - f_{\mathrm{P},H}(x) )^2 | \pi_H \bigr)
		\\
		&\geq \frac{\mathrm{P}(A_H(x)) (1 - \mathrm{P}(A_H(x)))}{n\overline{h}_{0,n}^{2d}}
		\\
		& \geq \frac{\mathrm{P}(A_H(x))}{2 n\overline{h}_{0,n}^{2d}}
		\geq \frac{\underline{c}_f \overline{h}_{0,n}^{d}}{4 n \overline{h}_{0,n}^{2d}}
		= \frac{\underline{c}_f}{4 n \overline{h}_{0,n}^{d}}.
	\end{align*}
	Consequently, for all $x \in B_{r, \sqrt{d} \cdot \overline{h}_0}^+ \cap \mathcal{A}_f^1$ and all $n \geq N'$, there holds
	\begin{align} \label{EstimationSingle}
		\mathbb{E}_{\mathrm{P}^n} \bigl( ( f_{\mathrm{D},H}(x) - f_{\mathrm{P},H}(x) )^2\bigr) 
		\geq \frac{\underline{c}_f}{4 n\overline{h}_{0,n}^{d}}.
	\end{align}
	Moreover 
	\begin{align*}
		\mathbb{E}_{\mathrm{P}^n} \bigl\|  f_{\mathrm{D},H} - f_{\mathrm{P},H}\bigr\|^2 \geq  \mu(\mathcal{A}_f^1\cap B_{R,\sqrt{d}\overline{h}_{0}}^+)\frac{\underline{c}_f}{4 n\overline{h}_{0,n}^{d}}.
	\end{align*}
	Thus, we proved the assertion.
\end{proof}

\begin{proof}[Proof of Theorem \ref{thm::LowerBoundSingles}]
	Recall the error decomposition \eqref{equ::L2Decomposition} of single random histogram transform density estimator.
	Then \eqref{ApproxiamtionSingle} and \eqref{EstimationSingle} yield that for all $x \in B_{R, \sqrt{d} \cdot \overline{h}_0}^+ \cap \mathcal{A}_f^1$ and all $n > N_0$, there holds
	\begin{align*}
		\mathbb{E}_{\mathrm{P}_H \otimes \mathrm{P}^n} &\|f_{\mathrm{D},H}-f\|^2 
		\geq 
		\\ 
		&\mu(B_{R, \sqrt{d} \cdot \overline{h}_0}^+ \cap \mathcal{A}_f^1) 
		\cdot \biggl( \frac{d}{16} \underline{c}_f'^2 c_0^2 \cdot \overline{h}_{0,n}^2 
		+ \frac{\underline{c}_f}{4 n\overline{h}_{0,n}^{d}} \biggr).
	\end{align*}
	By choosing 
	\begin{align*}
		\overline{h}_{0,n} := n^{-\frac{1}{2+d}}, 
	\end{align*}
	we obtain
	\begin{align*}
		\mathbb{E}_{\nu_n} (f_{\mathrm{D},H}(x)-f(x))^2 
		\gtrsim n^{-\frac{2}{2+d}},
	\end{align*}
	which proves the assertion.
\end{proof}

\subsubsection{Proof Related to Section \ref{sec::L3upperbound}}
\begin{proof}[Proof of Theorem \ref{thm::optimalForest}]
	Proposition \ref{thm::OracalBoost} together with Proposition \ref{prop::biasterm} implies 
	\begin{align*}
		\mathcal{R}&_{L_{\overline{h}_0}, \mathrm{P}}(f_{\mathrm{D},B}) - \mathcal{R}_{L_{\overline{h}_0}, \mathrm{P}}^*
		\\
		&\lesssim  \lambda \underline{h}_0^{-2d} + \overline{h}_0^{2(1+\alpha)} + T^{-1} \overline{h}_0^2 
		+ \lambda^{-\frac{1}{1+2\delta'}}n^{-\frac{2}{1+2\delta'}},
	\end{align*}
	where $\delta' := 1 - \delta$ and $\delta := (\underline{h}_0/c_d)^d$. Choosing 
	\begin{align*}
		\lambda_{n} &:= n^{-\frac{2(\alpha+d+1)}{2(1+\alpha)(2-\delta)+d}}, 
		\\
		\overline{h}_{0,n} &:= n^{-\frac{1}{2(1+\alpha)(2-\delta)+d}}, 
		\\
		T_n &\geq n^{\frac{2\alpha}{2(1+\alpha)(2-\delta)+d}},
	\end{align*}
	we obtain
	\begin{align*}
		\mathcal{R}_{L_{\overline{h}_0},\mathrm{P}}(f_{\mathrm{D},\lambda}) - \mathcal{R}_{L_{\overline{h}_0},\mathrm{P}}^*
		\lesssim n^{-\frac{2(1+\alpha)}{2(1+\alpha)(2-\delta)+d}}.
	\end{align*}
	This completes the proof.
\end{proof}

\begin{proof}[Proof of Proposition \ref{prop::L3}]
	By \eqref{StepOne}, we have
	\begin{align}\label{eq::L3deapprox}
		|f&_{\mathrm{P},H}(x) - f(x)|^3 
		\nonumber\\
		&= \biggl| \biggl( \frac{1}{2} - u \biggr)^{\top} R S^{-1} \nabla f(x) + c_{\alpha} \overline{h}_0^{1+\alpha} \biggr|^3
		\nonumber\\
		&= \biggl( \biggl( \frac{1}{2} - u \biggr)^{\top} R S^{-1} \nabla f(x) \biggr)^3 
		\nonumber\\
		& \phantom{=} 
		+ 3 \biggl( \biggl( \frac{1}{2} - u \biggr)^{\top} R S^{-1} \nabla f(x) \biggr)^2 c_{\alpha} \overline{h}_0^{1+\alpha} 
		\nonumber\\
		& \phantom{=} 
		+ 3 \biggl( \frac{1}{2} - u \biggr)^{\top} R S^{-1} \nabla f(x) \cdot c^2_{\alpha} \overline{h}_0^{2(1+\alpha)} 
		+ c_{\alpha}^3 \overline{h}_0^{3(1+\alpha)}.
	\end{align}
	Since the random variables $(u_i)_{i=1}^d$ are independent and identically distributed as $\mathrm{Unif}[0, 1]$, we have
	\begin{align*}
		\mathbb{E}_{\mathrm{P}_H} \biggl( \frac{1}{2} - u_i \biggr)^3 
		= \mathbb{E}_{\mathrm{P}_H} \biggl( \frac{1}{2} - u_i \biggr) 
		= 0.
	\end{align*}
	Consequently we have
	\begin{align*}
		\mathbb{E}_{\mathrm{P}_H} &\biggl( \biggl( \frac{1}{2} - u \biggr)^{\top} R S^{-1} \nabla f(x) \biggr)^3
		\\
		& = \mathbb{E}_{\mathrm{P}_H} \biggl( \sum_{i=1}^d \biggl( \frac{1}{2} - u_i \biggr) \sum_{j=1}^d R_{ij} h_j \frac{\partial f(x)}{\partial x_j} \biggr)^3 = 0,
		\\
		\mathbb{E}_{\mathrm{P}_H}&\biggl( \biggl( \frac{1}{2} - u \biggr)^{\top} R S^{-1} \nabla f(x) \biggr)
		\\
		& = \mathbb{E}_{\mathrm{P}_H} \biggl( \sum_{i=1}^d \biggl( \frac{1}{2} - u_i \biggr) \sum_{j=1}^d R_{ij} h_j \frac{\partial f(x)}{\partial x_j} \biggr) = 0.
	\end{align*}
	Moreover, \eqref{ErrorTermMain} implies
	\begin{align*}
		\mathbb{E}_{\mathrm{P}_H}& \biggl( \biggl( \frac{1}{2} - u \biggr)^{\top} R S^{-1} \nabla f(x) \biggr)^2 \\
		& = \frac{1}{12} \mathbb{E}_{\mathrm{P}_H} \sum_{i=1}^d \sum_{j=1}^d R_{ij}^2h_j^2\bigg(\frac{\partial f(x)}{\partial x_j} \bigg)^2 \leq \frac{d}{12} c_L^2 \overline{h}_0^2.
	\end{align*}
	Therefore, for any $x \in B_{R, \sqrt{d} \cdot \overline{h}_0}^+ \cap \mathcal{A}_f^1$, we have 
	\begin{align}\label{eq::L3approx}
		\mathbb{E}_{\mathrm{P}_H} |f_{\mathrm{P},H}(x) - f(x)|^3 \leq \frac{d}{4} c_L^3 \overline{h}_0^{3+\alpha} + c_{\alpha}^3 \overline{h}_0^{3(1+\alpha)}.
	\end{align}
	
	To bound the estimation error, let 
	$Y := \sum_{i=1}^n \eins_{\{ X_i \in A_H(x) \}}$
	and $\pi_H$ denote the partition of $B_R$ induced by $H$. 
	Then we have $Y \sim \mathrm{Bin} \bigl( n,\mathrm{P}(A_H(x)) \bigr)$ and
	\begin{align*}
		\mathbb{E}_{\mathrm{P}^n} &\bigl( (f_{\mathrm{D},H}(x) - f_{\mathrm{P},H}(x))^3 \big| \pi_H \bigr)
		\\& = \frac{1}{n^3 \mu(A_H(x))^3} \cdot
		\\
		&\quad \mathbb{E}_{\mathrm{P}^n} \biggl( \biggl( \sum_{i=1}^n \eins_{X_i \in A_H(x)} - n\mathrm{P}(A_H(x)) \biggr)^3 \bigg| \pi_H \biggr)
		\\
		& = \mathbb{E}_{\mathrm{P}_Y} \bigl( (Y-\mathbb{E}Y)^3 \bigr).
	\end{align*}
	Then the skewness of a binomial random variable implies that
	for any $x \in B_{R, \sqrt{d} \cdot \overline{h}_0}^+ \cap \mathcal{A}_f^1$, we have
	\begin{align}\label{eq::L3sample}
		\mathbb{E}_{\mathrm{P}^n}& \bigl( (f_{\mathrm{D},H}(x) - f_{\mathrm{P},H}(x))^3 \big| \pi_H \bigr) 
		\nonumber\\
		& = \frac{\mathrm{P}(A_H(x)) \bigl( 1 - \mathrm{P}(A_H(x)) \bigr) \bigl( 1 - 2 \mathrm{P}(A_H(x)) \bigr)}
		{n^2 \mu(A_H(x))^3} 
		\nonumber\\
		& \leq \frac{\overline{c}_f}{n^2 \underline{h}_0^{2d}}
		\leq \frac{\overline{c}_f}{c_0^2} \cdot \overline{h}_0^{-2d} \cdot n^{-2}.
	\end{align}
	
	Analogously, 
	for any $x \in B_{R, \sqrt{d} \cdot \overline{h}_0}^+ \cap \mathcal{A}_f^1$, there holds
	\begin{align}\label{eq::twoone}
		& \mathbb{E}_{\mathrm{P}^n\otimes\mathrm{P}_H} 
		\bigl( (f_{\mathrm{D},H}(x) - f_{\mathrm{P},H}(x))^2 \cdot |f_{\mathrm{P},H}(x) - f(x)| \bigr)
		\nonumber\\
		& = \mathbb{E}_{\mathrm{P}^n} (f_{\mathrm{D},H}(x) - f_{\mathrm{P},H}(x))^2        
		\cdot \mathbb{E}_{\mathrm{P}_H} |f_{\mathrm{P},H}(x) - f(x)|
		\nonumber\\
		& \leq \frac{\mathrm{P}(A_H(x)) (1 - \mathrm{P}(A_H(x)))}{n\mu(A_H(x)))^2} \cdot c_L \overline{h}_{0}^{1+\alpha}\\
		&\leq \frac{c_L^2}{c_0^2} n^{-1}\overline{h}_{0}^{-d+1+\alpha}.
	\end{align}
	
	Combining \eqref{equ::L3Decomposition} with \eqref{eq::L3approx}, \eqref{eq::L3sample} and \eqref{eq::twoone}, we obtain
	\begin{align*}
		\|f_{\mathrm{D},H}&-f\|_{L_3(\mu)}^3 \\
		& \leq \mu(B_{R, \sqrt{d} \cdot \overline{h}_0}^+ \cap \mathcal{A}_f^1) \cdot
		\biggl( \frac{d}{4} c_L^3 \overline{h}_0^{3+\alpha} + c_{\alpha}^3 \overline{h}_0^{3(1+\alpha)} 
		\\                          
		& \quad+ \frac{\overline{c}_f}{c_0^2} n^{-2} \overline{h}_0^{-2d} + \frac{3c_L^2}{c_0^2} n^{-1} \overline{h}_{0}^{-d+1+\alpha} \biggr),
	\end{align*}
	which completes the proof.
\end{proof}

\begin{proof}[Proof of Theorem \ref{thm::loglower}]
	Lemma \ref{lem::relationlogL2} 
	together with
	Theorem \ref{thm::LowerBoundSingles} and Proposition \ref{prop::L3} yields
	\begin{align*}
		\mathcal{R}_{L,\mathrm{P}}&(f_{\mathrm{D},H}) -  \mathcal{R}_{L,\mathrm{P}}^* 
		\\
		& \geq \frac{\|f_{\mathrm{D},H}-f\|_{L_2(\mu)}^2}{2\underline{c}_f} 
		- \frac{\|f_{\mathrm{D},H}-f\|^3_{L_3(\mu)}}{3\overline{c}_f^2} 
		\\
		& \gtrsim \overline{h}_{0,n}^2 +  n^{-1} \overline{h}_{0,n}^{-d} 
		- \overline{h}_0^{3+\alpha} 
		\\
		&\quad-  \overline{h}_0^{3(1+\alpha)} - n^{-2} \overline{h}_0^{-2d}- n^{-1} \overline{h}_{0}^{-d+1+\alpha}.
	\end{align*}
	By choosing 
	\begin{align*}
		\overline{h}_{0,n} := n^{-\frac{1}{2+d}}, 
	\end{align*}
	we obtain
	\begin{align*}
		\mathcal{R}_{L,\mathrm{P}}(f_{\mathrm{D},H}) -  \mathcal{R}_{L,\mathrm{P}}^*
		\gtrsim n^{-\frac{2}{2+d}},
	\end{align*}
	which yields the assertion.
\end{proof}

\section{Supplementary for Experiments}\label{sec::supple_exp}

\subsection{Descriptions of Synthetic Datasets}\label{DesSynData}
The detailed descriptions are shown in Table \ref{tab::descrips_syntheticdata}.
\begin{table*}[!ht]
	\centering
	\captionsetup{justification=centering}
	\caption{Descriptions of synthetic datasets.}
	\label{tab::descrips_syntheticdata}
	\begin{center}
		\begin{tabular}{ll}
			\toprule
			Type & True (Marginal) Distribution \\
			\midrule
			I 
			& $0.4 \cdot \mathcal{N} (e_d, 0.25\cdot \mathrm{I}_d) + 0.6 \cdot \mathcal{N} (-e_d, 0.25\cdot \mathrm{I}_d)$ \\
			\hline
			II 
			& $f_i := 0.7 \cdot \mathrm{Beta}(2, 10) + 0.3 \cdot \mathrm{Unif}(0.6, 1.0)$ \\
			\hline
			III
			& $f_i := 0.5\cdot \mathrm{Laplace}(0,0.5) + 0.5\cdot \mathrm{Unif}(2,4)$ \\
			\hline
			IV
			& $f_i := \mathrm{Exp}(0.5)$ for $1 = 1, \ldots, d-1$  and $f_d := \mathrm{Unif}(0,5)$ \\
			\bottomrule
		\end{tabular}
		\begin{tablenotes}
			\small{\item {*} For notational simplicity, we denote $e_d := (1, 1, \ldots)$, $e'_d := (1, -1, \ldots)$,  $\mathrm{I}_d$ as the identity matrix, and $f_i$ as the marginal distribution of the $i$-th dimension.
				For Types II, III, IV, the marginal distributions of the true density are independent, and the marginal distributions are identical for Types II and III.}
		\end{tablenotes}
	\end{center}
\end{table*}

In order to give clear visualization of the distributions, we take $d=2$ for instance, and give the 3D visualization of the above four types of distributions in Figure \ref{SynthDataPlot}, where $x$-axis and $y$-axis represent the $2$-dimensional feature space and $z$-axis represents the value of the density function.

\begin{figure}[!h]
	\centering
	\captionsetup{justification=centering}
	\subfigure[Type I]{
		\begin{minipage}[t]{0.47\columnwidth}
			\centering 	\includegraphics[width=\columnwidth]{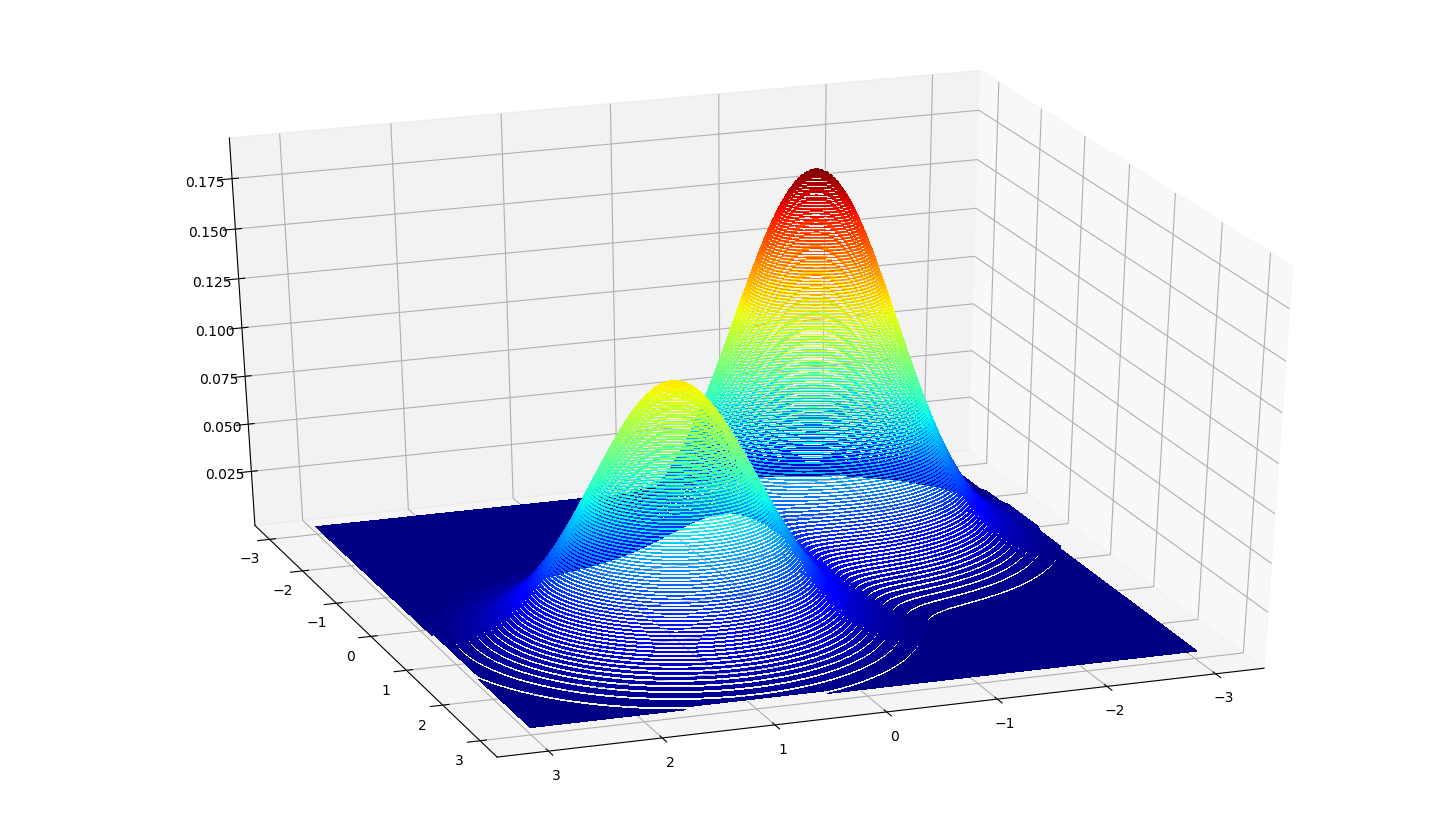} 
			\label{fig:side:a} 
		\end{minipage} 
	}
	\hspace{-5mm}
	\subfigure[Type II]{
		\begin{minipage}[t]{0.47\columnwidth}
			\centering 	\includegraphics[width=\columnwidth]{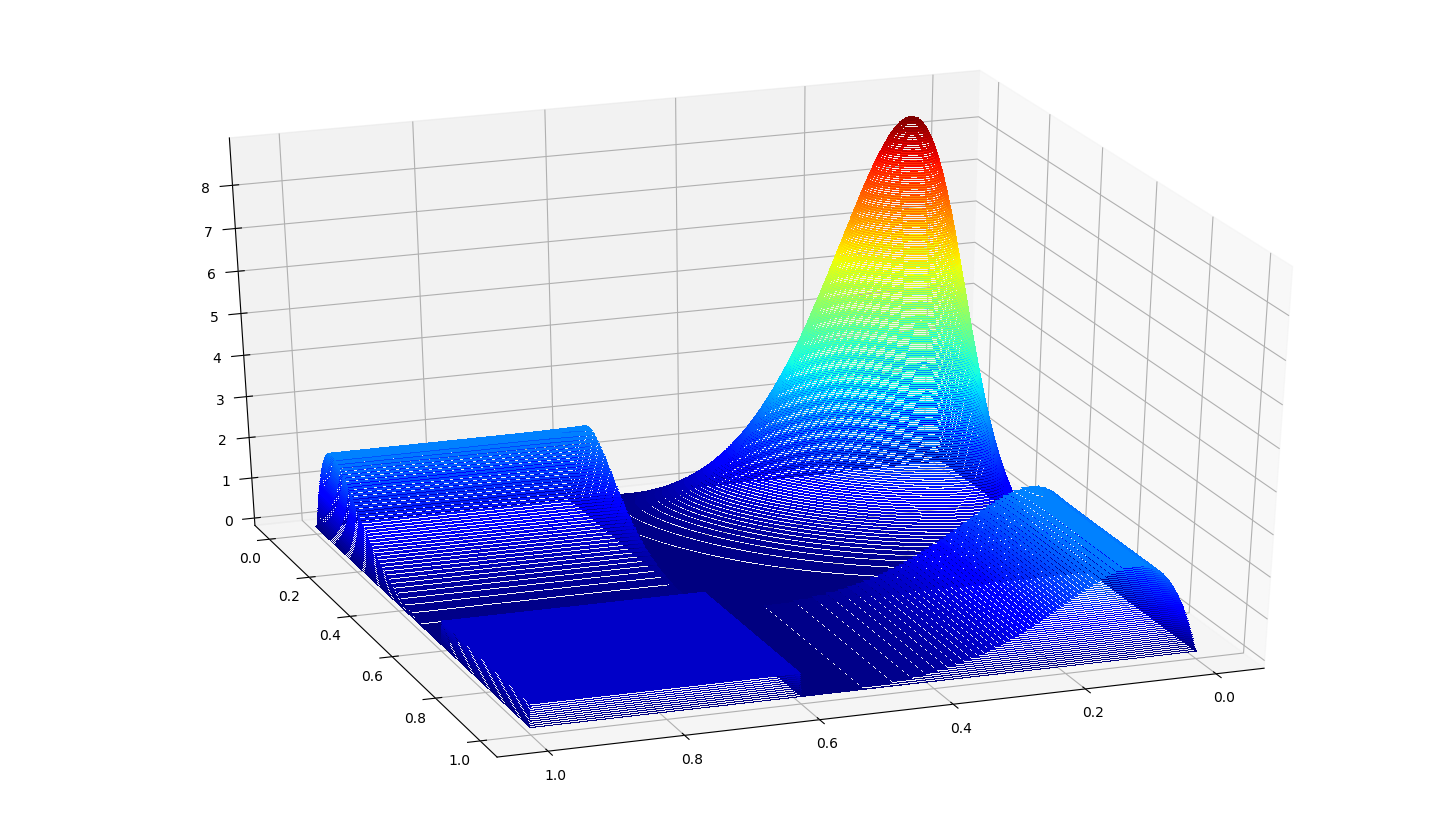} 
			\label{fig:side:b} 
		\end{minipage} 
	} \\
	\subfigure[Type III]{
		\begin{minipage}[t]{0.47\columnwidth}
			\centering 	\includegraphics[width=\columnwidth]{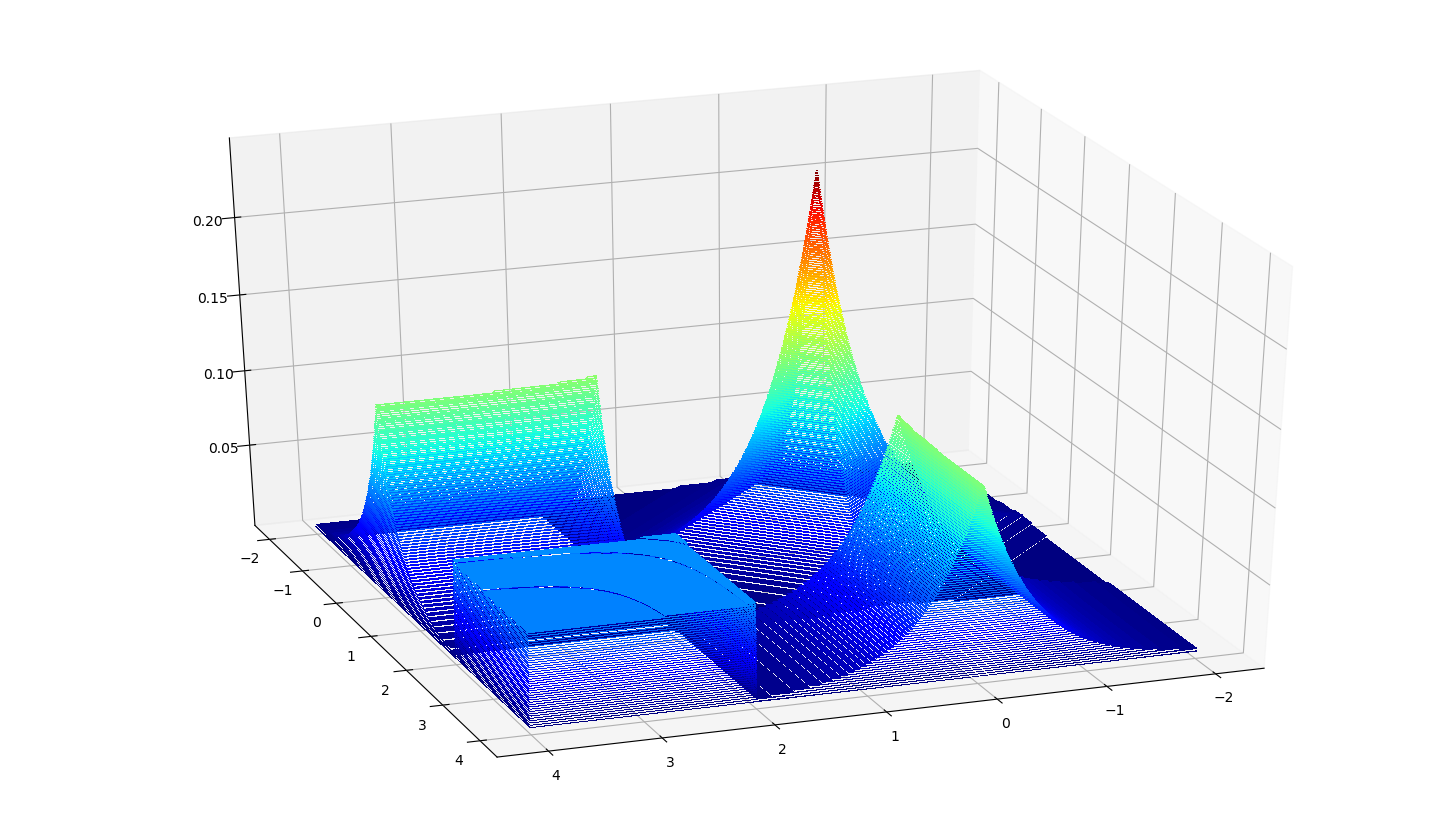} 
			\label{fig:side:c} 
		\end{minipage} 
	}
	\hspace{-5mm}
	\subfigure[Type IV]{
		\begin{minipage}[t]{0.47\columnwidth}
			\centering 	\includegraphics[width=\columnwidth]{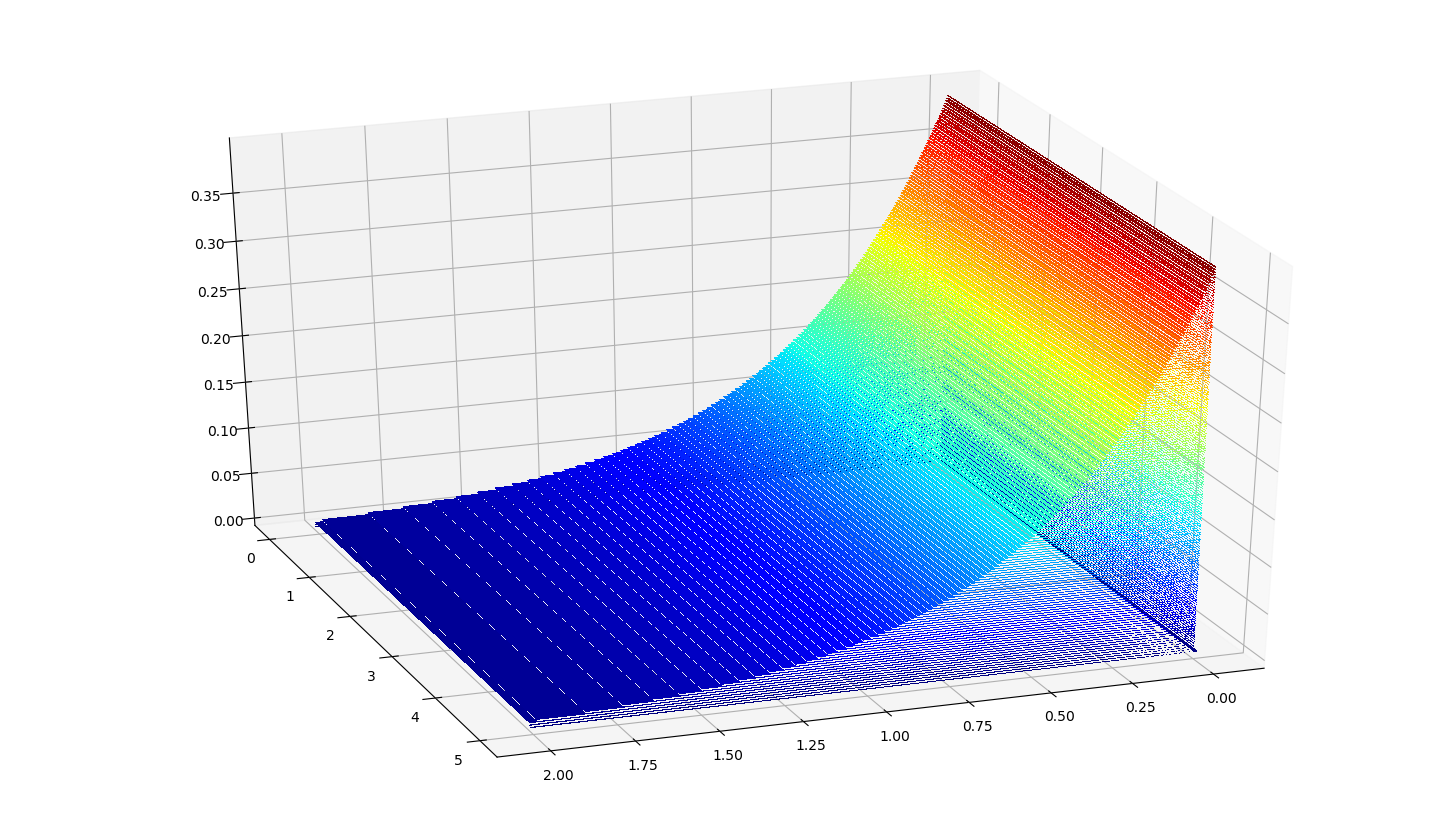} 
			\label{fig:side:d} 
		\end{minipage} 
	}
	\vskip -0.1in
	\caption{3D plots of the synthetic distributions with $d = 2$.}
	\label{SynthDataPlot}
\end{figure} 

\subsection{Descriptions of Real Datasets} \label{DesRealData}

\begin{table*}[!ht] 
	\centering
	\captionsetup{justification=centering}
	\caption{\normalsize{Descriptions of Benchmark Datasets}}
	\label{tab::RealDataDescription}
	\begin{tabular}{l|r|r|r|l|r|r|r}
		\toprule
		Datasets    & $n$           & $d$       & \#outliers(\%)   &  Datasets    & $n$           & $d$       & \#outliers(\%)   \\
		\midrule
		arrhythmia  & $452$      & $274$     & $66(15\%)$  &
		breastw  & $683$       & $9$     & $239(34.99\%)$  \\
		cardio  & $1,831$      & $21$    & $176(9.61\%)$ &
		forestcover & $286,048$    & $10$    & $2747(0.96\%)$  \\
		heart & $267$       & $44$    & $55(20.60\%)$  &
		http & $567,498$    & $3$     & $2211(0.39\%)$ \\
		ionosphere & $351$       & $33$    & $126(35.90\%)$ &
		letter & $1,600$ & $32$ & $100(6.25\%)$ \\
		mammo. & $11,183$     & $6$     & $260(2.32\%)$ &
		mnist & $7,602$ & $100$ & $700(9.2\%)$ \\
		mulcross & $262,144$    & $4$     & $26214(10.00\%)$ &
		musk & $3,062$ & $166$ & $97(3.2\%)$ \\
		optdigits & $5,216$ & $64$ & $150(3\%)$ &
		pendigits & $6,870$      & $16$    & $156(2.27\%)$ \\
		pima & $768 $      & $8$     & $268(34.90\%)$ &
		satellite & $6,435$ & $36$ &	$2036 (32\%)$ \\
		shuttle & $49,097$     & $9$     & $3511(7.15\%)$ &
		vertebral  & $240$     & $6$     & $30(12.5\%)$ \\
		vowels & $1,456$      & $12$    & $50(3.43\%)$ &
		wbc & $129$ &$13$ &	$10 (7.7\%)$ \\
		\bottomrule    
	\end{tabular}
\end{table*}

\begin{table*}[!ht] 
	\setlength{\tabcolsep}{9pt}
	\centering
	\captionsetup{justification=centering}
	\caption{\normalsize{\textit{AUC} performance on benchmark datasets}}
	\label{tab::RealDataResults}
	\footnotesize{
		\begin{tabular}{l|c|c|c|c|c}
			\toprule
			Datasets & GBHT (Ours) & $k$-NN & iForest & LOF & OCSVM \\ 
			\midrule
			arrhythmia  & $0.7952$ &	$0.8165$ & $\underline{0.8073}$	& $0.8130$ &	${0.7948}$ \\
			breastw  & $0.9872$ &	$\underline{0.9881}$ &	$\mathbf{0.9884}$ &	${0.4676}$ &	$0.9789$\\
			cardio  & $0.8921$ &$0.8744$	&$\underline{0.9297}$&	${0.6790}$ &	$\mathbf{0.9473}$ \\
			forestcover & $\mathbf{0.9360}$&	$\underline{0.8950}$&	$0.8792$ &	${0.5778}$&	$0.6565$ \\
			heart & $\mathbf{0.6228}$&	${0.1908}$&$0.2683$ &	$0.2941$ &	$\underline{0.5000}$ \\
			http & $\underline{0.9970}$&	${0.2309}$ &$\mathbf{0.9999}$&$0.3675$&$0.9953$ \\
			ionosphere & $\underline{0.9313}$&	$0.9294$&	${0.8520}$ &	$0.9023$& $\mathbf{0.9382}$ \\
			letter & $0.8222$ &	$\underline{0.9071}$ &	$0.6258$ &	$\mathbf{0.9120}$ &	${0.6860}$\\
			mammo. & $\mathbf{0.8786}$&$0.8527$&	$0.8631$ & ${0.7568}$&	$\underline{0.8721}$ \\
			mnist & $\underline{0.8385}$ &	$\mathbf{0.8591}$	&$0.8117$&	${0.7406}$ &	$0.8216$\\
			mulcross & $\mathbf{1.0000}$ &${0.0013}$ &$0.9642$ &$0.5848$	&$\underline{0.9778}$ \\
			musk & $\underline{0.9893} $ & $0.9367$ &	$\mathbf{1.0000}$ &	$0.5476$&${0.5281}$\\
			optdigits & $0.6381$ &	${0.4292}$ &	$\underline{0.7116}$ &	$0.6682$ &	$\mathbf{0.8966}$ \\
			pendigits & $0.8991$	&$0.8607$ &	$\underline{0.9538}$ &	${0.5437}$ &	$\mathbf{0.9607}$ \\
			pima & $\mathbf{0.6990}$&$0.6437$ &	$\underline{0.6796}$ &	$0.6162$ &	${0.5842}$ \\
			satellite & $\underline{0.7223}$ &	$\mathbf{0.7374}$ &	$0.7041$ &	${0.5701}$ & $0.7064$\\
			shuttle & $0.9842$ &	$0.8004$ &$\mathbf{0.9974}$ &${0.6035}$ & $\underline{0.9918}$ \\
			vertebral  & $\mathbf{0.5523}$ & ${0.3253}$ &	$0.3585$ & $0.5310$ & $\underline{0.5374}$ \\
			vowels & $0.9237$ &$\mathbf{0.9749}$&${0.7588}$ &	$\underline{0.9467}$ &$0.9153$ \\
			wbc & $\mathbf{0.9524}$ & $\underline{0.9501}$ & ${0.9412}$ &	$0.9460$ &	$0.9469$\\
			\hline
			Rank Sum    & $\mathbf{43}$             & $62$             & $60$                      & $78$              & $\underline{57}$                     \\
			\bottomrule    
	\end{tabular}}
	\begin{minipage}{\textwidth}
		\begin{tablenotes}
			\centering	
			\captionsetup{justification=centering}
			\footnotesize
			\item{*} The best results are marked in \textbf{bold}, the second best results are marked in $\underline{\text{underline}}$.
			\item{**} The last row shows the summation of ranks for each method, which is the lower the better.
		\end{tablenotes}
	\end{minipage}
\end{table*}

As follows are the datasets alphabetically listed, with the number of instances and features reported after preprocessing.
\begin{itemize}
	
	\item {\tt Adult} is also known as "Census Income" dataset. 
	It contains $48,842$ instances with $6$ countinuous and $8$ discrete attributes.
	Prediction task is to determine whether a person makes over 50K a year.
	
	\item {\tt Australian} is an interesting dataset with a good mix of attributes, which contains continuous, nominal with both small and large numbers of values. The dataset contains $690$ instances with $6$ numerical and $9$ categorical attributes, mainly concerning credit card applications.
	
	\item {\tt Breast-cancer} is originally for predicting whether a cancer is recurrence event. 
	It contains $675$ instances of dimension $11$, describing the status of the tumors and the patients. 
	
	\item {\tt Diabetes} dataset comprises $768$ samples and $9$ features. The attributes concern about the medical records of patients, consisting of $8$ numerical features and $1$ categorical feature.
	
	%
	
	\item {\tt Ionosphere} is a multivariate dataset for binary classification tasks, attribute to predict is either ``good'' or ``bad''. 
	This radar data was collected by a system in Goose Bay, Labrador.
	It contains $351$ instances of dimension $34$.
	
	\item {\tt Parkinsons} dataset is composed of a range of biomedical voice measurements from $31$ people, $23$ with Parkinson's disease (PD). 
	It contains $197$ instances of dimension $23$.
	
\end{itemize}


For anomaly detection, we select $20$ real datasets from the ODDS library, with various sample sizes and dimensionalities. Details of real-world datasets are shown in Table \ref{tab::RealDataDescription}.

\subsection{Gradient Boosted Histogram Transform (GBHT) for Anomaly Detection} \label{sec::Aanomaly}

We conduct numerical experiments to make a comparison between our GBHT and several popular anomaly detection algorithms such as the forest-based  Isolation Forest (iForest) \cite{liu2008isolation}, the distance-based $k$-Nearest Neighbor ($k$-NN) \cite{ramaswamy2000efficient} and Local Outlier Factor (LOF) \cite{Breunig2000LOF}, and the kernel-based one-class SVM (OCSVM) \cite{scholkopf2001estimating}, on $20$ real-world benchmark outlier detection datasets from the ODDS library. The detailed descriptions of these datasets can be found in Table \ref{tab::RealDataDescription} in Section \ref{DesRealData} of the supplement. The measure for the performance evaluation is the area under the ROC curve (\textit{AUC}). For each method, we choose the best \textit{AUC} performance when parameters go though their parameter grids.

The implementation details are below: For our method, the grid of $s_{\min}$ and $s_{\max}-s_{\min}$ are $\{-3,-2,-1,0\}$ and $\{0.5,1,2,3\}$, respectively. The number of iterations $T$ is chosen from $\{100,500\}$. Moreover, we incorporate Nesterov’s descent method \cite{biau2014accelerated} into our boosting algorithm for accelerating and set shrinkage parameter grid to be $\{0.1, 0.5\}$. For iForest, LOF and OCSVM, we utilized the implementation of scikit-learn. For $k$-NN and LOF, the parameter grid of number of neighbors $k$ is $\{5,10,15,\cdots,45,50\}$. As for iForest, we set the grid of the number of trees to be $\{100, 500\}$ and sub-sampling size to be $256$. For OCSVM, we use RBF kernel with gamma grid $\{0.001,0.01,\cdots,1,10\}$.
The experimental results are reported in Table \ref{tab::RealDataResults}.

\end{document}